\documentclass{article}
\usepackage[final]{neurips_2022}
\usepackage{xr}
\usepackage[utf8]{inputenc} 
\usepackage[T1]{fontenc}    
\usepackage{hyperref}       
\usepackage{url}            
\usepackage{booktabs}       
\usepackage{amsfonts}       
\usepackage{nicefrac}       
\usepackage{microtype}      

\usepackage{graphicx}
\usepackage{subcaption}
\usepackage{booktabs} 
\usepackage{amsmath}
\usepackage{amsthm}
\usepackage{amssymb}
\usepackage{mathtools}
\usepackage{natbib}
\usepackage{algorithm}
\usepackage{algorithmic}
\usepackage{thmtools}
\usepackage{thm-restate}
\usepackage{multirow}

\usepackage{hyperref}
\usepackage{xcolor}

\newcommand{\bw}{\text{\boldmath{$w$}}}

\newcommand{\bQ}{\text{\boldmath{$Q$}}}

\newcommand{\bz}{\boldsymbol{z}}
\newcommand{\bx}{\boldsymbol{x}}

\newcommand{\bu}{\boldsymbol{u}}

\newcommand{\bS}{\boldsymbol{S}}

\newcommand{\bI}{\boldsymbol{I}}

\newcommand{\bA}{\boldsymbol{A}}
\newcommand{\bM}{\boldsymbol{M}}
\newcommand{\bV}{\boldsymbol{V}}
\newcommand{\bbP}{\mathbb{P}}

\newcommand{\bbR}{\mathbb{R}}

\newcommand{\bv}{\boldsymbol{v}}

\DeclareMathOperator*{\argmin}{arg\,min}

\newtheorem{assumption}{\textbf{Assumption}}\newtheorem{definition}{\textbf{Definition}}\newtheorem*{definition*}{\textbf{Definition}}\newtheorem{corollary}{\textbf{Corollary}}\newtheorem{lemma}{\textbf{Lemma}}\newtheorem{theorem}{\textbf{Theorem}}\newtheorem{proposition}{\textbf{Proposition}}\newtheorem{example}{\textbf{Example}}

\newcommand{\mE}{\mathbb{E}}

\newcommand{\cD}{\mathcal{D}}
\newcommand{\cM}{\mathcal{M}}

\newcommand{\cW}{\mathcal{W}}
\newcommand{\cS}{\mathcal{S}}

\newcommand{\cN}{\mathcal{N}}

\newcommand{\cP}{\mathcal{P}}

\newcommand{\cA}{\mathcal{A}}

\newcommand{\cO}{\mathcal{O}}

\makeatother

\title{Characterization of Excess Risk for Locally Strongly Convex Population Risk}
\usepackage{neurips_2022}

\begin{document}
%
	\author{Mingyang Yi$^{1,2,3}\thanks{equal contribution}$, Ruoyu Wang$^{1,2*}$, Zhi-Ming Ma$^{1,2}$\\
		$^{1}$University of Chinese Academy of Sciences\\
		$^{2}$Academy of Mathematics and Systems Science, Chinese Academy of Sciences\\
		$^{3}$Huawei Noah’s Ark Lab \\
		\texttt{\{yimingyang17, wangruoyu17\}@mails.ucas.edu.cn} \\
		\texttt{mazm@amt.ac.cn}
	}
	\date{}
	\maketitle

	\begin{abstract}
		We establish upper bounds for the expected excess risk of models trained by proper iterative algorithms which approximate the local minima. Unlike the results built upon the strong globally strongly convexity or global growth conditions e.g., PL-inequality, we only require the population risk to be \emph{locally} strongly convex around its local minima. Concretely, our bound under convex problems is of order $\tilde{\cO}(1/n)$. For non-convex problems with $d$ model parameters such that $d/n$ is smaller than a threshold independent of $n$, the order of $\tilde{\cO}(1/n)$ can be maintained if the empirical risk has no spurious local minima with high probability. Moreover, the bound for non-convex problem becomes $\tilde{\cO}(1/\sqrt{n})$ without such assumption. Our results are derived via algorithmic stability and characterization of the empirical risk's landscape. Compared with the existing algorithmic stability based results, our bounds are dimensional insensitive and without restrictions on the algorithm's implementation, learning rate, and the number of iterations. Our bounds underscore that with locally strongly convex population risk, the models trained by any proper iterative algorithm can generalize well, even for non-convex problems, and $d$ is large.
		
	\end{abstract}
	
	\section{Introduction}\label{sec:intro}
	The core problem in machine learning is obtaining a model that generalizes well on unseen test data. The excess risk decides the model's performance on these unseen data, and it can be decomposed into optimization and generalization errors. The tool of algorithmic stability \citep{bousquet2002stability,bousquet2020sharper} has been proven to be a suitable tool for exploring the excess risk. Roughly speaking, the output of a stable algorithm is robust to a slight change in the algorithm's input, i.e., training set. The output of a stable algorithm has been proved to have controlled excess risk in \citep{bousquet2002stability}, and the result has been further developed under some specific algorithms \citep{hardt2016train,yuan2019stagewise,charles2018stability,chen2018stability,meng2017generalization,deng2020toward} e.g., stochastic gradient descent \citep{robbins1951stochastic} (SGD). However, these results have some limitations. The results in \citep{yuan2019stagewise,charles2018stability,meng2017generalization,li2022high} are obtained under the assumption of either global strong convexity or global growth conditions (PL-inequality \citep{karimi2016linear}). On the other hand, the results in \citep{hardt2016train,deng2020toward} are only applicable to a specific algorithm, i.e., SGD, and their bounds of generalization error diverge across training which is inconsistent with the observation that ``train longer, generalize better'' \citep{hoffer2017train}.     
	\par
	To improve these, we provide a unified analysis of the expected excess risk for a generic class of iterative algorithms without any strong global conditions, i.e., global strong convexity or global growth conditions in \citep{yuan2019stagewise,charles2018stability,meng2017generalization}. Concretely, we substitute the strong global conditions with weaker local strong convexity (see Section \ref{sec:preliminaries}) of population risk around its local minima. The substitution is based on the fact that the nice strong convexity property can be locally (though not globally) satisfied by many important problems, e.g., PCA \citep{gonen2017fast}, ICA \citep{ge2015escaping}, and matrix completion \cite{ge2016matrix}. We derive our results via algorithmic stability and characterize the empirical risk's landscape. For both convex and non-convex problems, our results can be applied to any proper algorithms that approximate local minima. Moreover, our generalization upper bounds do not diverge with the number of training steps.
	\par
	Technically, we upper bound both generalization and optimization errors to control the excess risk. We first show a fact that the locally strongly convexity around the local minima of population risk (population local minima) can be generalized to the local minima of empirical risk (empirical local minima), and the empirical local minima would concentrate around population local minima. Then for convex problems, we establish the generalization upper bound of the iterates of any proper algorithm via algorithmic stability by leveraging the facts of iterates will converge to empirical local minima, which concentrate around population local minima. For non-convex problems, our generalization error analysis includes three steps. 1) By applying similar arguments under the convex problem, we upper bound the generalization error of those empirical local minima around population local minima. 2) Then, we prove that, with high probability, there are no extra empirical local minima except for those concentrated around population local minima with guaranteed generalization capability. 3) Finally, we extrapolate the upper bound of the generalization error to the iterates obtained by the proper algorithm as they converge to empirical local minima.
	\par
	After controlling the generalization error, the excess risk is directly implied by characterizing the optimization error. By the proved local strong convexity of empirical risk and the convergence results of proper algorithms, the optimization error can be controlled as in \citep{bubeck2014convex,ghadimi2013stochastic,shamir2013stochastic,ge2015escaping,jin2017escape}. 
	\par
	Concretely, we establish an upper bound of order $\tilde{\cO}(1/n)$ ($\tilde{\cO}(\cdot)$ defined in Section \ref{sec:preliminaries}) for the expected excess risk of iterates obtained by any proper algorithm under convex problems. Here $n$ is the number of training samples. For non-convex problems with $d$ parameters of model, we establish an upper bound of order $\tilde{\cO}(1 / \sqrt{n} + \exp(-n(c_{1} - d/n))$ where $c_{1}$ is a constant independent of $n$ and $d$. Noticeably, the exponential term in the bound can be ignored when $d / n \leq c_{1}$, then our bound becomes $\tilde{\cO}(1/\sqrt{n})$. The bound can be applied to high-dimensional problems such that $d$ is in the same order of $n$. The result significantly improves the classical one of order $\cO(\sqrt{d/n})$ \citep{shalev2009stochastic}, which has polynomial dependence on $d$. Moreover, our bound of order $\tilde{\cO}(1/ \sqrt{n})$ can be improved to $\tilde{\cO}(1 / n)$ if the empirical risk has no spurious local minima with high probability, which can be satisfied for many important non-convex problems \citep{gonen2017fast,ge2016matrix,allen2019convergence}. 
	\par
	Our upper bounds to the excess risk underscore that, for both convex and non-convex problems satisfying our regularity conditions, the model trained by an algorithm can generalize on test data even when $d$ is large. Our improvements over existing classical results are summarized as follows. 
	\par
	\textbf{\textbullet} For convex problems, our bound improves the standard upper bound of the expected excess risk in the order of $\cO(\sqrt{1/n})$ \citep{hardt2016train} to $\tilde{\cO}(1/n)$, under an extra locally strongly convex assumption. 
	\par
	\textbf{\textbullet} For non-convex problems, we relax the dimensional-dependence in the standard excess risk bound of order $\cO(\sqrt{d/n})$ \citep{shalev2009stochastic}, under local strong convexity assumption.
	\par
	\textbf{\textbullet} In contrast to the existing algorithmic stability based works \citep{hardt2016train,yuan2019stagewise,charles2018stability}, our results can be applied to any algorithms that approximate local minima without restrictions on the implementation of algorithms, learning rate, and the number of iterations.
	
	\section{Preliminaries}\label{sec:preliminaries}
	\subsection{Notations and Assumptions}\label{sec:notation}
	In this subsection, we collect our (mostly standard) notations and assumptions. We use $\|\cdot\|$ to denote $\ell_{2}$-norm for vectors and spectral norm for matrices. $B_{p}(\bw, r)$ is $\ell_{p}$-ball with radius $r$ around $\bw\in\bbR^{d}$. Let dataset $\{\bz_{1}, \cdots, \bz_{n},\bz_{1}^{\prime},\cdots, \bz_{n}^{\prime}\}$ be $2n$ i.i.d samples from an unknown distribution, and $\bS=\{\bz_{1}, \cdots, \bz_{n}\}$ is the training set, $\bS^{i}=\{\bz_{1},\cdots, \bz_{i - 1}, \bz_{i}^{\prime}, \bz_{i + 1}, \cdots, \bz_{n}\}$ and $\bS^{\prime} = \bS^{1}$. Throughout this paper, we assume without further mention that the loss function $f(\bw, \bz)$ is differentiable w.r.t. to parameter $\bw$ for any $\bz$, $0 \leq f(\bw, \bz)\leq M$, and the parameter space $\cW\subseteq\bbR^{d}$ is a convex compact set. Thus $\|\bw_{1} - \bw_{2}\|\leq D$ for $\bw_{1}, \bw_{2}\in\cW$ and some positive constant $D$.
	The population risk is 
		$R(\bw) = \mE_{\bz}[f(\bw, \bz)]$ and its empirical counterpart on the training set $\bS$ is $R_{\bS}(\bw) = n^{-1}\sum_{i=1}^{n}f(\bw, \bz_{i})$. Let $\bw_{\bS}^{*}\in\arg\min_{\bw}R_{\bS}(\bw)$ and $\bw^{*} \in \arg\min_{\bw}R(\bw)$, 
	The projection operator $\cP_{\cW}(\cdot)$ is defined as $\cP_{\cW}(\bv) = \mathop{\arg\min}_{\bw\in\cW}\left\{\|\bw - \bv\|\right\}$. During our analysis, the order of sample size $n$ can go to infinity, and $d$ can diverge to infinity with $n$. But we assume the other quantities are universal constant independent of $n$. The symbol $\cO(\cdot)$ is the order of a number, while $\tilde{\cO}(\cdot)$ hides a poly-logarithmic factor in the number of model parameters $d$. The following two assumptions on loss function $f(\bw, \bz)$ are imposed on the population risk.
	\begin{assumption}[Smoothness]\label{ass:smoothness}
		For $0\leq j \leq 2$, each $\bz$ and any $\bw_{1}, \bw_{2} \in \cW$, 
		\begin{equation}
			\small
			\left\|\nabla^{j}f(\bw_{1}, \bz) - \nabla^{j}f(\bw_{2}, \bz)\right\| \leq L_{j}\|\bw_{1} - \bw_{2}\|,
		\end{equation}
		where $\nabla^{j} f(\bw, \bz)$ are respectively loss function, gradient, and Hessian at $\bw$ for $j = 0, 1, 2$.
	\end{assumption}
	\begin{assumption}[Non-Degenerate Local Minima]\label{ass:local strong convexity}
		For $\bw_{\rm local}^{*}$ in the set of local minima of population risk $R(\bw)$, $\nabla^{2}R(\bw_{\rm local}^{*}) \succeq \lambda > 0$, i.e., $\nabla^{2}R(\bw_{\rm local}^{*}) - \lambda\bI_{d}$ is a semi-positive definite matrix. 
	\end{assumption}
	Assumption \ref{ass:smoothness} says that the loss function should be smooth enough, which is a mild assumption and has been adopted in  \citep{hardt2016train,zhang2017empirical,gonen2017fast}. Assumption \ref{ass:smoothness} and \ref{ass:local strong convexity} together imply that the population risk is locally strongly convex around its local minima. The rationale behind the imposed local strong convexity is as follows. Though the strong global conditions (e.g., global strong convexity) in \citep{hardt2016train,yuan2019stagewise,charles2018stability,chen2018stability,meng2017generalization,deng2020toward} do not hold in many problems, the weaker locally strongly convex condition can be satisfied by many important problems, e.g., generalized linear regression \citep{mei2018landscape}, robust regression \citep{mei2018landscape}, PCA \citep{gonen2017fast}, ICA \citep{ge2015escaping}, and matrix completion \citep{ge2016matrix}. The detailed examples of import problems that satisfy the assumptions imposed in this paper are in Appendix \ref{app:examples}.  
	
	
	\subsection{Stability and Generalization}\label{sec:stability and generalization}
	\begin{definition}[Proper Algorithm]\label{def:proper}
		The algorithm $\cA$ is proper if it approximates local minima \footnote{Please notice that local minima are all global minima for convex problem.} of empirical risk $R_{\bS}(\bw)$.
	\end{definition}
	This is a rough definition of the discussed proper algorithm. The sense in which algorithms approximate local minima will be made clear in our formal theoretical results. Let $\cA(\bS)$ be the parameters obtained by an algorithm $\cA$, e.g., SGD, on the training set $\bS$. The performance of model on unseen data is determined by the excess risk $R(\cA(\bS)) - \inf_{\bw} R(\bw)$, which is the gap of population risk between the current model and the optimal one. 
	In this paper, we explore the expected excess risk $\mE_{\cA,\bS}[R(\cA(\bS)) - \inf_{\bw} R(\bw)]$ where $\mE_{\cA,\bS}[\cdot]$ means the expectation is taken over the randomized algorithm $\cA$ and the training set $\bS$. We may neglect the subscript if there is no obfuscation. Since $R_{\bS}(\bw_{\bS}^{*}) \leq R_{\bS}(\bw^{*})$, we have the following decomposition. 
	\begin{equation}\label{eq:decomposition of excess risk}
		\small
		\begin{aligned} 
			\mE_{\cA,\bS}[R(\cA(\bS)) - R(\bw^{*})] & = \mE_{\cA,\bS}[R(\cA(\bS)) - R_{\bS}(\bw^{*})] \leq \mE_{\cA,\bS}[R(\cA(\bS)) - R_{\bS}(\bw_{\bS}^{*})] \\
			& = \mE_{\cA,\bS}[R_{\bS}(\cA(\bS)) - R_{\bS}(\bw_{\bS}^{*})] + \mE_{\cA,\bS}[R(\cA(\bS)) - R_{\bS}(\cA(\bS))]\\
			& \leq \underbrace{\mE_{\cA,\bS}[R_{\bS}(\cA(\bS)) - R_{\bS}(\bw_{\bS}^{*})]}_{\mathcal{E}_{\rm opt}} +  
			\underbrace{|\mE_{\cA,\bS}[R(\cA(\bS)) - R_{\bS}(\cA(\bS))]|}_{\mathcal{E}_{\rm gen}}.
		\end{aligned}
	\end{equation}
	\par
	The expected excess risk is upper bounded by the sum of optimization error $\mathcal{E}_{\rm opt}$ and generalization error $\mathcal{E}_{\rm gen}$. $\mathcal{E}_{\rm opt}$ is decided by the convergence rate of the algorithm $\cA$ \citep{bubeck2014convex,ghadimi2013stochastic}. The generalization error $\mathcal{E}_{\rm gen}$ can be controlled by algorithmic stability \citep{bousquet2002stability} as follows. 
	\begin{definition}
		An algorithm $\cA$ is $\epsilon$-uniformly stable, if
		\begin{equation}
			\small
			\epsilon_{\rm stab} = \mE_{\bS, \bS^{\prime}}\left[\sup_{\bz}|\mE_{\cA}[f(\cA(\bS), \bz) - f(\cA(\bS^{\prime}), \bz)]|\right] \leq \epsilon,
		\end{equation}
		where $\bS$ and $\bS^{\prime}$ are defined at the beginning of Section \ref{sec:notation}.
	\end{definition}
	The $\epsilon$-uniformly stable is different from the one in \citep{hardt2016train}, which does not take expectation over training sets $\bS$ and $\bS^{\prime}$. The next theorem shows that the uniform stability implies the expected generalization of the model, i.e., $\mathcal{E}_{\rm gen} \leq \epsilon_{\rm stab}$. The idea of Theorem \ref{thm:stability and generalization} is similar to the ones in  \citep{bousquet2002stability,hardt2016train,charles2018stability}, and its proof is in Appendix \ref{app:proof of stability and generalization}.
	\begin{theorem}\label{thm:stability and generalization}
		If $\cA$ is $\epsilon$-uniformly stable, then
		\begin{equation}
			\small
			\mathcal{E}_{\rm gen} = \left|\mE_{\cA,\bS}\left[R(\cA(\bS)) - R_{\bS}(\cA(\bS))\right]\right| \leq \epsilon.
		\end{equation}
	\end{theorem} 
	\par
	Please note that all the analysis in this paper is applicable to the practically infeasible empirical risk minimization ``algorithm'' such that $\cA(\cS) = \bw_{\bS}^{*}$. However, to make our results more practical, we suppose $\cA$ as iterative algorithms in the sequel. For any given iterative algorithm $\cA$, let $\bw_{t}$ and $\bw_{t}^{\prime}$ denote the output of the algorithm when $\cA$ is iterated $t$ steps on the training set $\bS$ and $\bS^{\prime}$ respectively. 
	\section{Excess Risk under Convex Problems}\label{sec:testing error of convex function}
	In this section, we propose upper bounds of the expected excess risk for convex problems. We impose the following convexity assumption throughout this section. 
	\begin{assumption}[Convexity]\label{ass:convexity}
		For each $\bz$ and any 
		$\bw_{1}, \bw_{2} \in \cW$, $f(\bw, \bz)$ satisfies 
		\begin{equation}
			\small
			f(\bw_{1}, \bz) - f(\bw_{2}, \bz) \leq \langle \nabla f(\bw_{1}, \bz), \bw_{1} - \bw_{2}\rangle.
		\end{equation} 
	\end{assumption} 
	\subsection{Generalization Error under Convex Problems}\label{sec:stability of convex function}
	As we have discussed, in the existing literature \cite{hardt2016train,yuan2019stagewise,charles2018stability,chen2018stability,meng2017generalization,deng2020toward}, researchers have explored the excess risk via the algorithmic stability to control the error generalization. However, the obtained generalization upper bounds of order $\cO(1/n)$ in  \citep{hardt2016train,yuan2019stagewise,charles2018stability,meng2017generalization} are built upon the strong assumptions of either global strong convexity or global growth conditions, e.g., PL-inequality \citep{karimi2016linear}. On the other hand, the generalization upper bounds in \citep{hardt2016train,deng2020toward} are only applied to SGD, and they diverge as the number of iterations grows. For example, Theorem 3.8 in \citep{hardt2016train} establishes an upper bound $2L_{0}^{2}\sum_{k=0}^{t - 1}\eta_{k}/n$ to the algorithmic stability of SGD with learning rate $\eta_{k}$, which diverges when $t\to\infty$, as the convergence of SGD requires $\sum_{k=0}^{\infty}\eta_{k}=\infty$ \citep{bottou2018optimization}. Thus the bound can not explain the observation that the generalization error of SGD trained model converges to a constant \citep{bottou2018optimization,hoffer2017train}.     
	\par
	To mitigate the drawbacks in the existing literature, we propose the following new upper bound of algorithmic stability (Theorem \ref{thm:stability of convex function}). Our bound can be applied on the top of any proper algorithm defined in Definition \ref{def:proper}, and it remains small for an arbitrary number of iterations as long as the sample size $n$ is large. Under convexity Assumption \ref{ass:convexity}, the proper algorithm means that $\mE\left[R_{\bS}(\bw_{t}) - R_{\bS}(\bw_{\bS}^{*})\right] \rightarrow 0$ as $t\rightarrow \infty$. Our theorem is based on the following intuition. Due to the locally strongly convex property discussed after Assumption \ref{ass:local strong convexity}, there exists (with high probability) the unique global minimum $\bw_{\bS}^{*}$ of $R_{\bS}(\cdot)$ and $\bw_{\bS^{\prime}}^{*}$ of $R_{\bS^{\prime}}(\cdot)$ that concentrate around the unique (the uniqueness is from Assumption \ref{ass:local strong convexity}) population global minimum $\bw^{*}$. Then, the provable convergence results of $\bw_{t}\rightarrow \bw_{\bS}^{*}$ and $\bw_{t}^{\prime} \rightarrow \bw_{\bS^{\prime}}^{*}$ imply the algorithmic stability (see Lemma \ref{lem:upper bound on two minimums} in Appendix). 
	\begin{theorem}\label{thm:stability of convex function}
		Under Assumption \ref{ass:smoothness}-\ref{ass:convexity},
		\begin{equation}\small\label{eq:stab bound convex}
			\begin{aligned}
				\epsilon_{\rm stab}(t) & \leq \frac{4\sqrt{2}L_{0}(\lambda + 4DL_{2})}{\lambda^{\frac{3}{2}}}\sqrt{\epsilon(t)} + \frac{8L_{0}}{n\lambda}\left(L_{0} + \frac{64L_{0}^{2}L_{2}^{2}D}{\lambda^{3}}\right) + \frac{128L_{0}L_{1}^{2}D}{n\lambda^{2}}\left(5\sqrt{\log d} + \frac{4e\log d }{\sqrt{n}}\right)^{2} \\
				& = \tilde{\cO}(\sqrt{\epsilon(t)} + 1 / n),
			\end{aligned}
		\end{equation}
		where $\epsilon_{\rm stab}(t) = \mE_{\bS, \bS^{\prime}}\left[\sup_{\bz}|\mE_{\cA}[f(\bw_{t}, \bz) - f(\bw^{\prime}_{t}, \bz)]|\right]$ is the stability of $\bw_{t}$, and $\epsilon(t) = \mE\left[R_{\bS}(\bw_{t}) - R_{\bS}(\bw_{\bS}^{*})\right]$, $\bw_{\bS}^{*}$ is the global minimum of $R_{\bS}(\cdot)$.
	\end{theorem}
	The proof of this theorem is in Appendix \ref{app: proof in 3.1}. The expected generalization error of $\bw_{t}$ is upper bounded by the right hand side of \eqref{eq:stab bound convex} due to Theorem \ref{thm:stability and generalization}. Compared with the existing result \citep{hardt2016train}, the extra term related to $\sqrt{\epsilon(t)}$ in our bound originates from our proof technique, and it seems to be unavoidable according to \citep{shalev2009stochastic}. Since for proper algorithms, e.g., GD and SGD, $\epsilon(t)\to 0$ as $t\to \infty$ the leading term of the upper bound \eqref{eq:stab bound convex} is $C^{*}\log d/ n=\tilde{\cO}(1 / n)$ with $C^{*} = 3200L_{0}L_{1}^{2}D/\lambda^{2}$. 
	\par
	
	In summary, the local strong convexity (Assumption \ref{ass:local strong convexity}) enables us to establish an algorithmic stability based generalization bound \eqref{eq:stab bound convex}. The bound improves the classical result of SGD $2L_{0}^{2}\sum_{k=0}^{t - 1}\eta_{k}/n$ in \citep{hardt2016train} as it can be applied to any proper algorithm with any learning rate and number of iterations. 
	\subsection{Excess Risk Under Convex Problems}\label{sec:optimization error for convex function}
	According to \eqref{eq:decomposition of excess risk}, we can upper bound the expected excess risk by combining the generalization upper bound \eqref{eq:stab bound convex} with the convergence results in convex optimization. 
	\begin{theorem}\label{thm:excess risk for convex loss}
		For $\bw_{\bS}^{*}\in\argmin_{\bw}R_{\bS}(\bw)$, and $\bw^{*}\in\argmin_{\bw}R(\bw)$, under Assumption \ref{ass:smoothness}-\ref{ass:convexity},
		\begin{equation}
			\small
			\begin{aligned}
				\mE\left[R(\bw_{t}) - R(\bw^{*})\right] & \leq \epsilon(t) + \frac{4\sqrt{2}L_{0}(\lambda + 4DL_{2})}{\lambda^{\frac{3}{2}}}\sqrt{\epsilon(t)} +  \frac{8L_{0}}{n\lambda}\left(L_{0} + \frac{64L_{0}^{2}L_{2}^{2}D}{\lambda^{3}}\right) \\
				& + \frac{128L_{0}L_{1}^{2}D}{n\lambda^{2}}\left(5\sqrt{\log d} + \frac{4e\log d }{\sqrt{n}}\right)^{2} \\
				& = \tilde{\cO}(\sqrt{\epsilon(t)} + 1 / n),
			\end{aligned}
		\end{equation}
		where $\epsilon(t) = \mE\left[R_{\bS}(\bw_{t}) - R_{\bS}(\bw_{\bS}^{*})\right]$.
	\end{theorem}
	This theorem provides an upper bound of the expected excess risk. The bound decreases with the number of training steps $t$, and is of order $\tilde{\cO}(1/n)$ if $t$ is sufficiently large. 
	\paragraph{Comparison.}
	Under the extra local strong convexity assumption, our result significantly improves the bound of order $\cO(1/\sqrt{n})$ in \citep{hardt2016train}. On the other hand, our bound matches (in order) the result under strongly convex problem \citep{shalev2009stochastic,zhang2017empirical}. It seems our result has a worse dependence on the strong convex parameter $\lambda$, i.e., from $1/\lambda$ to $1 /\lambda^{4}$. The worse dependence is acceptable as local strong convexity is weaker than strong convexity. Moreover, our bound is not necessarily weaker compared to the current results \citep{shalev2009stochastic,zhang2017empirical} under global strongly convex problem. This is because $\lambda$ in our bound is the local strongly convex parameter restricted around the minimum point, which is larger than the global one over the whole parameter space appears in \cite{zhang2017empirical}. Improving the dependence on $\lambda$ without sacrificing the order of $n$ seems to be infeasible based on our techniques\footnote{The dependence can be improved to $1 / \lambda^{2}$ with a worse order of $n$ (from $1/n$ to $1 / \sqrt{n}$).}. It might be a meaningful topic to be explored in the future. Finally, our result has no conflict with the lower bound for general convex problem in the order of $\cO(\sqrt{d/n})$ \citep{feldman2016generalization}. This is because Assumption \ref{ass:smoothness} and \ref{ass:local strong convexity} restrict our result to a smaller class of distributions and functions, which rules out the counter-examples in \citep{feldman2016generalization}. 
	\par
	To make our results concrete, we apply them to GD and SGD as examples. Note that $R_{\bS}(\bw) = n^{-1}\sum_{i=1}^{n}f(\bw, \bz_{i})$, the GD and SGD respectively start from $\bw_{0}$ follow the update rules of 
	\begin{equation}\small\label{eq:GD}
		\small
		\bw_{t + 1} = \cP_{\cW}\left(\bw_{t} - \eta_{t}\nabla R_{\bS}(\bw_{t})\right),
	\end{equation}
	and 
	\begin{equation}\label{eq:SGD}
		\small
		\bw_{t + 1} = \cP_{\cW}\left(\bw_{t} - \eta_{t}\nabla f(\bw_{t}, \bz_{i_{t}})\right),
	\end{equation} 
	where $i_{t}$ is randomly sampled from $1$ to $n$. Note the convergence rate of $\bw_{t}$ updated by GD and SGD are respectively $\cO(1/t)$ \citep{bubeck2014convex} and $\tilde{\cO}(1/\sqrt{t})$ \citep{shamir2013stochastic}, we have the following two corollaries declare the converged expected excess risks whose proofs appear in Appendix \ref{app:proof of sec3.2}. 
	\begin{corollary}
		Under Assumption \ref{ass:smoothness}-\ref{ass:convexity}, if $\bw_{t}$ is updated by GD in \eqref{eq:GD} with $\eta_{t}=1/L_{1}$, then 
		\begin{equation}
			\small
			\begin{aligned}
				R(\bw_{t}) - R(\bw^{*}) \leq \tilde{\cO}\left(\frac{1}{\sqrt{t}} + \frac{1}{n}\right).	
			\end{aligned}
		\end{equation}
	\end{corollary}
	\begin{corollary}
		Under Assumption \ref{ass:smoothness}-\ref{ass:convexity}, if $\bw_{t}$ is updated by SGD in \eqref{eq:SGD} with $\eta_{t}=D/(L_{1}\sqrt{t + 1})$, then
		\begin{equation}
			\small
			\begin{aligned}
				\mE\left[R(\bw_{t}) - R(\bw^{*})\right] \leq \tilde{\cO}\left(\frac{1}{t^{\frac{1}{4}}} + \frac{1}{n}\right).	
			\end{aligned}
		\end{equation}
	\end{corollary}
	\section{Excess Risk Under Non-Convex Problems}\label{sec:analysis of non-convex function}
	In this section, we present the upper bounds of the expected excess risk of iterates obtained by proper algorithms that approximate local minima under non-convex problems.
	\subsection{Generalization Error Under Non-Convex Problems}\label{sec:Generalization Error for Non-Convex Function}
	In this subsection, we study the generalization error under non-convex problems. Unfortunately, the analysis in Section \ref{sec:testing error of convex function} can not be directly generalized here due to the following reason. The generalization error under convex problems relies on the fact that there exists the \emph{unique} empirical local minima $\bw_{\bS}^{*}$ of $R_{\bS}(\cdot)$ and $\bw_{\bS^{\prime}}^{*}$ of $R_{\bS^{\prime}}(\cdot)$ that concentrate around the \emph{unique} population local minimum $\bw^{*}$ of $R(\cdot)$. Under non-convex problems, there can be many empirical and population local minima. The iterates obtained on $\bS$ and $\bS^{\prime}$ may converge to different empirical local minima away from each other, which invalidates our methods used in convex problems. 
	\par
	Fortunately, we can prove that for each population local minimum, there is an empirical local minimum concentrated around it with high probability. If the generalization upper bound for these local minima is established, and there are no extra empirical local minima, the convergence results of the iterates obtained by proper algorithms imply their generalization ability. Next, we prove our results following this road map.   
	\par
	First, we establish the generalization upper bound for the empirical local minima around the population local minima. 
	According to Proposition \ref{prop:dist of local minimum} in the Appendix \ref{app: Proof in sec4.1}, there are only finite population local minima, thus the non-convex problems with local minima consists of a manifold \citep{liu2022loss} is not considered in this paper. Let $\cM = \{\bw_{1}^{*},\cdots,\bw_{K}^{*}\}$ be the set of population local minima. The number of local minima $K$ may depend on the problem of interest. In many important non-convex problems, $K$ can be quite small, e.g., $K=2$ for PCA \citep{gonen2017fast} and $K = 1$ for robust regression \citep{mei2018landscape}.
	\par
	Then, we notice that the population risk is strongly convex in $B_{2}(\bw_{k}^{*}, \lambda/(4L_{2}))$. Similar to the scenario under convex problems, we can verify that the empirical risk is locally strongly convex in $B_{2}(\bw_{k}^{*}, (\lambda/4L_{2}))$ with high probability. Next, we consider the following points 
	\begin{equation}\label{eq:wk}\small
		\bw_{\bS,k}^{*} = \mathop{\arg \min}_{\bw \in B_2(\bw_{k}^{*}, \frac{\lambda}{4L_{2}})}R_{\bS}(\bw),
	\end{equation}
	for $k = 1\dots, K$. We show that $\bw_{\bS,k}^{*}$ is a local minimum of $R_{\bS}(\cdot)$ with high probability and present the generalization bound of it. Note that in Theorem \ref{thm:stability and generalization}, $\cA$ can be infeasible. We construct an auxiliary sequence $\bw_{t}$ via an infeasible algorithm. 
	\begin{equation}
		\small
		\begin{aligned}
			\bw_{t + 1} & = \cP_{B_{2}(\bw^{*}_{k}, \frac{\lambda}{4L_{2}})}\left(\bw_{t} - \frac{1}{L_{1}}\nabla R_{\bS}(\bw_{t})\right).
		\end{aligned}
	\end{equation}
	Then, as $\bw_{t}$ locates in $B_{2}(\bw_{k}^{*}, \lambda/(4L_{2}))$ in which $R_{\bS}(\cdot)$ is strongly convex with high probability, we can establish the algorithmic stability bound of the $\bw_{t}$. Combining this with the convergence result of $\bw_{t}$ to $\bw_{\bS,k}^{*}$ implies the generalization ability of $\bw_{\bS,k}^{*}$. The following lemma states our result rigorously. 
	\begin{lemma}\label{lem:genearlization error on minima}
		Under Assumption \ref{ass:smoothness} and \ref{ass:strict saddle}, for $k = 1,\dots, K$, with probability at least
		\begin{equation}
			\small
			1 - \frac{512L_{0}^{2}L_{2}^{2}}{n\lambda^{4}} - \frac{128L_{1}^{2}}{n\lambda^{2}}\left(5\sqrt{\log d}  + \frac{4e\log d }{\sqrt{n}}\right)^{2},
		\end{equation}
		$\bw^{*}_{\bS,k}$\footnote{Please note the definition of $\bw^{*}_{\bS,k}$ in \eqref{eq:wk} which is not necessary to be a local minimum.} is a local minimum of $R_{\bS}(\cdot)$. Moreover, for such $\bw_{\bS,k}^{*}$, we have 
		\begin{equation}
			\small
			\begin{aligned}
				|\mE_{\bS}[R_{\bS}(\bw^{*}_{\bS,k}) - R(\bw^{*}_{\bS,k})]| &\leq \frac{8L_{0}}{n\lambda} \left( L_{0} + \frac{64L_{0}^{2}L_{2}^{2}}{\lambda^{3}}\right)\min\left\{3D, \frac{3\lambda}{2L_2}\right\} \\
				& + \frac{128L_{0}L_{1}^{2}}{n\lambda^{2}}\left(5\sqrt{\log d}  + \frac{4e\log d }{\sqrt{n}}\right)^{2}\min\left\{3D, \frac{3\lambda}{2L_2}\right\}.
			\end{aligned}
		\end{equation}
	\end{lemma}
	\par
	The lemma is proved in Appendix \ref{app:proof of theorem generalization error on minima}., and it guarantees the generalization ability of those empirical local minima located around population local minima. The expected generalization error on these local minima is of order $\tilde{\cO}(1/n)$ as in convex problems. In the sequel, we show that there are no extra empirical local minima expected for these $\bw^{*}_{\bS,k}$ with high probability, under the following mild assumption, which also appears in \citep{mei2018landscape,gonen2017fast}. 
	\begin{assumption}[Strict saddle]\label{ass:strict saddle}
		There exists $\alpha, \lambda > 0$ such that  
		$\|\nabla R(\bw)\| > \alpha$ on the boundary of  $\cW$, and 
		\begin{equation}
			\small
			\|\nabla R(\bw)\| \leq \alpha \Rightarrow |\sigma_{\rm min}(\nabla^2 R(\bw))| \geq \lambda,
		\end{equation}
		where $\sigma_{\rm min}(\nabla^{2}R(\bw))$ is $\nabla^{2}R(\bw)$'s smallest eigenvalue.
	\end{assumption}
	The Assumption \ref{ass:strict saddle} is a generalized version of local strong convexity Assumption \ref{ass:local strong convexity} (can be implied by Assumption \ref{ass:strict saddle}). A vast vary of machine learning problems satisfy this assumption, e.g., generalized linear regression, robust regression, normal mixture model, tensor decomposition, matrix completion, PCA, and ICA \citep{gonen2017fast,mei2018landscape,zhang2017empirical}. We refer readers to \citep{gonen2017fast,ge2015escaping,ge2016matrix,mei2018landscape} for more details of this assumption. 
	\par
	Let $\cM_{\bS} = \{\bw:\bw \ \text{is a local minimum of} \ R_{\bS}(\cdot) \}$ be the set consists of all the local minima of empirical risk $R_{\bS}(\cdot)$. 
	Then we establish the following non-asymptotic probability bound. 
	\begin{lemma}\label{lem:no-extra-minimum}
		Under Assumption \ref{ass:smoothness} and \ref{ass:strict saddle}, for $r=\min\left\{\frac{\lambda}{8L_{2}}, \frac{\alpha^{2}}{16L_{0}L_{1}}\right\}$, with probability at least
		\begin{equation}\label{eq:probability bound}
			\small
			\begin{aligned}
				1 - 2\left(\frac{3D}{r}\right)^{d}\exp\left(-\frac{n\alpha^{4}}{128L_{0}^{4}}\right) & - 4d\left(\frac{3D}{r}\right)^{d}\exp\left(-\frac{n\lambda^{2}}{128L_{1}^{2}}\right) \\
				& - K\left\{\frac{512L_{0}^{2}L_{2}^{2}}{n\lambda^{4}} + \frac{128L_{1}^{2}}{n\lambda^{2}}\left(5\sqrt{\log d}  + \frac{4e\log d }{\sqrt{n}}\right)^{2}\right\},
			\end{aligned}
		\end{equation}
		we have
		\begin{enumerate}
			\item[i:] $\cM_{\bS} = \{\bw^{*}_{\bS,1}, \dots, \bw^{*}_{\bS,K}\}$;
			\item[ii:] for any $\bw \in \cW$, if $\|\nabla R_{\bS}(\bw)\| < \alpha^2/(2L_0)$ and $\nabla^2 R_{\bS}(\bw)\succ -\lambda/2$, then  $\|\bw - \cP_{\cM_{\bS}}(\bw)\| \leq \lambda\|\nabla R_{\bS}(\bw)\| / 4$, 
		\end{enumerate}
		where $\nabla^2 R_{\bS}(\bw)\succ -\lambda/2$ means $\nabla^{2}R_{\bS}(\bw) + \lambda/2\bI_{d}$ is a positive definite matrix.
	\end{lemma}
	The first conclusion in this lemma states that there are no extra empirical local minima except for those $\bw_{\bS,k}^{*}$ concentrate around population local minima, which have guaranteed generalization ability (by Theorem \ref{lem:genearlization error on minima}). The second result is that the empirical risk is ``error bound'' (see \citep{karimi2016linear} for its definition) around its local minima, with high probability. The ``error bound'' is a nice property in optimization \citep{karimi2016linear}. Proof of the lemma is in Appendix \ref{app:proof of theorem no extra minimum}. The probability bound \eqref{eq:probability bound} will appear in the generalization bound of iterates obtained by proper algorithms accounting for the existence of those empirical local minima away from population local minima. We defer the discussion to the bound after providing our generalization upper bound in Theorem \ref{thm:generalization error for non-convex}. 
	\par
	We move forward to derive the generalization upper bound of those iterates obtained by the proper algorithm that approximates the local minima under non-convex problems. Under strict saddle Assumption \ref{ass:strict saddle}, the proper algorithm $\cA$ approximates the second-order stationary point (SOSP) \footnote{$\bw$ is a $(\epsilon, \gamma)$-second-order stationary point (SOSP) if $\|\nabla R_{\bS}(\bw)\| \leq \epsilon$ and $\nabla^{2}R_{\bS}(\bw)\succeq -\gamma$}, that says with probability at least $1 - \delta$ ($\delta$ is a constant that can be arbitrary small), 
	\begin{equation}\label{eq:escape from saddle point}
		\small
		\|\nabla R_{\bS}(\bw_{t})\| \leq \zeta(t), \qquad \nabla^{2}R_{\bS}(\bw_{t})\succeq -\rho(t)
	\end{equation}
	where $\bw_{t}$ is updated by the algorithm $\cA$, and $\zeta(t), \rho(t)\to 0$ (which may have poly-logarithmic dependence on $\delta$ \citep{jin2017escape}) as $t \to\infty$. 
	\par
	To instantiate such proper algorithms, we construct an algorithm that satisfies \eqref{eq:escape from saddle point} in Appendix \ref{app:alg SOSP}. The following theorem establishes a generalization upper bound of $\bw_{t}$ obtained by such $\cA$. 
	\begin{theorem}\label{thm:generalization error for non-convex}
		Under Assumption \ref{ass:smoothness}, \ref{ass:local strong convexity} and \ref{ass:strict saddle}, if $\bw_{t}$ satisfies \eqref{eq:escape from saddle point} and $r$ defined in Lemma \ref{lem:no-extra-minimum}, by choosing $t$ such that $\zeta(t) < \alpha^2/(2L_0)$ and $\rho(t) < \lambda/2$ we have 
		\begin{equation}\label{eq:non-convex gen bound with supurious local minima}
			\small
			\begin{aligned}
				|\mE_{\cA,\bS}\left[R(\bw_{t}) - R_{\bS}(\bw_{t})\right]| & \leq \frac{8L_{0}}{\lambda}\zeta(t) + 2L_{0}D\delta + \frac{2KM}{\sqrt{n}} + \frac{8KL_{0}^{2}}{n\lambda} \\
				& + \left(L_{0}\min\left\{3D,\frac{3\lambda}{2L_{2}}\right\} + 2M\right)\xi_{n, 1} + 2M\xi_{n, 2} \\
				& = \tilde{\cO}\left(\zeta(t) + \frac{1}{\sqrt{n}}\right) \qquad (d / n \leq \cO(1)), 
			\end{aligned}
		\end{equation}
		where 
		\begin{equation}
			\small
			\begin{aligned}
				\xi_{n,1} & =K\left\{\frac{512L_{0}^{2}L_{2}^{2}}{n\lambda^{4}} + \frac{128L_{1}^{2}}{n\lambda^{2}}\left(5\sqrt{\log d}  + \frac{4e\log d }{\sqrt{n}}\right)^{2}\right\},
			\end{aligned}
		\end{equation}
		and
		\begin{equation}
			\small
			\begin{aligned}
				\xi_{n,2} \!\!=\!\! 2\left(\frac{3D}{r}\right)^{d}\!\!\exp\left(\!\!-\frac{n\alpha^{4}}{128L_{0}^{4}}\right) \!+\! 4d\left(\frac{3D}{r}\right)^{d}\!\!\exp\left(\!\!-\frac{n\lambda^{2}}{128L_{1}^{2}}\right).
			\end{aligned}
		\end{equation}
		If with probability at least $1-\delta^{\prime}$ ($\delta^{\prime}$ can be arbitrary small), $R_{\bS}(\cdot)$ has no spurious local minimum, then
		\begin{equation}\label{eq:non-convex gen bound without supurious local minima}
			\small
			\begin{aligned}
				|\mE_{\cA,\bS}\left[R(\bw_{t}) - R_{\bS}(\bw_{t})\right]|& \leq \frac{8L_{0}}{\lambda}\zeta(t) + 2L_{0}D\delta + 6M\delta^{\prime} + \frac{8(K+4)L_{0}^{2}}{n\lambda} \\
				& + \left(\frac{(K+4)L_{0}}{K}\min\left\{3D,\frac{3\lambda}{2L_{2}}\right\} + 6M\right)\xi_{n,1} + 6M\xi_{n,2} \\
				& = \tilde{\cO}\left(\zeta(t) + \frac{1}{n}\right) \qquad (d / n \leq \cO(1)) .
			\end{aligned}
		\end{equation}
	\end{theorem}
	\par
	This theorem is proved in Appendix \ref{app:proof of generalization error for non-convex}, and it provides upper bounds of the expected generalization error of iterates obtained by any proper algorithm that approximates SOSP. We present an explanation of each term in it as follows. 
	The $2DL_{0}\delta$ is of order $\cO(1/\sqrt{n})$ or $\cO(1/n)$ as we take the corresponded $\delta=1/\sqrt{n}$ or $1/n$, and $8L_{0}\zeta(t)/\lambda$ can be arbitrary small if we take a sufficiently large $t$. Since $\xi_{n,1}$ is of order $\tilde{\cO}(1/n)$, we next explore $\xi_{n,2}$. The leading term in $\xi_{n,2}$ is 
	\begin{equation}\label{eq:xi2}
		\small
		\begin{aligned}
			4d\left(\frac{3D}{r}\right)^{d}\exp\left(-\frac{n\lambda^{2}}{128L_{1}^{2}}\right) = \exp\left(\log{4d} + d\log\left(\frac{3D}{r}\right) - \frac{n\lambda^{2}}{128L_{1}^{2}}\right).
		\end{aligned}
	\end{equation}
	If $d$ is large enough to make $\log 4d \leq d\log(3D/r)$, then $\xi_{n,2} \leq \exp(-c_{2}n( c_{1} - \frac{d}{n}))$, 
	where $c_{1} = \lambda^{2}/(256L_{1}^{2}\log(3D/r))$ and $c_{2} = 2\log(3D/r)$. Thus $\xi_{n,2} \ll \tilde{\cO}(1 / n)$ provided by $d/n < c_{1}$. In this case, the $2KM/\sqrt{n}$ appears in bound \eqref{eq:non-convex gen bound with supurious local minima} implies it is of order $\tilde{\cO}(1/\sqrt{n})$, even under high-dimensional problems such that $d$ is in the same order of $n$. The $K$ can be small here for many non-convex problems, as previously discussed. Moreover, the bound \eqref{eq:non-convex gen bound without supurious local minima} improves the result in \eqref{eq:non-convex gen bound with supurious local minima} to $\tilde{\cO}(1/n)$, under the condition of empirical risk has no spurious local minima with high probability (i.e. $\delta^{\prime}\leq \tilde{\cO}(1 / n)$). The condition has been proven to be satisfied by many important non-convex optimization problems e.g., PCA \citep{gonen2017fast}, matrix completion \citep{ge2016matrix}, and over-parameterized neural network \citep{kawaguchi2016deep,allen2019convergence,du2019gradient}. 
	\par
	\paragraph{Comparison.} Under the extra strictly saddle Assumption \ref{ass:strict saddle}, our bounds (no matter whether imposing the no spurious local minima assumption) improve the classical results of order $\cO(\sqrt{d/n})$ based on the uniform convergence theory \citep{shalev2009stochastic} or the one of order $\cO(t^{c}/n)$ for a positive $c$ \citep{hardt2016train,yuan2019stagewise} based on algorithmic stability. \citep{gonen2017fast} get the result of order $\tilde{\cO}(d/n)$ under the same Assumptions \ref{ass:smoothness} and \ref{ass:strict saddle} imposed in this paper. However, their bound has a linear dependence on $d$, thus can not be non-vacuous like ours when $d$ is in the same order of $n$.   
	\par
	Specifically, if the parameter space satisfies some sparsity conditions \citep{bickel2009simultaneous, zhang2010nearly, javanmard2014confidence, javanmard2018debiasing, fan2017estimation, wainwright2019} , we can extrapolate Theorem \ref{thm:generalization error for non-convex} to ultrahigh-dimensional problem such that $d\gg n$. For example, suppose the parameter space $\cW$ is contained in a $\ell_{1}$-ball, i.e., $\|\bw_{1} - \bw_{2}\|_{1}\leq D^{\prime}$ for some positive $D^{\prime}$. Note that the covering number (defined in \citep{wainwright2019}) of polytopes (Corollary 0.0.4 in \citep{vershynin2018}) is much smaller than that of $\ell_{2}$-ball. Then, applying the similar proof of Theorem \ref{thm:generalization error for non-convex} establishes the same upper bound of generalization error w.r.t. $\bw_{t}$ with $\xi_{n,2}$ in Theorem \ref{thm:generalization error for non-convex} replaced by
	\begin{equation}
		\small
		\begin{aligned}
			2(2d)^{(2D^{\prime}/r)^{2}+1}\exp\left(-\frac{n\alpha^{4}}{128L_{0}^{4}}\right) + 
			2(2d)^{(2D^{\prime}/r)^{2}+2}\exp\left(-\frac{n\lambda^{2}}{128L_{1}^{2}}\right) \ll \tilde{\cO}\left(\frac{1}{n}\right), 
		\end{aligned}
	\end{equation} 
	where the much smaller relationship is valid as long as $\log(d) / n \to 0$. 
	
	\subsection{Excess Risk Under Non-Convex Problems}\label{sec: testing error for non-convex function}
	In this subsection, we establish upper bounds for the expected excess risk of iterates obtained by proper algorithms under non-convex problems. In contrast to convex optimization, the proper algorithm under non-convex problems is not guaranteed to find the global minimum, as it only approximates SOSP. Hence the optimization error may not vanish as in Theorem \ref{thm:excess risk for convex loss}. 
	The following theorem proved in Appendix \ref{app:Proof in sec4.3} establishes an upper bound of the expected excess risk.
	\begin{theorem}\label{thm:excess risk for non-convex}
		Under Assumption \ref{ass:smoothness}, \ref{ass:local strong convexity} and \ref{ass:strict saddle}, if $\bw_{t}$ satisfies \eqref{eq:escape from saddle point}, by choosing $t$ in \eqref{eq:escape from saddle point} such that $\zeta(t) < \alpha^2/(2L_0)$ and $\rho(t) < \lambda/2$, we have 
		\begin{equation}\label{eq:excess error for non-convex}
			\small
			\begin{aligned}
				\mE_{\cA,\bS}\left[R(\bw_{t}) - R(\bw^{*})\right] & \leq \frac{4L_{0}}{\lambda}\zeta(t) + L_{0}D\delta + \frac{2KM}{\sqrt{n}} \\
				& + \frac{8KL_{0}^{2}}{n\lambda} + \left(L_{0}\min\left\{3D,\frac{3\lambda}{2L_{2}}\right\} + 2M\right)\xi_{n, 1} + 2M\xi_{n, 2}\\
				& + \mE_{\cA,\bS}[R_{\bS}(\cP_{\cM_{\bS}}(\bw_{t})) - R_{\bS}(\bw_{\bS}^{*})] \\
				&  = \mE_{\cA,\bS}[R_{\bS}(\cP_{\cM_{\bS}}(\bw_{t})) - R_{\bS}(\bw_{\bS}^{*})] + \tilde{\cO}\left(\zeta(t) + \frac{1}{\sqrt{n}}\right) \qquad (d / n \leq \cO(1)),
			\end{aligned}
		\end{equation}
		where $\bw^{*}$ is the global minimum of the population risk.
		If with probability at least $1-\delta^{\prime}$ ($\delta^{\prime}$ can be arbitrary small), $R_{\bS}(\cdot)$ has no spurious local minimum, then
		\begin{equation}\label{eq:excess error for non-convex without spurious local minima}
			\small
			\begin{aligned}
				\mE_{\cA,\bS}\left[R(\bw_{t}) - R(\bw^{*})\right] & \leq \frac{4L_{0}}{\lambda}\zeta(t) + L_{0}D\delta + 8M\delta^{\prime} + \frac{8(K+4)L_{0}^{2}}{n\lambda} \\
				& + \left(\frac{(K+4)L_{0}}{K}\min\left\{3D,\frac{3\lambda}{2L_{2}}\right\} + 8M\right)\xi_{n,1} + 8M\xi_{n,2} \\
				& = \tilde{\cO}\left(\zeta(t) + \frac{1}{n}\right) \qquad (d / n \leq \cO(1)),
			\end{aligned}
		\end{equation}
		where $\xi_{n,1}$ and $\xi_{n, 2}$ are defined in Theorem \ref{thm:generalization error for non-convex}, and $\bw_{\bS}^{*}$ is the global minimum of $R_{\bS}(\cdot)$.
	\end{theorem}
	\par
	From the discussions in the last section, the bound \eqref{eq:excess error for non-convex} and \eqref{eq:excess error for non-convex without spurious local minima} become $\cO(1/\sqrt{n})$ and $\tilde{\cO}(1/n)$, respectively, when $d$ is in the same order of $n$ and $t\rightarrow\infty$. Besides that, in \eqref{eq:excess error for non-convex}, expected for the order of convergence rate $\cO(\zeta(t))$ and the generalization bound of order $\tilde{\cO}(1/\sqrt{n} + \exp(-c_{2}n(c_{1} - d/n))$ \footnote{The difference in the coefficients of the convergence rate term $\zeta(t)$ between the bounds in Theorem \ref{thm:generalization error for non-convex} and \ref{thm:excess risk for non-convex} is due to a technique issue and not essential.}, there is an extra $\mE_{\cA,\bS}[R_{\bS}(\cP_{\cM_{\bS}}(\bw_{t})) - R_{\bS}(\bw_{\bS}^{*})]$ in the bound \eqref{eq:excess error for non-convex}, compared with the result of convex problems in Theorem \ref{thm:excess risk for convex loss}. This is the gap between the empirical global minimum and the algorithmic approximated empirical local minimum. The gap seems necessary as the proper algorithm is not guaranteed to find the global minima, and if so, the gap becomes zero.    
	\par
	The bound \eqref{eq:excess error for non-convex without spurious local minima} of order $\tilde{\cO}(1/n)$ is obtained under empirical risk without spurious local minima, which is proven to be hold on many important non-convex problems e.g., PCA \citep{gonen2017fast}, matrix completion \citep{ge2016matrix}, and over-parameterized neural network \citep{kawaguchi2016deep,allen2019convergence,du2019gradient,zou2020gradient}. 
	\section{Related Works}\label{sec:related work}
	\paragraph{Generalization} 
	The generalization error is the gap between the model's performance on training and unseen test data. One of the central tools to bound the generalization error in statistical learning is uniform convergence theory. However, this method is unavoidably related to the capacity of hypothesis space e.g., VC dimension \citep{blumer1989learnability,cherkassky1999model,opper1994learning,guyon1993automatic}, Rademacher complexity  \citep{bartlett2002rademacher,mohri2009rademacher,neyshabur2018role}, covering number \citep{williamson2001generalization,zhang2002covering,shawe1999generalization}, or entropy integral \citep{wainwright2019}. Thus, these results are not well suited for high-dimensional hypothesis spaces, which makes the mentioned measures to be large. 
	\par
	The generalization error of the iterates obtained by some algorithms, e.g., GD or SGD, is often of more interest. There are plenty of papers working on this topic via the tool of algorithmic stability \citep{bousquet2002stability,feldman2019high,bousquet2020sharper,gonen2017fast,shalev2009stochastic}, differential privacy \citep{dwork2015preserving,jung2020anew}, robustness of model \citep{xu2012robustness,sinha2018certifying,yi2021improved}, and information theory \citep{xu2017information,steinke2020reasoning,bu2020tightening}. However, these tools either depend heavily on algorithm implementation (algorithmic stability and information theory) or require unverifiable conditions (robustness and differential privacy).  
	This paper combines the technique of characterizing empirical loss landscape and algorithmic stability to explore the generalization under both convex and non-convex problems. Our methods develop a new way to use algorithmic stability, which can be applied without restrictions on the algorithm, learning rate, and the number of iterations. 
	
	\paragraph{Optimization} Results in this paper are related to both convex and non-convex problems. 
	\par
	For convex problems, \cite{bubeck2014convex} summarizes most of the classical algorithms in convex optimization. Some other novel methods \citep{johnson2013accelerating,roux2012stochastic,nguyen2017sarah} with lower computational complexity have also been extensively explored. Recently, the non-convex optimization has attracted quite a lot attentions owing to the development of deep learning \citep{he2016deep,vaswani2017attention}. But most of the existing algorithms \citep{ghadimi2013stochastic,arora2018theoretical,nguyen2017stochastic,chen2018convergence,fang2018spider,yi2021towards} approximate the first-order stationary point instead of local minima. 
	\par
	Under non-convex problem, the algorithm that approximates SOSP is proper (approximate local minima) in this paper. We refer readers for recent progress in the topic of developing algorithms approximating SOSP to  \citep{ge2015escaping,fang2019sharp,daneshmand2018escaping,jin2017escape,jin2019stochastic,xu2018first,mokhtari2018escaping,zhang2017hitting,jin2018local}. 
	The discussed proper algorithms in this paper have constrained parameter space which is different from the ones in  \citep{bian2015complexity,cartis2018second,mokhtari2018escaping}. To resolve this, we also develop an algorithm that approximates SOSP under our constraints in Appendix \ref{app:alg SOSP}.
	\paragraph{Excess Risk} A straightforward way to characterize the excess risk is by controlling the generalization and optimization errors, respectively, as we did in this paper. Thus, for this problem, the used tools are similar to the ones in analyzing generalization, e.g., uniform convergence theory \citep{vapnik1999nature,zhang2017empirical,feldman2016generalization}, algorithmic stability \citep{hardt2016train,charles2018stability,chen2018stability,yuan2019stagewise,deng2020toward}, information theory \citep{negrea2019information,neu2021information}. However, the discussed drawbacks of these tools also appeared.
	Our results are built upon the combination of characterizing empirical risk's landscape and algorithmic stability. Moreover, they are dimensional insensitive, independent of algorithm's implementation, and they improve the order of existing results under both convex and non-convex problems.  
	\section{Conclusion}
	This paper provides a unified analysis of the expected excess risk of models trained by proper algorithms under convex and non-convex problems. Our primary techniques are algorithmic stability and the non-asymptotic characterization of the empirical risk's landscape. 
	\par
	Under the conditions of local strong convexity around population local minima and some other mild regularity conditions, we establish the upper bounds of the expected excess risk in the order of $\tilde{\cO}\left(1/n\right)$ and $\tilde{\cO}(1/\sqrt{n})$ (can be improved to $\tilde{\cO}(1/n)$ when empirical risk has no spurious local minima with high probability) under convex and non-convex problems respectively. 
	\par
	The presented results improve the existing results in many aspects. For convex problems, our results improve the standard excess risk bound of order $\cO(\sqrt{1/n})$ \citep{hardt2016train} to $\tilde{\cO}(1/n)$ under locally convex assumption. For non-convex problems, our results significantly improve the standard uniform convergence bound in the order of $\cO(\sqrt{d/n})$ \citep{shalev2009stochastic} when $d/n$ is smaller than a universal constant. Moreover, our results can be generally applied to algorithms that approximate local minima, and they have no restrictions on the algorithm, learning rate, and number of iterations.

	\newpage
	
	\bibliography{reference}
	\bibliographystyle{apalike}

	\appendix
	\section{Proof of Theorem \ref{thm:stability and generalization}}\label{app:proof of stability and generalization}
\begin{proof}
	Recall that $\{\bz_{1}, \cdots, \bz_{n},\bz_{1}^{\prime},\cdots, \bz_{n}^{\prime}\}$ are $2n$ i.i.d samples from the target population, $\bS=\{\bz_{1}, \cdots, \bz_{n}\}$, $\bS^{i}=\{\bz_{1},\cdots, \bz_{i - 1}, \bz_{i}^{\prime}, \bz_{i + 1}, \cdots, \bz_{n}\}$, and $\bS^{\prime} = \bS^{1}$. We have 
	\begin{equation}
		\small
		\begin{aligned}
			\mE_{\cA,\bS}\left[R(\cA(\bS)) - R_{\bS}(\cA(\bS))\right] & = \mE_{\cA, \bS, \bz}\left[\frac{1}{n}\sum\limits_{i=1}^{n}\left(f(\cA(\bS), \bz) - f(\cA(\bS), \bz_{i})\right)\right] \\
			& = \mE_{\cA,\bS, \bS^{i}}\left[\frac{1}{n}\sum\limits_{i=1}^{n}\left(f(\cA(\bS^{i}), \bz_{i}) - f(\cA(\bS), \bz_{i})\right)\right] \\
			& = \frac{1}{n}\sum\limits_{i=1}^{n}\mE_{\cA,\bS, \bS^{i}}\left[f(\cA(\bS^{i}), \bz_{i}) - f(\cA(\bS), \bz_{i})\right].
		\end{aligned}
	\end{equation}
	Thus
	\[\small
	\begin{split}
		|\mE_{\cA,\bS}\left[R(\cA(\bS)) - R_{\bS}(\cA(\bS))\right]|
		&\leq \frac{1}{n}\sum\limits_{i=1}^{n}\mE_{\bS, \bS^{i}}\left|\mE_{\cA}\left[f(\cA(\bS^{i}), \bz_{i}) - f(\cA(\bS), \bz_{i})\right]\right| \\
		&\leq \mE_{\bS, \bS^{\prime}}\left[\sup_{\bz}|\mE_{\cA}[f(\cA(\bS^{\prime}), \bz) - f(\cA(\bS), \bz)]|\right] \\
		& \leq \epsilon,
	\end{split}
	\]
	where the last inequality is due to the $\epsilon$-uniform stability. 
\end{proof}
\section{Proofs in Section \ref{sec:testing error of convex function}}
Throughout this and the following proofs, for any symmetric matrix $\bA$, we denote its smallest and largest eigenvalue by $\sigma_{\rm min}(\bA)$ and $\sigma_{\rm max}(\bA)$, respectively.
\subsection{Proofs in Section \ref{sec:stability of convex function}}\label{app: proof in 3.1}
Before providing the proof of Theorem \ref{thm:stability of convex function}, we need several lemmas. First we define two ``good events" 
\begin{equation}
	\small
	\begin{aligned}
		E_{1} & = \left\{\|\nabla R_{\bS}(\bw^{*})\| \leq \frac{\lambda^{2}}{16L_{2}}, \|\nabla R_{\bS^{\prime}}(\bw^{*})\| \leq \frac{\lambda^{2}}{16L_{2}}\right\} \\
		E_{2} & = \left\{\|\nabla^{2} R_{\bS}(\bw^{*}) - \nabla^{2} R(\bw^{*})\|\leq \frac{\lambda}{4}, \|\nabla^{2} R_{\bS^{\prime}}(\bw^{*}) - \nabla^{2} R(\bw^{*})\|\leq \frac{\lambda}{4}\right\}
	\end{aligned}
\end{equation}
The following lemma is based on the fact that on event $E_{1}\bigcap E_{2}$ the empirical global minimum is around the population global minimum. \begin{lemma}\label{lem:upper bound on two minimums}
	Under Assumptions \ref{ass:smoothness}-\ref{ass:convexity}, there exists global minimum $\bw_{\bS}^{*}$ and $\bw_{\bS^{\prime}}^{*}$ of $R_{\bS}(\cdot)$ and $R_{\bS^{\prime}}(\cdot)$ such that   
	\begin{equation}\label{eq:two minimum bound}
		\small
		\mE\left[\|\bw_{\bS}^{*} - \bw_{\bS^{\prime}}^{*}\| \textbf{1}_{E_{1}\bigcap E_{2}}\right] \leq \frac{8L_{0}}{n\lambda},
	\end{equation}
	where $\textbf{1}_{(\cdot)}$ is the indicative function and $\bw^{*}$ is the sole global minimum of $R(\cdot)$.
\end{lemma}
\begin{proof} 
	To begin with, we show $R_{\bS}(\cdot)$ is locally strongly convex around $\bw^{*}$ with high probability. Then, by providing that there exists $\bw_{\bS}^{*}$ and $\bw_{\bS^{\prime}}^{*}$ locates in the region, we get the conclusion. 
	\par
	We claim that if the event $E_{1}\bigcap E_{2}$	happens, then $\nabla^{2} R_{\bS}(\bw)\succeq \frac{\lambda}{2}$ for any $\bw\in B_{2}(\bw^{*}, \frac{\lambda}{4L_{2}})$. Since  
	\begin{equation}\small
		\begin{aligned}
			\sigma_{\rm min}(\nabla^{2}R_{\bS}(\bw)) & = \sigma_{\rm min}\left(\nabla^{2}R_{\bS}(\bw) - \nabla^{2}R_{\bS}(\bw^{*}) + \nabla^{2}R_{\bS}(\bw^{*}) - \nabla^{2}R(\bw^{*}) + \nabla^{2}R(\bw^{*})\right) \\
			& \geq \sigma_{\rm min}(\nabla^{2}R(\bw^{*})) - \|\nabla^{2}R_{\bS}(\bw) - \nabla^{2}R_{\bS}(\bw^{*})\| - \|\nabla^{2}R_{\bS}(\bw^{*}) - \nabla^{2}R(\bw^{*})\| \\
			& \geq \lambda - L_{2}\|\bw - \bw^{*}\| - \frac{\lambda}{4} \geq \frac{\lambda}{2},
		\end{aligned}
	\end{equation}
	where the last inequality is due to the Lipschitz Hessian and event $E_{2}$. After that, we show that both $\bw^{*}_{\bS}$ and $\bw_{\bS^{\prime}}^{*}$ locate in $B_{2}(\bw^{*}, \frac{\lambda}{4L_{2}})$, when $E_{1}, E_{2}$ hold. Let $\bw = \gamma\bw_{\bS}^{*} + (1 - \gamma)\bw^{*}$, with $\gamma = \frac{\lambda}{4L_{2}\|\bw_{\bS}^{*} - \bw^{*}\|}$ then
	\begin{equation}
		\small
		\|\bw - \bw^{*}\| = \gamma\|\bw_{\bS}^{*} - \bw^{*}\|.
	\end{equation}
	One can see $\bw\in S_{2}(\bw^{*}, \frac{\lambda}{4L_{2}})$. Thus by the strong convexity, 
	\begin{equation}
		\small
		\begin{aligned}
			\|\bw - \bw^{*}\|^{2} \leq \frac{4}{\lambda}(R_{\bS}(\bw) - R_{\bS}(\bw^{*}) + \langle \nabla R_{\bS}(\bw^{*}), \bw - \bw^{*}\rangle) < \frac{4}{\lambda}\|\nabla R_{\bS}(\bw^{*})\|\|\bw - \bw^{*}\|,
		\end{aligned}
	\end{equation}
	where the last inequality is due to the convexity such that 
	\begin{equation}
		\small
		R_{\bS}(\bw) - R_{\bS}(\bw^{*}) = R_{\bS}(\gamma\bw_{\bS}^{*} + (1 - \gamma)\bw^{*}) - R_{\bS}(\bw^{*}) \leq \gamma(R_{\bS}(\bw_{\bS}^{*}) - R_{\bS}(\bw^{*})) < 0
	\end{equation}
	and Schwarz inequality. Then,
	\begin{equation}
		\small
		\frac{\lambda}{4}\|\bw - \bw^{*}\| = \frac{\lambda^{2}}{16L_{2}\|\bw^{*}_{\bS} - \bw^{*}\|}\|\bw^{*}_{\bS} - \bw^{*}\| = \frac{\lambda}{16L_{2}^{2}} < \|\nabla R_{\bS}(\bw^{*})\|,
	\end{equation}
	which leads to a contraction to event $E_{1}$. Thus, we conclude that $\bw_{\bS}^{*}\in B_{2}(\bw^{*}, \frac{\lambda}{4L_{2}})$. Identically, one can verify that $\bw_{\bS^{\prime}}^{*}\in B_{2}(\bw^{*}, \frac{\lambda}{4L_{2}})$.
	\par
	Since both $\bw_{\bS}^{*}$ and $\bw_{\bS^{\prime}}^{*}$ are in $B_{2}(\bw^{*}, \frac{\lambda}{4L_{2}})$ on event $E_{1}\bigcap E_{2}$, $\bS$ and $\bS^{\prime}$ differs in $\bz_{1}$, then we have 
	\begin{equation}\label{eq:bounded distance of optima}
		\small
		\begin{aligned}
			\|\bw_{\bS}^{*} - \bw_{\bS^{\prime}}^{*}\| & \leq \frac{4}{\lambda}\|\nabla R_{\bS}(\bw_{\bS^{\prime}}^{*})\| \\
			& = \frac{4}{\lambda}\left\|\frac{1}{n}\sum\limits_{\bz\in\bS}\nabla f(\bw_{\bS^{\prime}}^{*}, \bz)\right\| \\
			& = \frac{4}{n\lambda}\left\|\nabla f(\bw_{\bS^{\prime}}^{*}, \bz_{1})  - \nabla f(\bw_{\bS^{\prime}}^{*}, \bz^{\prime}_{1})\right\|\\
			& \leq \frac{8L_{0}}{n\lambda},
		\end{aligned}
	\end{equation}
	where the last equality is due to $\bw_{\bS^{\prime}}^{*}$ is the minimum of $R_{\bS^{\prime}}(\cdot)$. 
	The lemma follows from the fact
	\begin{equation}
		\small
		\mE\left[\|\bw_{\bS}^{*} - \bw_{\bS^{\prime}}^{*}\| \textbf{1}_{E_{1}\bigcap E_{2}}\right] \leq \frac{8L_{0}}{n\lambda}\bbP(E_{1}\bigcap E_{2}) \leq \frac{8L_{0}}{n\lambda}.
	\end{equation}
\end{proof}

Next, we show that the ``good event" happens with high probability.
\begin{lemma}\label{lem:good event prob}
	Under Assumption \ref{ass:smoothness}, 
	\begin{equation}\label{eq:E1c+E2c prob bound}
		\small
		\bbP(E_{1}^{c}\bigcup E_{2}^{c}) \leq \bbP(E^{c}_{1}) + \bbP(E^{c}_{2}) \leq
		\frac{512L_{0}^{2}L_{2}^{2}}{n\lambda^{4}} + \frac{128L_{1}^{2}}{n\lambda^{2}}\left(5\sqrt{\log d}  + \frac{4e\log d }{\sqrt{n}}\right)^{2},
	\end{equation}
	where $E_{k}^{c}$ is the complementary of $E_{k}$ for $k=1,2$.
\end{lemma}
\begin{proof}
	By Assumption \ref{ass:smoothness}, we have $ \|\nabla f(\bw, \bz)\|\leq L_{0}$ and $\|\nabla^2 f(\bw, \bz)\|\leq L_{1}$ for any $\bw \in \cW$ and $\bz$. Thus $\mE_{\bz}[\|\nabla f(\bw^{*}, \bz)\|^2] \leq L_{0}^2$ and $\mE[\|\nabla^2 f(\bw, \bz) - \nabla^2 R(\bw)\|^2] \leq 4L_{1}^2$. For $E^{c}$, a simple Markov's inequality implies   
	\begin{equation}\small\label{eq:event probability bound}
		\small
		\begin{aligned}
			\bbP(E^{c}) & = \bbP\left(E^{c}_{1}\bigcup E^{c}_{2}\right)
			\leq \bbP(E^{c}_{1}) + \bbP(E^{c}_{2})\\
			& = 2\bbP\left(\|\nabla R_{\bS}(\bw^{*})\| > \frac{\lambda^{2}}{16L_{2}}\right) + 2\bbP\left(\|\nabla^{2} R_{\bS}(\bw^{*}) - \nabla^{2} R(\bw^{*})\|> \frac{\lambda}{4}\right)\\
			& \leq \frac{512L_{2}^{2}}{\lambda^{4}}\mE\left[\|\nabla R_{\bS}(\bw^{*})\|^{2}\right] + \frac{32}{\lambda^{2}}\mE\left[\|\nabla^{2} R_{\bS}(\bw^{*}) - \nabla^{2} R(\bw^{*})\|^{2}\right].
		\end{aligned}
	\end{equation}
	By similar arguments as in the proof of Lemma 7 in \citep{zhang2013communication}, we have
	\begin{equation}
		\small
		\mE\left[\|\nabla R_{\bS}(\bw^{*})\|^{2}\right] \leq \frac{L_{0}^{2}}{n}, 
	\end{equation}
	and
	\begin{equation}
		\small
		\mE\left[\|\nabla^{2} R_{\bS}(\bw^{*}) - \nabla^{2} R(\bw^{*})\|^{2}\right] \leq \frac{1}{n}\left(10\sqrt{\log d} L_{1} + \frac{8e\log d L_{1}}{\sqrt{n}}\right)^{2}.
	\end{equation}
	Combining these with \eqref{eq:event probability bound}, we have
	\begin{equation}
		\small
		\bbP(E^{c}) \leq \bbP(E^{c}_{1}) + \bbP(E^{c}_{2}) \leq
		\frac{512L_{0}^{2}L_{2}^{2}}{n\lambda^{4}} + \frac{128L_{1}^{2}}{n\lambda^{2}}\left(5\sqrt{\log d}  + \frac{4e\log d }{\sqrt{n}}\right)^{2}.
	\end{equation}	
	Then Lemma \ref{lem:upper bound on two minimums} follows from \eqref{eq:bounded distance of optima} and \eqref{eq:E1c+E2c prob bound}. 
\end{proof}
\par
This lemma shows the fact that there exists empirical global minimum on the training set $\bS$ and $\bS^{\prime}$ concentrate around population global minimum $\bw^{*}$, so the two empirical global minimum are close with each other. Besides that, the empirical risk is locally strongly convex around this global minimum with high probability. 
\par
To present the algorithmic stability, we need to show the convergence of $\bw_{t}$ to $\bw_{\bS}^{*}$ with $\bw_{t}$ trained on the training set $\bS$. However, there is no convergence rate of $\|\bw_{t} - \bw_{\bS}^{*}\|$ under general convex problems, because the quadratic growth condition \footnote{For $f:\bbR^{d}\rightarrow \bbR$, quadratic growth means $\frac{\mu}{2}\|\bw - \bw^{*}\|^{2}\leq f(\bw) - f(\bw^{*})$ for some $\mu>0$, where $\bw^{*}$ is the global minimum.} only holds for strongly convex problems \footnote{For $f:\bbR^{d}\rightarrow \bbR$, strongly convex means $f(\bw_{1}) - f(\bw_{2})\leq \langle\nabla f(\bw_{1}), \bw_{1} - \bw_{2}\rangle - \frac{\lambda}{2}\|\bw_{1} - \bw_{2}\|$ for some $\lambda > 0$ and any $\bw_{1}, \bw_{2}\in\bbR^{d}$.} . Fortunately, the local strong convexity of $R_{\bS}(\cdot)$ and $R_{\bS^{\prime}}(\cdot)$ enables us to upper bound $\|\bw_{t} - \bw_{\bS}^{*}\|$ and $\|\bw_{t}^{\prime} - \bw_{\bS^{\prime}}^{*}\|$ after a certain number of iterations.  
\begin{lemma}\label{lem:convergence of distance for sgd}
	Under Assumption \ref{ass:smoothness} and \ref{ass:convexity}, for any global minimum $\bw_{\bS}^{*}$ of $R_{\bS}(\cdot)$, define event
	\begin{equation}
		\small
		E_{0,r} = \left\{\nabla^{2} R_{\bS}(\bw)\succeq \frac{\lambda}{4}: \forall \bw\in B_{2}(\bw_{\bS}^{*}, r)\right\}
	\end{equation}
	for some $r > 0$ and the training set $\bS$. Then
	\begin{equation}\small
		\mE[\|\bw_{t} - \bw_{\bS}^{*}\|\emph{\textbf{1}}_{E_{0,r}}] \leq \frac{2\sqrt{2}(r + D)}{r\sqrt{\lambda}}\mE\left[R_{\bS}(\bw_{t}) - R_{\bS}(\bw_{\bS}^{*})\right]^{\frac{1}{2}}.
	\end{equation}  
\end{lemma}
\begin{proof}
	Define event 
	\begin{equation}
		\small
		E_{1,r} = \left\{R_{\bS}(\bw_{t}) - R_{\bS}(\bw_{\bS}^{*}) < \frac{\lambda r^{2}}{8}\right\}.
	\end{equation}
	First, we prove on event $E_{0, r}\bigcap E_{1, r}$ we have $\nabla^{2}R_{\bS}(\bw_{t})\succeq \frac{\lambda}{4}$.
	If $E_{0, r}$ holds and $\bw_{t}\in B_{2}(\bw_{\bS}^{*}, r)$, the conclusion is full-filled. On the other hand, if $\bw_{t}\notin B_{2}(\bw_{\bS}^{*}, r)$ and $E_{0, r} \bigcap E_{1, r}$ happens, for any $\bw$ with $\|\bw - \bw_{\bS}^{*}\| = r$, we have 
	\begin{equation}\label{eq: greater on the ball}
		\small
		R_{\bS}(\bw) - R_{\bS}(\bw_{\bS}^{*}) \geq \frac{\lambda r^{2}}{8},
	\end{equation}
	since $E_{0, r}$ holds. Then, let $\bw = \gamma \bw_{t} + (1 - \gamma) \bw_{\bS}^{*}$ with $\gamma = \frac{r}{\|\bw_{t} - \bw_{\bS}^{*}\|}$. Due to $\bw\in B_{2}(\bw_{\bS}^{*}, r)$ and the convexity of $R_{\bS}(\cdot)$, 
	\begin{equation}
		\small
		\begin{aligned}
			R_{\bS}(\bw) - R_{\bS}(\bw_{\bS}^{*}) \leq \gamma(R_{\bS}(\bw_{t}) - R_{\bS}(\bw_{\bS}^{*})) < \frac{\lambda r^{2}}{8},
		\end{aligned}
	\end{equation}
	which leads to a contraction to \eqref{eq: greater on the ball}. Hence, we conclude that on $E_{0, r}\bigcap E_{1, r}$, 
	\begin{equation}\label{eq: bound on itera}
		\small
		\|\bw_{t} - \bw_{\bS}^{*}\| \leq \frac{2\sqrt{2}}{\sqrt{\lambda}}(R_{\bS}(\bw_{t}) - R_{\bS}(\bw_{\bS}^{*}))^{\frac{1}{2}}, 
	\end{equation}
	due to the local strong convexity. With all these derivations, we see that
	\begin{equation}
		\small
		\begin{aligned}
			\mE\left[\|\bw_{t} - \bw_{\bS}^{*}\|\textbf{1}_{E_{0,r}}\right] & = \mE\left[\textbf{1}_{E_{0, r}\bigcap E_{1, r}}\|\bw_{t} - \bw_{\bS}^{*}\|\right] + \mE\left[\textbf{1}_{E_{0, r}\bigcap E_{1, r}^{c}}\|\bw_{t} - \bw_{\bS}^{*}\|\right] \\
			& \overset{a}{\leq} \frac{2\sqrt{2}}{\sqrt{\lambda}} \mE\left[R_{\bS}(\bw_{t}) - R_{\bS}(\bw_{\bS}^{*})\right]^{\frac{1}{2}} + D\bbP(E_{1, r}^{c})\\
			& \leq \frac{2\sqrt{2}}{\sqrt{\lambda}} \mE\left[R_{\bS}(\bw_{t}) - R_{\bS}(\bw_{\bS}^{*})\right]^{\frac{1}{2}} + D\frac{2\sqrt{2}}{r\sqrt{\lambda}}\mE\left[\left(R_{\bS}(\bw_{t}) - R_{\bS}(\bw_{\bS}^{*})\right)^{\frac{1}{2}}\right]\\
			& \leq \frac{2\sqrt{2}(r + D)}{r\sqrt{\lambda}}\mE\left[R_{\bS}(\bw_{t}) - R_{\bS}(\bw_{\bS}^{*})\right]^{\frac{1}{2}},
		\end{aligned}
	\end{equation}
	where $a$ is due to \eqref{eq: bound on itera} and Jesen's inequality. Thus, we get the conclusion. 
\end{proof}
\subsubsection{Proof of Theorem \ref{thm:stability of convex function}}
With all these lemmas, we are now ready to prove the Theorem \ref{thm:stability of convex function}.
\paragraph{Restate of Theorem \ref{thm:stability of convex function}} \emph{
	Under Assumption \ref{ass:smoothness}-\ref{ass:convexity}, we have
	\begin{equation}\small
		\epsilon_{\rm stab}(t) \leq \frac{4\sqrt{2}L_{0}(\lambda + 4DL_{2})}{\lambda^{\frac{3}{2}}}\sqrt{\epsilon(t)} + \frac{8L_{0}}{n\lambda}\left\{L_{0} + \frac{64L_{0}^{2}L_{2}^{2}D}{\lambda^{3}} + \frac{16L_{1}^{2}D}{\lambda}\left(5\sqrt{\log d} + \frac{4e\log d }{\sqrt{n}}\right)^{2}\right\},
	\end{equation}
	where $\epsilon_{\rm stab}(t) = \mE_{\bS, \bS^{\prime}}\left[\sup_{\bz}|\mE_{\cA}[f(\bw_{t}, \bz) - f(\bw^{\prime}_{t}, \bz)]|\right]$ is the stability of the output in the $t$-th step, and $\epsilon(t) = \mE\left[R_{\bS}(\bw_{t}) - R_{\bS}(\bw_{\bS}^{*})\right]$ with $\bw_{\bS}^{*}$ as global minimum of $R_{\bS}(\cdot)$.}
\begin{proof}
	At first glance, 
	\begin{equation}\label{eq:three sum}\small
		\begin{aligned}
			|f(\bw_{t}, \bz) - f(\bw_{t}^{\prime}, \bz)| & \leq L_{0}\|\bw_{t} - \bw_{t}^{\prime}\|\\
			& \leq L_{0}(\|\bw_{t} - \bw_{\bS}^{*}\| + \|\bw_{t}^{\prime} - \bw_{\bS^{\prime}}^{*}\| + \|\bw_{\bS}^{*} - \bw_{\bS^{\prime}}^{*}\|).
		\end{aligned}
	\end{equation}
	We respectively bound these three terms. An upper bound of the third term can be verified by Lemma \ref{lem:upper bound on two minimums}. As proven in Lemma \ref{lem:upper bound on two minimums}, when the two events
	\begin{equation}
		\small
		\begin{aligned}
			E_{1} & = \left\{\|\nabla R_{\bS}(\bw^{*})\| \leq \frac{\lambda^{2}}{16L_{2}}, \|\nabla R_{\bS^{\prime}}(\bw^{*})\| \leq \frac{\lambda^{2}}{16L_{2}}\right\} \\
			E_{2} & = \left\{\|\nabla^{2} R_{\bS}(\bw^{*}) - \nabla^{2} R(\bw^{*})\|\leq \frac{\lambda}{4}, \|\nabla^{2} R_{\bS^{\prime}}(\bw^{*}) - \nabla^{2} R(\bw^{*})\|\leq \frac{\lambda}{4}\right\}
		\end{aligned}
	\end{equation}
	hold, there exists empirical global minimum $\bw_{\bS}^{*}$ and $\bw_{\bS^{\prime}}^{*}$ such that $\nabla^{2} R_{\bS}(\bw_{\bS}^{*}) \succeq \frac{\lambda}{2}$ and $\nabla^{2} R_{\bS^{\prime}}(\bw_{\bS^{\prime}}^{*}) \succeq \frac{\lambda}{2}$. Thus for $\|\bw - \bw_{\bS}^{*}\|\leq \frac{\lambda}{4L_{2}}$, we have 
	\begin{equation}
		\small
		\sigma_{\rm min}(\nabla^{2} R_{\bS}(\bw)) \geq \sigma_{\rm min}(\nabla^{2} R_{\bS}(\bw_{\bS}^{*})) - \|\nabla^{2} R_{\bS}(\bw) - \nabla^{2} R_{\bS}(\bw_{\bS}^{*})\|\geq \frac{\lambda}{2} - L_{2}\|\bw - \bw_{\bS}^{*}\|\geq \frac{\lambda}{4}. 
	\end{equation}
	Hence, we conclude that event $E_{1}\bigcap E_{2}\subseteq E_{\bS}\bigcap E_{\bS^{\prime}}$ with 
	\begin{equation}\small
		\begin{aligned}
			E_{\bS} & = \left\{\nabla^{2} R_{\bS}(\bw)\succeq \frac{\lambda}{4}: \bw\in B_{2}(\bw_{\bS}^{*}, \frac{\lambda}{4L_{2}})\right\}\\
			E_{\bS^{\prime}} & = \left\{\nabla^{2} R_{\bS^{\prime}}(\bw)\succeq \frac{\lambda}{4}: \bw\in B_{2}(\bw_{\bS^{\prime}}^{*}, \frac{\lambda}{4L_{2}})\right\}.
		\end{aligned}
	\end{equation}
	By choosing $r=\frac{\lambda}{4L_{2}}$ in Lemma \ref{lem:convergence of distance for sgd}, 
	\begin{equation}\label{eq:two sample opt bound}
		\small
		\begin{aligned}
			\mE\left[\|\bw_{t} - \bw_{\bS}^{*}\|\textbf{1}_{E_{\bS}} + \|\bw_{t}^{\prime} - \bw_{\bS^{\prime}}^{*}\|\textbf{1}_{E_{\bS^{\prime}}}\right] & \leq \left(\frac{4\sqrt{2}}{\sqrt{\lambda}} + \frac{16\sqrt{2}DL_{2}}{\lambda^{\frac{3}{2}}}\right)\sqrt{\epsilon(t)}.
		\end{aligned}
	\end{equation}
	Note that $E_{\bS}^{c}\bigcup E_{\bS^{\prime}}^{c} \subseteq E_{1}^{c}\bigcup E_{2}^{c}$ and on the event $E_{1}^{c}\bigcup E_{2}^{c}$ we still have 
	\begin{equation}
		\small
			|f(\bw_{t}, \bz) - f(\bw_{t}^{\prime}, \bz)| \leq L_{0}\|\bw_{t} - \bw_{t}^{\prime}\| \leq L_{0}D.
	\end{equation}
	Combining this with \eqref{eq:two minimum bound}, \eqref{eq:E1c+E2c prob bound}, \eqref{eq:three sum} and \eqref{eq:two sample opt bound}, we get the conclusion. 
\end{proof}
\subsection{Proofs in Section \ref{sec:optimization error for convex function}}\label{app:proof of sec3.2}
We now respectively prove the convergence results of GD and SGD w.r.t the terminal point in Section \ref{sec:optimization error for convex function}. The two convergence results imply the conclusion of the two Corollaries in Section \ref{sec:optimization error for convex function}. 
\begin{lemma}
	Under Assumption \ref{ass:smoothness} and \ref{ass:convexity}, we have 
	\begin{equation}
		\small
		R_{\bS}(\bw_{t}) - R_{\bS}(\bw_{\bS}^{*}) \leq \frac{D^{2}L_{1}}{2t},
	\end{equation}
	where $\bw_{t}$ is updated by GD in \eqref{eq:GD} with $\eta_{t} = 1 / L_{1}$.
\end{lemma}
\begin{proof}
	The following descent equation holds due to the Lipschitz gradient,  
	\begin{equation}
		\small
		\begin{aligned}
			R_{\bS}(\bw_{k}) - R_{\bS}(\bw_{k - 1}) \leq \langle\nabla R_{\bS}(\bw_{k - 1}), \bw_{k} - \bw_{k - 1}\rangle + \frac{L_{1}}{2}\|\bw_{k} - \bw_{k - 1}\|^{2} \leq -\frac{1}{2L_{1}}\|\bw_{k} - \bw_{k - 1}\|^{2},
		\end{aligned}
	\end{equation}
	where the last inequality is because the property of projection. On the other hand, we have 
	\begin{equation}\label{eq:descent inner product}
		\small
		\begin{aligned}
			\|\bw_{k} - \bw_{\bS}^{*}\|^{2} & = \|\bw_{k} - \bw_{k - 1} + \bw_{k - 1} - \bw_{\bS}^{*}\|^{2} \\
			& \leq \|\bw_{k} - \bw_{k - 1}\|^{2} + 2\langle\bw_{k} - \bw_{k - 1}, \bw_{k - 1} - \bw_{\bS}^{*}\rangle + \|\bw_{k - 1} - \bw_{\bS}^{*}\|^{2}.
		\end{aligned}
	\end{equation}
	Then, due to the co-coercive of $R_{\bS}(\cdot)$ (see Lemma 3.5 in \citep{bubeck2014convex}), we have 
	\begin{equation}
		\small
		\begin{aligned}
			\sum\limits_{k=1}^{t}\left(R_{\bS}(\bw_{k}) - R_{\bS}(\bw_{\bS}^{*})\right) & \leq \sum\limits_{k=1}^{t}L_{1}\left(\langle \bw_{k - 1} - \bw_{k}, \bw_{k - 1} - \bw_{\bS}^{*}\rangle - \frac{1}{2}\|\bw_{k} - \bw_{k - 1}\|^{2}\right) \\
			& \overset{a}{\leq} \sum\limits_{k=1}^{t} \frac{L_{1}}{2}\left(\|\bw_{k - 1} - \bw_{\bS}^{*}\|^{2} - \|\bw_{k} - \bw_{\bS}^{*}\|^{2}\right) \\
			& \leq \frac{D^{2}L_{1}}{2},
		\end{aligned}
	\end{equation}
	where $a$ is due to \eqref{eq:descent inner product}. The descent equation shows 
	\begin{equation}\small
		R_{\bS}(\bw_{t}) - R_{\bS}(\bw_{\bS}^{*}) \leq \frac{1}{t}\sum\limits_{k=1}^{t}(R_{\bS}(\bw_{k}) - R_{\bS}(\bw_{\bS}^{*})) \leq \frac{D^{2}L_{1}}{2t}.
	\end{equation}
	Thus, we get the conclusion. 
\end{proof}
For SGD, the following convergence result holds for the terminal point. This conclusion is Theorem 2 in \citep{shamir2013stochastic}, we give the proof of it to make this paper self-contained. 
\begin{lemma}\label{lem: convergence result for t}
	Under Assumption \ref{ass:smoothness} and \ref{ass:convexity},
	\begin{equation}
		\small
		\mE[R_{\bS}(\bw_{t}) - R_{\bS}(\bw_{\bS}^{*})] \leq \frac{D(L_{1}^{2} + 2L_{0}^{2})}{2L_{1}\sqrt{t + 1}}(1 + \log{(t + 1)}), 
	\end{equation} 
	for $\bw_{t}$ updated by SGD in \eqref{eq:SGD} with $\eta_{t} = \frac{D}{L_{1}\sqrt{t + 1}}$. 
\end{lemma}
\begin{proof}
	By the convexity of $R_{\bS}(\cdot)$, 
	\begin{equation}\label{eq:loss bound}
		\small
		\begin{aligned}
			& \sum\limits_{k=j}^{t}\mE\left[(R_{\bS}(\bw_{k}) - R_{\bS}(\bw))\right] \leq \sum\limits_{k=j}^{t}\mE\left[\langle \nabla R_{\bS}(\bw_{k}), \bw_{k} - \bw\rangle\right] \\
			&\leq \frac{1}{2D}\sum\limits_{k=j}^{t} L_{1}\sqrt{k + 1}\mE\left[\|\bw_{k} - \bw\|^{2} - \|\bw_{k + 1} - \bw\|^{2} + \frac{D^{2}}{L_{1}^{2}(k + 1)}\|\nabla f(\bw_{k}, \bz_{i_{k}})\|^{2}\right] \\
			& \leq \frac{\sqrt{j + 1}L_{1}}{2D}\|\bw_{j} - \bw\|^{2} + \frac{L_{1}}{2D}\sum\limits_{k = j + 1}^{t}\left(\sqrt{k + 1} - \sqrt{k}\right)\|\bw_k - \bw\|^{2} + \frac{DL_{0}^{2}}{2L_{1}}\sum\limits_{k = j}^{t}\frac{1}{\sqrt{k + 1}} \\
			& \leq \frac{\sqrt{j + 1}L_{1}}{2D}\|\bw_{j} - \bw\|^{2} + \frac{DL_{1}}{2}\left(\sqrt{t + 1} - \sqrt{j + 1}\right) + \frac{DL_{0}^{2}}{2L_{1}}\sum\limits_{k=j}^{t}\frac{1}{\sqrt{k + 1}}
		\end{aligned}
	\end{equation}
	for any $0\leq j\leq t$ and $\bw$, where the second inequality is due to the property of projection. By choosing $\bw = \bw_{j}$, one can see 
	\begin{equation}
		\small
		\begin{aligned}
			\sum\limits_{k=j}^{t}\mE\left[(R_{\bS}(\bw_{k}) - R_{\bS}(\bw_{j}))\right] & \leq \frac{DL_{1}}{2}\left(\sqrt{t + 1} - \sqrt{j + 1}\right) + \frac{DL_{0}^{2}}{L_{1}}(\sqrt{t + 1} - \sqrt{j}) \\
			& \leq \frac{D(L_{1}^{2} + 2L_{0}^{2})}{2L_{1}}(\sqrt{t + 1} - \sqrt{j}).
		\end{aligned}	
	\end{equation}
	Here we use the inequality $\sum_{k=j}^{t}1/\sqrt{k + 1} \leq 2(\sqrt{t + 1} - \sqrt{j})$. Let $S_{j} = \frac{1}{t - j + 1}\sum_{k = j}^{t}\mE\left[R_{\bS}(\bw_{k})\right]$, we have 
	\begin{equation}
		\small
		\begin{aligned}
			(t - j)S_{j + 1} - (t - j + 1)S_{j} & = -\mE[R_{\bS}(\bw_{j})] \leq -S_{j} + \frac{D(L_{1}^{2} + 2L_{0}^{2})}{2L_{1}(t - j + 1)}\left(\sqrt{t + 1} - \sqrt{j}\right) \\
			& \leq -S_{j} + \frac{D(L_{1}^{2} + 2L_{0}^{2})}{2L_{1}(\sqrt{t + 1} + \sqrt{j})}\\
			& \leq -S_{j} + \frac{D(L_{1}^{2} + 2L_{0}^{2})}{2L_{1}\sqrt{t + 1}},
		\end{aligned}
	\end{equation}
	which concludes 
	\begin{equation}
		\small
		S_{j + 1} - S_{j} \leq \frac{D(L_{1}^{2} + 2L_{0}^{2})}{2L_{1}(t - j)\sqrt{t + 1}}. 
	\end{equation}
	Thus 
	\begin{equation}\label{eq:bound on terminal}
		\small
		\mE[R_{\bS}(\bw_{t})] = S_{t} \leq S_{0} + \frac{D(L_{1}^{2} + 2L_{0}^{2})}{2L_{1}\sqrt{t + 1}}\sum\limits_{j=0}^{t - 1}\frac{1}{t - j} \leq S_{0} + \frac{D(L_{1}^{2} + 2L_{0}^{2})}{2L_{1}\sqrt{t + 1}}(1 + \log{(t + 1)}).
	\end{equation}
	Here we use the inequality $\sum_{k=1}^{t}1/k \leq 1 + \log{(t + 1)}$. By taking $\bw = \bw_{\bS}^{*}, j=0$ in \eqref{eq:loss bound} and dividing $t + 1$ in both side of the above equation, we have 
	\begin{equation}\small
		\label{eq:convergence rate convex average}
		S_{0} - R_{\bS}(\bw_{\bS}^{*}) \leq \frac{DL_{1}}{2\sqrt{t + 1}} + \frac{DL_{0}^{2}}{L_{1}\sqrt{t + 1}} = \frac{D(L_{1}^{2} + 2L_{0}^{2})}{2L_{1}\sqrt{t + 1}}.
	\end{equation}
	Combining this with \eqref{eq:bound on terminal}, the proof is completed.
\end{proof}
\par
In convex optimization, the convergence results are usually on the running average scheme i.e., $\bar{\bw}_{t} = (\bw_{0} + \cdots + \bw_{t}) / t$, especially for the randomized algorithm \citep{bubeck2014convex}. In this case, we can take $\bar{\bw}_{t}$ to be the output of the algorithm after $t$ update steps. One can prove the convergence rate of order $\cO(1 /\sqrt{t})$ for $\bar{\bw}_{t}$ from \eqref{eq:convergence rate convex average}. But Lemma \ref{lem: convergence result for t} gives the nearly optimal convergence result for the terminal point $\bw_{t}$ without involving average. 
\par
Combining the convergence result of $\bar{\bw}_{t}$ and our Theorem \ref{thm:excess risk for convex loss}, we conclude that the expected excess risk of $\bar{\bw}_{t}$ obtained by SGD is also upper bounded by $\tilde{\cO}\left(t^{-1/4} + n^{-1}\right)$. 
\section{Proof in Section \ref{sec:analysis of non-convex function}}\label{app: Proof in sec4}
\subsection{Generalization Error on Empirical Local Minima}\label{app: Proof in sec4.1}
To begin our discussion, we give a proposition to the finiteness of population local minima.
\begin{proposition}\label{prop:dist of local minimum}
	Let $\bw^{*}_{i}$ and $\bw^{*}_{j}$ be two local minima of $R(\cdot)$. Then $\|\bw^{*}_{i} - \bw_{j}^{*}\| \geq 4\lambda/L_{2}$.
\end{proposition}
\begin{proof}
	Denote $c = \|\bw^{*}_{i} - \bw_{j}^{*}\|$ and define
	\begin{equation}
		\small
		g(t) = \frac{{\rm d}}{{\rm d}t}R\left(\bv^* + \frac{t}{c}(\bw^* - \bv^*)\right).
	\end{equation}
	Then $g(0) = g(c) = 0$, $g'(0) \geq \lambda$ and $g'(c) \geq \lambda$. By Assumption \ref{ass:smoothness}, $g'(\cdot)$ is Liptchitz continuous with Liptchitz constant $L_2$ and hence $g'(t) \geq \lambda - L_2 \min\{t, c-t\}$ for $t \in [0, c]$. Thus 
	\begin{equation}
		\small
		0 = \int_{0}^{c} g'(t)dt \geq c\lambda - L_2 \int_{0}^{c}\min\{t, c-t\}dt = c\lambda - L_2\frac{c^2}{4}, 
	\end{equation}
	and this implies $c \geq 4\lambda/L_2$.
\end{proof}
Due to the parameter space $\cW\subseteq \bbR^{d}$ is compact set, Heine–Borel Theorem and the above proposition implies that there only exists finite population local minima. The following lemma is needed in the sequel.
\begin{lemma}\label{lem:distance of two points for non-convex function}
	Under Assumption \ref{ass:smoothness}, \ref{ass:local strong convexity}, for any local minimum $\bw^{*}_{k}$ of $R(\cdot)$ with $1\leq k \leq K$ and the two training sets $\bS$ and $\bS^{\prime}$, $\bw_{\bS,k}^{*}$ and $\bw_{\bS^{\prime},k}^{*}$ are empirical local minimum of $R_{\bS}(\cdot)$ and $R_{\bS}(\cdot)$ respectively on the event $E_{k}$, where 
	\begin{equation}\label{eq:event E}
		\small
		E_{k} = E_{1,k}\bigcap E_{2,k}
	\end{equation}
	with
	\begin{equation}
		\small
		\begin{aligned}
			E_{1,k} & = \left\{\|\nabla R_{\bS}(\bw^{*}_{k})\| < \frac{\lambda^{2}}{16L_{2}}, \|\nabla R_{\bS^{\prime}}(\bw^{*}_{k})\| < \frac{\lambda^{2}}{16L_{2}}\right\} \\
			E_{2,k} & = \left\{\|\nabla^{2} R_{\bS}(\bw^{*}_{k}) - \nabla^{2} R(\bw^{*}_{k})\|\leq \frac{\lambda}{4}, \|\nabla^{2} R_{\bS^{\prime}}(\bw^{*}_{k}) - \nabla^{2} R(\bw^{*}_{k})\|\leq \frac{\lambda}{4}\right\},
		\end{aligned}
	\end{equation}
	and 
	\begin{equation}
		\small
		\begin{aligned}
			\bbP\left(E_{k}^{c}\right) \leq \frac{512L_{0}^{2}L_{2}^{2}}{n\lambda^{4}} + \frac{128L_{1}^{2}}{n\lambda^{2}}\left(5\sqrt{\log d}  + \frac{4e\log d }{\sqrt{n}}\right)^{2},
		\end{aligned}
	\end{equation}
	for any $k$. 
\end{lemma}
\begin{proof}
	First, as in the proof of Lemma \ref{lem:upper bound on two minimums}, we have $\nabla^{2} R_{\bS}(\bw)\succeq \frac{\lambda}{2}, \nabla^{2} R_{\bS^{\prime}}(\bw)\succeq \frac{\lambda}{2}$ for $\bw \in B_{2}(\bw^{*}_{k}, \frac{\lambda}{4L_{2}})$ when the event $E_{2,k}$ holds. This is due to $\bw^{*}_{k}$ is a local minimum of $R(\cdot)$. Then for any $\bw\in B_{2}(\bw^{*}_{k}, \frac{\lambda}{4L_{2}})$ with $\|\bw\|=\frac{\lambda}{4L_{2}}$, we have 
	\begin{equation}
		\small
		\begin{aligned}
			R_{\bS}(\bw) - R_{\bS}(\bw^{*}_{k}) & \geq \langle\nabla R_{\bS}(\bw^{*}_{k}), \bw - \bw^{*}_{k}\rangle + \frac{\lambda}{4}\|\bw - \bw^{*}_{k}\|^{2} \\
			& \geq -\|\nabla R_{\bS}(\bw^{*}_{k})\|\|\bw - \bw^{*}_{k}\| + \frac{\lambda}{4}\|\bw - \bw^{*}_{k}\|^{2} \\ 
			& \geq \left(\frac{\lambda}{4}\|\bw - \bw^{*}_{k}\| - \|\nabla R_{\bS}(\bw^{*}_{k})\|\right)\|\bw - \bw^{*}_{k}\| \\
			& = \left(\frac{\lambda^{2}}{16L_{2}} - \|\nabla R_{\bS}(\bw^{*}_{k})\|\right)\|\bw - \bw^{*}_{k}\| > 0,
		\end{aligned}
	\end{equation}
	when event $E_{k}$ holds. Then the function $R_{\bS}(\cdot)$ has at least one local minimum in the inner of $B_2(\bw_{k}^{*}, \frac{\lambda}{4L_{2}})$. Remind that
	\begin{equation}
		\bw_{\bS,k}^{*} = \mathop{\arg \min}_{\bw \in B_2(\bw_{k}^{*}, \frac{\lambda}{4L_{2}})}R_{\bS}(\bw),
	\end{equation}
	then $\bw_{\bS,k}^{*}$ is a local minimum of $R_{\bS}(\cdot)$. Similarly, $\bw_{\bS^{\prime},k}^{*}$ is a local minimum of $R_{\bS^{\prime}}(\cdot)$. Thus we get the conclusion by event probability upper bound \eqref{eq:event probability bound}.
\end{proof}
\par
This lemma implies that $R_{\bS}(\cdot)$ is locally strongly convex around those local minima close to population local minima with high probability. Now, we are ready to give the proof of Lemma \ref{lem:genearlization error on minima}.
\subsubsection{Proof of Lemma \ref{lem:genearlization error on minima}}\label{app:proof of theorem generalization error on minima}
\paragraph{Restate of Lemma \ref{lem:genearlization error on minima}} \emph{Under Assumption \ref{ass:smoothness} and \ref{ass:strict saddle}, for $k = 1,\dots, K$, with probability at least
	\begin{equation}\label{eq:prob local minimum}
		\small
		1 - \frac{512L_{0}^{2}L_{2}^{2}}{n\lambda^{4}} - \frac{128L_{1}^{2}}{n\lambda^{2}}\left(5\sqrt{\log d}  + \frac{4e\log d }{\sqrt{n}}\right)^{2},
	\end{equation}
	$\bw^{*}_{\bS,k}$\footnote{Please note the definition of $\bw^{*}_{\bS,k}$ in \eqref{eq:wk} which is not necessary to be a local minimum.} is a local minimum of $R_{\bS}(\cdot)$. Moreover, for such $\bw_{\bS,k}^{*}$, we have 
	\begin{equation}\label{eq:generalization on local minima}
		\small
		\begin{aligned}
			&|\mE_{\bS}[R_{\bS}(\bw^{*}_{\bS,k}) - R(\bw^{*}_{\bS,k})]| \\
			&\leq \frac{8L_{0}}{n\lambda} \left[ L_{0} + \left\{\frac{64L_{0}^{2}L_{2}^{2}}{\lambda^{3}} + \frac{16L_{1}^{2}}{\lambda}\left(5\sqrt{\log d}  + \frac{4e\log d }{\sqrt{n}}\right)^{2}\right\}\min\left\{3D, \frac{3\lambda}{2L_2}\right\}\right].
		\end{aligned}
\end{equation}}
\begin{proof}
	The first statement of this Theorem follows from Lemma \ref{lem:distance of two points for non-convex function}.
	We prove \eqref{eq:generalization on local minima} via the stability of the proposed auxiliary sequence in Section \ref{sec:Generalization Error for Non-Convex Function}. Let $\cA_{0,k}$ on the training set $\bS$ and $\bS^{\prime}$ be the following auxiliary projected gradient descent algorithm that follow the update rule 
	\begin{equation}\label{eq:infeasible gd}
		\small
		\begin{aligned}
			\bw_{t + 1, k} & = \cP_{B_{2}(\bw^{*}_{k}, \frac{\lambda}{4L_{2}})}\left(\bw_{t, k} - \frac{1}{L_{1}}\nabla R_{\bS}(\bw_{t, k})\right), \\
			\bw_{t + 1, k}^{\prime} & = \cP_{B_{2}(\bw^{*}_{k}, \frac{\lambda}{4L_{2}})}\left(\bw_{t, k}^{\prime} - \frac{1}{L_{1}}\nabla R_{\bS^{\prime}}(\bw_{t, k}^{\prime})\right),
		\end{aligned}
	\end{equation}
	start from $\bw_{0,k} = \bw_{0,k}^{\prime} = \bw^*_{k}$. Although this sequence is infeasible, the generalization bounds based on the stability of it are valid. First note that 
	\begin{equation}
		\small
		\left\|\bw_{t, k} - \bw_{t, k}^{\prime}\right\| \leq \left\|\bw_{t, k} - \bw^{*}_{\bS, k}\right\| + \left\|\bw_{t, k}^{\prime} - \bw^{*}_{\bS^{\prime}, k}\right\| + \left\|\bw^{*}_{\bS, k} - \bw^{*}_{\bS^{\prime}, k}\right\|.
	\end{equation}
	If event $E_{k}$ defined in \eqref{eq:event E} holds, due to Lemma \ref{lem:distance of two points for non-convex function}, $\bw^{*}_{\bS, k}$ and $\bw^{*}_{\bS^{\prime}, k}$ are respectively empirical local minimum of $R_{\bS}(\cdot)$ and $R_{\bS^{\prime}}(\cdot)$, and the two empirical risk are $\lambda/2$-strongly convex in $B_{2}(\bw^{*}_{k}, \frac{\lambda}{4L_{2}})$. As in Lemma \ref{lem:upper bound on two minimums}, we have 
	\begin{equation}
		\small
		\|\bw^{*}_{\bS, k} - \bw^{*}_{\bS^{\prime}, k}\| \leq \frac{8L_0}{n\lambda}
	\end{equation}  
	and
	\begin{equation}
		\small
		\bbP(E^c_{k}) \leq \frac{512L_{0}^{2}L_{2}^{2}}{n\lambda^{4}} + \frac{128L_{1}^{2}}{n\lambda^{2}}\left(5\sqrt{\log d}  + \frac{4e\log d }{\sqrt{n}}\right)^{2}.
	\end{equation}
	By the standard convergence rate of projected gradient descent i.e., Theorem 3.10 in \citep{bubeck2014convex}, we have 
	\begin{equation}
		\small
		\|\bw_{t, k} - \bw^{*}_{\bS, k}\| \leq \exp\left(-\frac{\lambda t}{4L_{1}}\right)\frac{\lambda}{4L_{2}},   
	\end{equation}
	and
	\begin{equation}
		\small
		\|\bw_{t, k}^{\prime} - \bw^{*}_{\bS^{\prime}, k}\| \leq \exp\left(-\frac{\lambda t}{4L_{1}}\right)\frac{\lambda}{4L_{2}}.     
	\end{equation}
	on event $E_{k}$.
	Since $\cA_{0,k}$ is a deterministic algorithm, similar to the proof of Lemma \ref{lem:upper bound on two minimums}, we see
	\begin{equation}
		\small
		\begin{split}
			\epsilon_{\text{stab}}(t) & =\mE_{\bS}\mE_{\bS^{\prime}}\left[\sup_z|f(\bw_{t, k},\bz) - f(\bw_{t, k}^{\prime},\bz)|\right]\\
			& \leq L_{0}\mE_{\bS}\mE_{\bS^{\prime}}\left[\|\bw_{t, k} - \bw_{t, k}^{\prime}\|\right] \\
			& \leq L_{0}\left(\frac{8L_{0}}{n\lambda} + 2\exp\left(-\frac{\lambda t}{4L_{1}}\right)\frac{\lambda}{4L_{2}}\right) \bbP(E_{k}) +
			L_0\min\left\{D, \frac{\lambda}{2L_2}\right\}P(E_{k}^{c})\\
			& \leq L_{0}\left(\frac{8L_{0}}{n\lambda} + \exp\left(-\frac{\lambda t}{4L_{1}}\right)\frac{\lambda}{2L_{2}}\right)\\
			& + L_{0}\left\{\frac{512L_{0}^{2}L_{2}^{2}}{n\lambda^{4}} + \frac{128L_{1}^{2}}{n\lambda^{2}}\left(5\sqrt{\log d}  + \frac{4e\log d }{\sqrt{n}}\right)^{2}\right\}\min\left\{D, \frac{\lambda}{2L_2}\right\}.
		\end{split}
	\end{equation}
	Then, according to Theorem \ref{thm:stability and generalization},
	\begin{equation}
		\small
		|\mE[R_{\bS}(\bw_{t, k}) - R(\bw_{t, k})]| \leq \epsilon_{\text{stab}}(t).
	\end{equation}
	Because
	\begin{equation}
		\small
		\begin{split}
			& |\mE[R_{\bS}(\bw^{*}_{\bS,k}) - R(\bw^{*}_{\bS,k})] - \mE[R_{\bS}(\bw_{t, k}) - R(\bw_{t, k})]|\\
			& \leq 2L_{0}\mE\left[\|\bw_{t, k} - \bw^{*}_{\bS,k}\|\right] \\
			& \leq L_{0}\exp\left(-\frac{\lambda t}{4L_{1}}\right)\frac{\lambda}{2L_{2}}  + 
			L_{0}\left\{\frac{512L_{0}^{2}L_{2}^{2}}{n\lambda^{4}} + \frac{128L_{1}^{2}}{n\lambda^{2}}\left(5\sqrt{\log d}  + \frac{4e\log d }{\sqrt{n}}\right)^{2}\right\}\min\{2D, \frac{\lambda}{L_2}\},
		\end{split}
	\end{equation}
	we have
	\begin{equation}
		\small
		\begin{split}
			|\mE[R_{\bS}(\bw^{*}_{\bS,k}) - R(\bw^{*}_{\bS,k})]| & \leq L_{0}\left(\frac{8L_{0}}{n\lambda} + \exp\left(-\frac{\lambda t}{4L_{1}}\right)\frac{\lambda}{L_{2}}\right)\\
			& + L_{0}\left\{\frac{512L_{0}^{2}L_{2}^{2}}{n\lambda^{4}} + \frac{128L_{1}^{2}}{n\lambda^{2}}\left(5\sqrt{\log d}  + \frac{4e\log d }{\sqrt{n}}\right)^{2}\right\}\min\left\{3D, \frac{3\lambda}{2L_2}\right\}.
		\end{split}
	\end{equation}
	Since $t$ is arbitrary, the inequality in the theorem follows by invoking $t \to \infty$.
\end{proof}
\subsection{No Extra Empirical Local Minima}\label{app:Proof of no extra local minima}
To justify the statement in the main body of this paper, we need to introduce some definitions and results in random matrix theory. We refer readers to \citep{wainwright2019} for more details of this topic. Remind that for any deterministic matrix $\bQ$, $\exp(\bQ)$ is defined as
\begin{equation}
	\small
	\exp(\bQ) = \sum_{k=0}^{\infty}\frac{1}{k!} \bQ^k.
\end{equation}
Then, for random matrix $\bQ$, $\mE[\exp(\bQ)]$ is defined as
\begin{equation}\small
	\mE[\exp(\bQ)] = \sum_{k=0}^{\infty}\frac{1}{k!}\mE \bQ^k.
\end{equation}
\begin{definition}[Sub-Gaussian random matrix]\label{def:sub-Gaussian}
	A zero-mean symmetric random matrix $\bM \in \bbR^{p\times p}$ is Sub-Gaussian with matrix parameters $\bV\in \bbR^{p\times p}$ if 
	\begin{equation}
		\small
		\mE[\exp(c\bM)] \preceq \exp\left(\frac{c^{2}\bV}{2}\right),
	\end{equation}
	for all $c \in \bbR$.
\end{definition}
Note that when $p = 1$, Definition \ref{def:sub-Gaussian} becomes the definition of sub-Gaussian random variable. 

\begin{lemma}\label{lem:subgaussian matrices}
	Let $\theta\in \{-1, +1\}$ be a Rademacher random variable independent of $\bz$. Under Assumption \ref{ass:smoothness}, for any $\bw\in \cW$, $\theta \langle\nabla f(\bw, \bz), \nabla R(\bw)\rangle$ and $\theta\nabla^{2} f(\bw, \bz)$ are Sub-Gaussian with parameter $L_{0}^{4}$ and $L_{1}^2 \bI_{d}$ respectively.
\end{lemma}
\begin{proof}
	According to Assumption \ref{ass:smoothness}, we have $\|\nabla f(\bw, \bz)\|\leq L_{0}$ and $\|\nabla^2 f(\bw, \bz)\| \leq L_{1}$. Because $\nabla R(\bw) = \mE[\nabla f(\bw, \bz)]$, we have $\|\nabla R(\bw)\|\leq L_{0}$ and 
	\begin{equation}\small
		|\langle\nabla f(\bw, \bz),\nabla R(\bw)\rangle| \leq \|\nabla f(\bw, \bz)\|\|\nabla R(\bw)\| \leq L_{0}^{2}.
	\end{equation}
	Hence
	\begin{equation}
		\small
		\begin{split}
			\mE[\exp(c\theta\langle\nabla f(\bw, \bz),\nabla R(\bw)\rangle)\mid \bz]
			& = \sum_{k=0}^{\infty}\frac{(c\langle\nabla f(\bw, \bz),\nabla R(\bw)\rangle)^k}{k!}\mE[\theta^k] \\
			& \overset{a}{=} \sum_{k=0}^{\infty}\frac{(c\langle\nabla f(\bw, \bz),\nabla R(\bw)\rangle)^{2k}}{2k!}\\
			& \leq
			\sum_{k=0}^{\infty}\frac{(c L_{0}^{2})^{2k}}{2k!}\\ 
			& = \exp\left(\frac{L_{0}^{4}c^2}{2}\right),
		\end{split}
	\end{equation}
	where $a$ is due to $\mE\theta^{k} = 0$ for all odd $k$. 
	This implies 
	\begin{equation}
		\small
		\mE[\exp(c\theta\langle\nabla f(\bw, \bz),\nabla R(\bw)\rangle)] \leq \exp\left(\frac{L_{0}^{4}c^2}{2}\right),
	\end{equation}
	then $\theta \langle\nabla f(\bw, \bz), \nabla R(\bw)\rangle$ is Sub-Gaussian with parameter $L_{0}^{4}$. Similar arguments can show $\theta\nabla^{2} f(\bw, \bz)$ is Sub-Gaussian matrix with parameter $L_{1}^2 \bI_{d}$, since $\|\nabla^{2}f(\bw, \bz)\|\leq L_{1}$.
\end{proof}
We have the following concentration results for the gradient and Hessian of empirical risk. 
\begin{lemma}\label{lem:Hoeffding}
	For any $\delta > 0$,
	\begin{equation}
		\small
		\bbP\left(\left|\frac{1}{n}\sum_{i=1}^{n}\langle\nabla f(\bw, \bz_i),\nabla R(\bw)\rangle - \|\nabla R(\bw)\|^{2}\right| \geq \delta\right) \leq 2 \exp\left(-\frac{n\delta^2}{8L_{0}^{4}}\right),
	\end{equation}
	and
	\begin{equation}
		\small
		\bbP\left(\left\|\frac{1}{n}\sum_{i=1}^{n}\nabla^2 f(\bw, \bz_i) - \nabla^2 R(\bw)\right\| \geq \delta\right) \leq 2d \exp\left(-\frac{n\delta^2}{8L_{1}^{2}}\right).
	\end{equation}
\end{lemma}
\begin{proof}
	Note that $\mE[\langle \nabla f(\bw, \bz_i),\nabla R(\bw)\rangle] = \|\nabla R(\bw)\|^{2}$ and 
	$\mE[\nabla^2f(\bw, \bz_i)] = \nabla^2 R(\bw)$.
	According to symmetrization inequality (Proposition 4.1.1 (b) in \citep{wainwright2019}),
	for any $c\in \bbR$
	\begin{equation}
		\small
		\mE\left[\exp\left(\left|\frac{c}{n}\sum_{i=1}^{n}\langle\nabla f(\bw, \bz_i),\nabla R(\bw)\rangle - \|\nabla R(\bw)\|^{2}\right|\right)\right] \leq \mE\left[\exp\left(\left|\frac{2c}{n}\sum_{i=1}^{n}\theta_{i}\langle\nabla f(\bw, \bz_i),\nabla R(\bw)\rangle\right|\right)\right],
	\end{equation}
	and
	\begin{equation}\small
		\begin{aligned}
		& \mE\left[\exp\left(\sup_{\|\bu\| = 1}c\bu^{T}\left(\frac{1}{n}\sum_{i=1}^{n}\nabla^2 f(\bw, \bz_i) - \nabla^2 R(\bw)\right)\bu\right)\right] \\
		& \leq \mE\left[\exp\left( \sup_{\|\bu\| = 1}2c\bu^T\left(\frac{1}{n}\sum_{i=1}^{n}\theta_i\nabla^2 f(\bw, \bz_i)\right)\bu\right)\right],
		\end{aligned}
	\end{equation}
	where $\theta_{1},\dots, \theta_{n}$ are i.i.d. Rademacher random variables independent of $\bz_{1},\dots, \bz_{n}$.
	\par
	Because $\theta_{i}\langle\nabla f(\bw, \bz_i),\nabla R(\bw)\rangle$ is Sub-Gaussian with parameter $L_{0}^{4}$,
	\begin{equation}
		\small
		\begin{split}
			& \mE\left[\exp\left(2c\left|\frac{1}{n}\sum_{i=1}^{n}\theta_{i}\langle\nabla f(\bw, \bz_i)\nabla R(\bw)\rangle\right|\right)\right] \\
			& \leq \mE\left[\exp\left(\frac{2c}{n}\sum_{i=1}^{n}\theta_{i}\langle\nabla f(\bw, \bz_i),\nabla R(\bw)\rangle\right)\right] + \mE\left[\exp\left(-\frac{2c}{n}\sum_{i=1}^{n}\theta_{i}\langle\nabla f(\bw, \bz_i),\nabla R(\bw)\rangle\right)\right]\\
			& \leq 2\exp\left(\frac{2L_{0}^{4}c^2}{n}\right).
		\end{split}
	\end{equation}
	Thus by Markov's inequality, 
	\begin{equation}\small
		\bbP\left(\left|\frac{1}{n}\sum_{i=1}^{n}\langle\nabla f(\bw, \bz_i),\nabla R(\bw)\rangle - \|\nabla R(\bw)\|^{2}\right| \geq \delta\right) \leq 2\exp\left(-c\delta + \frac{2L_{0}^{4}c^2}{n}\right).
	\end{equation}
	Taking $c = n\delta/(4L_{0}^{4})$, the first inequality is full-filled. 
	By the spectral mapping property of the matrix exponential function and Sub-Gaussian property of $\theta_i\nabla^2 f(\bw, \bz_i)$,
	\begin{equation}
		\small
		\begin{split}
			\mE\left[\exp\left( \sup_{\|\bu\| = 1}\bu^T\left(\frac{2c}{n}\sum_{i=1}^{n}\theta_i\nabla^2 f(\bw, \bz_i)\right)\bu\right)\right]
			=&
			\mE\left[\exp\left( \sigma_{\rm max}\left(\frac{2c}{n}\sum_{i=1}^{n}\theta_i\nabla^2 f(\bw, \bz_i)\right)\right)\right]\\
			=& \mE\left[\sigma_{\rm max}\left(\exp\left(\frac{2c}{n}\sum_{i=1}^{n}\theta_i\nabla^2 f(\bw, \bz_i)\right)\right)\right] \\
			\leq& {\rm tr}\left\{\mE\left[\exp\left(\frac{2c}{n}\sum_{i=1}^{n}\theta_i\nabla^2 f(\bw, \bz_i)\right)\right]\right\} \\
			\leq& {\rm tr}\left\{\exp\left(\frac{2L_{1}^{2}c^2\bI_{d}}{n}\right)\right\} \\
			=& d\exp\left(\frac{2L_{1}^{2}c^2}{n}\right).
		\end{split}
	\end{equation}
	Thus 
	\begin{equation}\small
		\begin{split}
			& \mE\left[\exp\left(c\left\|\frac{1}{n}\sum_{i=1}^{n}\nabla^2 f(\bw, \bz_i) - \nabla^2 R(\bw)\right\|\right)\right] \\
			& \leq \mE\left[\exp\left(\sup_{\|\bu\| = 1}\bu^T\left(\frac{c}{n}\sum_{i=1}^{n}\nabla^2 f(\bw, \bz_i) - \nabla^2 R(\bw)\right)\bu\right)\right] \\
			& + \mE\left[\exp\left(\sup_{\|\bu\| = 1}\bu^T\left(\frac{-c}{n}\sum_{i=1}^{n}\nabla^2 f(\bw, \bz_i) - \nabla^2 R(\bw)\right)\bu\right)\right]\\
			& \leq 2d\exp\left(\frac{2L_{1}^{2}c^2}{n}\right).
		\end{split}
	\end{equation}
	Again by Markov's inequality
	\begin{equation}
		\small
		\bbP\left(\left\|\frac{1}{n}\sum_{i=1}^{n}\nabla^2 f(\bw, \bz_i) - \nabla^2 R(\bw)\right\| \geq \delta\right) \leq 2d\exp\left(-c\delta + \frac{2L_{1}^{2}c^2}{n}\right).
	\end{equation}
	Taking $c = n\delta/(4L_{1}^{2})$, the second inequality follows.
\end{proof}
The next lemma establishes Liptchitz property of $\langle\nabla f(\bw, \bz),\nabla R(\bw)\rangle$ and $\|\nabla R(\bw)\|^2$.
\begin{lemma}\label{lem:lip}
	For any $\bw,\bw^{\prime} \in \cW$, we have
	\begin{equation}
		\small
		|\langle\nabla f(\bw, \bz),\nabla R(\bw)\rangle - \langle\nabla f(\bw^{\prime}, \bz),\nabla R(\bw^{\prime})\rangle| \leq 2L_{0}L_{1} \|\bw - \bw^{\prime}\|,
	\end{equation}
	and 
	\begin{equation}
		|\|\nabla R(\bw)\|^2 - \|\nabla R(\bw^{\prime})\|^2| \leq 2L_{0}L_{1} \|\bw - \bw^{\prime}\|.
	\end{equation}
\end{lemma}
\begin{proof}
	We have 
	\begin{equation}
		\small
		\begin{split}
			|\langle\nabla f(\bw, \bz),\nabla R(\bw)\rangle - \langle\nabla f(\bw^{\prime}, \bz),\nabla R(\bw^{\prime})\rangle|
			\leq& |\langle\nabla f(\bw, \bz) - \nabla f(\bw^{\prime}, \bz),\nabla R(\bw)\rangle| \\ 
			+ & |\langle\nabla f(\bw^{\prime}, \bz),(\nabla R(\bw) - \nabla R(\bw^{\prime}))\rangle|\\
			\leq& 2L_{0}L_{1}\|\bw - \bw^{\prime}\|,
		\end{split} 
	\end{equation}
	and 
	\begin{equation}
		\small
		|\|\nabla R(\bw)\|^2 - \|\nabla R(\bw^{\prime})\|^2| = |\langle\nabla R(\bw) - \nabla R(\bw^{\prime}),\nabla R(\bw) + \nabla R(\bw^{\prime}\rangle)| \leq 2L_{0}L_{1}\|\bw - \bw^{\prime}\|
	\end{equation}
	due to the Lipschitz gradient. Hence we get the conclusion. 
\end{proof}
Now, we are ready to provide the proof of Lemma \ref{lem:no-extra-minimum}.
\subsubsection{Proof of Lemma \ref{lem:no-extra-minimum}}\label{app:proof of theorem no extra minimum}
\paragraph{Restate of Lemma \ref{lem:no-extra-minimum}} \emph{Under Assumption \ref{ass:smoothness} and \ref{ass:strict saddle}, for $r=\min\left\{\frac{\lambda}{8L_{2}}, \frac{\alpha^{2}}{16L_{0}L_{1}}\right\}$, with probability at least 
	\begin{equation}
		\small
		\begin{aligned}
			1 - 2\left(\frac{3D}{r}\right)^{d}\exp\left(-\frac{n\alpha^{4}}{128L_{0}^{4}}\right) 
			& - 4d\left(\frac{3D}{r}\right)^{d}\exp\left(-\frac{n\lambda^{2}}{128L_{1}^{2}}\right) \\
			& - K\left\{\frac{512L_{0}^{2}L_{2}^{2}}{n\lambda^{4}} + \frac{128L_{1}^{2}}{n\lambda^{2}}\left(5\sqrt{\log d}  + \frac{4e\log d }{\sqrt{n}}\right)^{2}\right\},
		\end{aligned}
	\end{equation}
	we have
	\begin{enumerate}
		\item[i:] $\cM_{\bS} = \{\bw^{*}_{\bS,1}, \dots, \bw^{*}_{\bS,K}\}$;
		\item[ii:] for any $\bw \in \cW$, if $\|\nabla R_{\bS}(\bw)\| < \alpha^2/(2L_0)$ and $\nabla^2 R_{\bS}(\bw)\succ -\lambda/2$, then  $\|\bw - \cP_{\cM_{\bS}}(\bw)\| \leq \lambda\|\nabla R_{\bS}(\bw)\| / 4$, 
\end{enumerate}
	where $\nabla^2 R_{\bS}(\bw)\succ -\lambda/2$ means $\nabla^{2}R_{\bS}(\bw) + \lambda/2\bI_{d}$ is a positive definite matrix.}
\begin{proof}
	Let 
	\begin{equation}\small
		r = \min\left\{\frac{\lambda}{8L_{2}}, \frac{\alpha^{2}}{16L_{0}L_{1}}\right\},
	\end{equation}
	then according to the result of covering number of $\ell_{2}$-ball and covering number is increasing by inclusion (i.e., \citep{zhang2017empirical}), there are $N \leq (3D/r)^{d}$ points $\bw_1,\dots,\bw_{N}\in \cW$ such that: $\forall\bw \in \cW$, $\exists j\in\{1,\cdots,N\}$, $\|\bw - \bw_{j}\| \leq r$. Then, by Lemma \ref{lem:Hoeffding} and Bonferroni inequality we have
	\begin{equation}\label{gradient prob bound}
		\small
		\begin{split}
			\bbP\left(\max_{1\leq j \leq N}\left|\langle R_{\bS}(\bw_{j}),\nabla R(\bw_{j} )\rangle - \|\nabla R(\bw_{j})\|^2\right| \geq \frac{\alpha^{2}}{4}\right) \leq 2\left(\frac{3D}{r}\right)^{d}\exp\left(-\frac{n\alpha^{4}}{128L_{0}^{4}}\right),
		\end{split}
	\end{equation}
	and
	\begin{equation}\label{hessian prob bound}
		\small
		\bbP\left(\max_{1\leq j \leq N}\left\|\nabla^2 R_{\bS}(\bw_{j}) - \nabla^2 R(\bw_{j})\right\|\right) \leq 4d\left(\frac{3D}{r}\right)^{d}\exp\left(-\frac{n\lambda^{2}}{128L_{1}^{2}}\right).
	\end{equation}
	Define the event 
	\begin{equation}\label{eq:event H}
		\small
		\begin{split}
			H = \Bigg\{&\max_{1\leq j \leq N}\left|\langle\nabla R_{\bS}(\bw_{j}),\nabla R(\bw_{j})\rangle - \|\nabla R(\bw_{j})\|\right| \leq \frac{\alpha^{2}}{4}, \\ 
			&\max_{1\leq j \leq N}\left\|\nabla^2 R_{\bS}(\bw_{j}) - \nabla^2 R(\bw_{j})\right\| \leq \frac{\lambda}{4},\\
			&\bw^{*}_{\bS,k}\ \text{is a local minimum of } R_{\bS}(\cdot), \ k=1,\dots,K
			\Bigg\},
		\end{split}
	\end{equation}
	then combining inequalities \eqref{eq:prob local minimum}, \eqref{gradient prob bound}, \eqref{hessian prob bound}, and Bonferroni inequality, we have
	\begin{equation}
		\small
		\begin{aligned}
			\bbP(H) & \geq 1- 2\left(\frac{3D}{r}\right)^{d}\exp\left(-\frac{n\alpha^{4}}{128L_{0}^{4}}\right) - 
			4d\left(\frac{3D}{r}\right)^{d}\exp\left(-\frac{n\lambda^{2}}{128L_{1}^{2}}\right)\\
			& - K\left\{\frac{512L_{0}^{2}L_{2}^{2}}{n\lambda^{4}} + \frac{128L_{1}^{2}}{n\lambda^{2}}\left(5\sqrt{\log d}  + \frac{4e\log d }{\sqrt{n}}\right)^{2}\right\}.
		\end{aligned}
	\end{equation}
	Next, we show that on event $H$, the two statements in Lemma \ref{lem:no-extra-minimum} hold. For any $\bw \in \cW$ there is $j\in \{1,\dots,N\}$ such that $\|\bw - \bw_{j}\| \leq r$. When event $H$ holds, due to Lemma \ref{lem:lip}, we have
	\begin{equation}
		\small
		\begin{split}
			\left|\langle\nabla R_{\bS}(\bw), \nabla R(\bw)\rangle - \|\nabla R(\bw)\|^2\right|  & \leq \left|\langle \nabla R_{\bS}(\bw_{j}),\nabla R(\bw_{j})\rangle - \|\nabla R(\bw_{j})\|^2\right| \\
			& + \left|\langle \nabla R_{\bS}(\bw), \nabla R(\bw)\rangle - \langle \nabla R_{\bS}(\bw_{j}),\nabla R(\bw_{j})\rangle\right| \\
			& + \left|\|\nabla R(\bw)\|^2 - \|\nabla R(\bw_{j})\|^2\right| \\
			& \leq \frac{\alpha^{2}}{4} + \frac{\alpha^{2}}{8} + \frac{\alpha^{2}}{8} \\ 
			& = \frac{\alpha^{2}}{2},
		\end{split}
	\end{equation}
	and
	\begin{equation}
		\small
		\begin{split}
			\left\|\nabla^2R_{\bS}(\bw) - \nabla^2R(\bw)\right\| & \leq  \left\|\nabla^2R_{\bS}(\bw_{j}) - \nabla^2R(\bw_{j})\right\| \\
			& + \left\|\nabla^2R_{\bS}(\bw) - \nabla^2R_{\bS}(\bw_{j})\right\| + \left\|\nabla^2R(\bw) - \nabla^2R(\bw_j)\right\|\\
			& \leq \frac{\lambda}{4} + \frac{\lambda}{8} + \frac{\lambda}{8} \\
			& = \frac{\lambda}{2}.
		\end{split}
	\end{equation}
	Let $\cD = \{\bw: \|\nabla R(\bw)\| \leq \alpha\}$.
	According to Lemma 8 in the supplemental file of \citep{mei2018landscape}, there exists disjoint open sets $\{\cD_{k}\}_{k=1}^\infty$ with $\cD_{k}$ possibly empty for $k \geq K + 1$ such that $\cD = \cup^{\infty}_{k=1}\cD_{k}$. Moreover $\bw^{*}_k\in \cD_{k}$, for $1\leq k \leq K$ and $\sigma_{\rm min}(\nabla^2 R(\bw)) \geq \lambda$ for each $\bw \in \cup_{k=1}^{K}\cD_{k}$ while $\sigma_{\rm min}(\nabla^2 R(\bw)) \leq - \lambda$ for each $\bw \in \cup_{k=K+1}^{\infty}\cD_{k}$.
	\par
	Thus when the event $H$ holds, for $\bw \in \cD^c$, we have
	\begin{equation}\label{eq:gradient bound}
		\small
		\langle \nabla R_{\bS}(\bw), \nabla R(\bw)\rangle \geq \frac{\alpha^2}{2},
	\end{equation}
	and thus $\bw$ is not a critical point of the empirical risk. On the other hand, Weyl's theorem implies
	\begin{equation}
		\small
		|\sigma_{\rm min}(\nabla^2 R_{\bS}(\bw)) - \sigma_{\rm min}(\nabla^2 R(\bw))| \leq \|\nabla^2 R_{\bS}(\bw) - \nabla^2 R(\bw)\| \leq \frac{\lambda}{2}.
	\end{equation}
	Hence $\sigma_{\rm min}(\nabla^2 R_{\bS}(\bw)) \leq -\lambda/2$ for each $\bw \in \cup_{k=K + 1}^{\infty}\cD_{k}$, and then $\bw$ is not a empirical local minimum. Moreover, $\sigma_{\rm min}(\nabla^2 R_{\bS}(\bw)) \geq \lambda/2$ for each $\bw \in \cup_{k=1}^{K}\cD_{k}$, thus for $k = 1, \dots, K$, $R_{\bS}(\cdot)$ is strongly convex in $\cD_{k}$ and there is at most one local minimum in $\cD_{k}$. Hence when $H$ holds, $R_{\bS}(\cdot)$ has at most $K$ local minimum point, and $\bw^{*}_{\bS,1},\dots,\bw^{*}_{\bS,K}$ are $K$ distinct local minima. This proves $\cM_{\bS} = \{\bw^{*}_{\bS,1}, \dots, \bw^{*}_{\bS,K}\}$.  By inequality \eqref{eq:gradient bound}, we have
	\begin{equation}
		\small
		\frac{\alpha^2}{2} \leq \langle \nabla R_{\bS}(\bw), \nabla R(\bw)\rangle \leq \|\nabla R_{\bS}(\bw)\|\|\nabla R(\bw)\| \leq L_0 \|\nabla R_{\bS}(\bw)\|
	\end{equation}
	for $\bw \in \cD^c$. Thus if $\|\nabla R_{\bS}(\bw)\| < \alpha^2/(2L_0)$ and $\nabla^2 R_{\bS}(\bw)\succ -\lambda/2$, then $\bw \in \cup_{k=1}^{K}\cD_{k}$. The second statement of Lemma \ref{lem:no-extra-minimum} follows from the fact that $R_{\bS}(\cdot)$ is $\lambda/2$-strongly convex on each of $\cD_{k}$ for $k=1,\dots,K$.
\end{proof}
\subsection{Proof of Theorem \ref{thm:generalization error for non-convex}}\label{app:proof of generalization error for non-convex}
The following is the proof of Theorem \ref{thm:generalization error for non-convex}, it provides upper bound of the expected excess risk of any proper algorithm for non-convex problems that efficiently approximates SOSP. We first introduce the following lemma which is a variant of Lemma \ref{lem:genearlization error on minima}.
\begin{lemma}\label{lem:absolute generalize}
	Under Assumptions \ref{ass:smoothness} and \ref{ass:strict saddle}
	\begin{equation}\label{eq:absolute generalization on local minima}
		\small
		\begin{aligned}
			&\mE_{\bS}\left[|R_{\bS}(\bw^{*}_{\bS,k}) - R(\bw^{*}_{\bS,k})|\right]\\
			&\leq \frac{2M}{\sqrt{n}} + \frac{8L_{0}}{n\lambda} \left[ L_{0} + \left\{\frac{64L_{0}^{2}L_{2}^{2}}{\lambda^{3}} + \frac{16L_{1}^{2}}{\lambda}\left(5\sqrt{\log d}  + \frac{4e\log d }{\sqrt{n}}\right)^{2}\right\}\min\left\{3D, \frac{3\lambda}{2L_2}\right\}\right].
		\end{aligned}
	\end{equation}
\end{lemma}
\begin{proof}
	
	For $\bw \in B_{2}(\bw^{*}_{k}, \frac{\lambda}{4L_{2}})$, by Weyl's theorem (Exercise 6.1 in \citep{wainwright2019}),
	\begin{equation}
		\small
		\sigma_{\rm min}(\nabla^2R(\bw)) \geq \sigma_{\rm min}(\nabla^2R(\bw^{*}_{k})) - \|\nabla^2R(\bw) - \nabla^2R(\bw^{*}_{k})\| \geq \lambda - L_{2}\|\bw - \bw_{k}^{*}\| \geq \frac{3\lambda}{4}.
	\end{equation}
	Hence $R(\cdot)$ is strongly convex in $B_{2}(\bw^{*}_{k}, \frac{\lambda}{4L_{2}})$. Then because $\bw^{*}_{k}$ is a local minimum of $R(\cdot)$, we have  
	\begin{equation}
		\small
		\bw^{*}_{k} = \mathop{\arg\min}_{\bw \in B_{2}(\bw^{*}_{k}, \frac{\lambda}{4L_{2}})} R(\bw).
	\end{equation}
	Thus $R(\bw^{*}_{k}) \leq R(\bw^{*}_{\bS,k})$ and $R_{\bS}(\bw^{*}_{\bS,k}) \leq R_{\bS}(\bw^{*}_{k})$. Then 
	\begin{equation}
		\small
		(R_{\bS}(\bw^{*}_{\bS,k}) - R(\bw^{*}_{\bS,k}))_{+} \leq |R_{\bS}(\bw^{*}_{k}) - R(\bw^{*}_{k})|,
	\end{equation}
	and 
	\begin{equation}
		\small
		\begin{split}
			\mE\left[(R_{\bS}(\bw^{*}_{\bS,k}) - R(\bw^{*}_{\bS,k}))_{+}\right] & \leq \mE\left[|R_{\bS}(\bw^{*}_{k}) - R(\bw^{*}_{k})|\right]\\
			& \overset{a}{\leq} \left(\mE\left[(R_{\bS}(\bw^{*}_{k}) - R(\bw^{*}_{k}))^{2}\right]\right)^{\frac{1}{2}}\\
			& \leq \frac{M}{\sqrt{n}},
		\end{split}
	\end{equation}
	where $a$ is due to Jensen's inequality. Hence
	\begin{equation}
		\small
		\begin{split}
			\mE\left[|R_{\bS}(\bw^{*}_{\bS,k}) - R(\bw^{*}_{\bS,k})|\right] &= \mE\left[(R_{\bS}(\bw^{*}_{\bS,k}) - R(\bw^{*}_{\bS,k}))_{+}\right] + \mE\left[(R_{\bS}(\bw^{*}_{\bS,k}) - R(\bw^{*}_{\bS,k}))_{-}\right]\\
			& = 2\mE\left[(R_{\bS}(\bw^{*}_{\bS,k}) - R(\bw^{*}_{\bS,k}))_{+}\right] - \mE\left[R_{\bS}(\bw^{*}_{\bS,k}) - R(\bw^{*}_{\bS,k})\right] \\
			& \leq 2\mE\left[(R(\bw^{*}_{\bS,k}) - R(\bw^{*}_{\bS,k}))_{+}\right] + |\mE\left[R_{\bS}(\bw^{*}_{\bS,k}) - R(\bw^{*}_{\bS,k})\right]|\\
			& \leq \frac{2M}{\sqrt{n}} + |\mE\left[R_{\bS}(\bw^{*}_{\bS,k}) - R(\bw^{*}_{\bS,k})\right]|.
		\end{split}
	\end{equation}
	Then \eqref{eq:absolute generalization on local minima} follows from \eqref{eq:generalization on local minima}. 
\end{proof}
Then we are ready to give the proof of Theorem \ref{thm:generalization error for non-convex}.
\par
\paragraph{Restate of Theorem \ref{thm:generalization error for non-convex}}
\emph{Under Assumption \ref{ass:smoothness}, \ref{ass:local strong convexity} and \ref{ass:strict saddle}, if $\bw_{t}$ satisfies \eqref{eq:escape from saddle point} and $r$ defined in Lemma \ref{lem:no-extra-minimum}, by choosing $t$ such that $\zeta(t) < \alpha^2/(2L_0)$ and $\rho(t) < \lambda/2$ we have 
	\begin{equation}\label{eq:non-convex gen}
		\small
		\begin{aligned}
			|\mE_{\cA,\bS}\left[R(\bw_{t}) - R_{\bS}(\bw_{t})\right]|
			& \leq \frac{8L_{0}}{\lambda}\zeta(t) + 2L_{0}D\delta + \frac{2KM}{\sqrt{n}} + \frac{8KL_{0}^{2}}{n\lambda} \\
			& + \left(L_{0}\min\left\{3D,\frac{3\lambda}{2L_{2}}\right\} + 2M\right)\xi_{n, 1} + 2M\xi_{n, 2},
		\end{aligned}
	\end{equation}
	where 
	\begin{equation}
		\small
		\begin{aligned}
			\xi_{n,1} & =K\left\{\frac{512L_{0}^{2}L_{2}^{2}}{n\lambda^{4}} + \frac{128L_{1}^{2}}{n\lambda^{2}}\left(5\sqrt{\log d}  + \frac{4e\log d }{\sqrt{n}}\right)^{2}\right\},
		\end{aligned}
	\end{equation}
	and
	\begin{equation}
		\small
		\begin{aligned}
			\xi_{n,2} = 2\left(\frac{3D}{r}\right)^{d}\exp\left(-\frac{n\alpha^{4}}{128L_{0}^{4}}\right) + 4d\left(\frac{3D}{r}\right)^{d}\exp\left(-\frac{n\lambda^{2}}{128L_{1}^{2}}\right).
		\end{aligned}
	\end{equation}
	If with probability at least $1-\delta^{\prime}$ ($\delta^{\prime}$ can be arbitrary small), $R_{\bS}(\cdot)$ has no spurious local minimum, then
	\begin{equation}\label{eq:non-convex gen nslm}
		\small
		\begin{aligned}
			|\mE_{\cA,\bS}\left[R(\bw_{t}) - R_{\bS}(\bw_{t})\right]| & \leq \frac{8L_{0}}{\lambda}\zeta(t) + 2L_{0}D\delta + 6M\delta^{\prime} + \frac{8(K+4)L_{0}^{2}}{n\lambda} \\
			& + \left(\frac{(K+4)L_{0}}{K}\min\left\{3D,\frac{3\lambda}{2L_{2}}\right\} + 6M\right)\xi_{n,1} + 6M\xi_{n,2}.
		\end{aligned}
	\end{equation}
} 
\begin{proof}
	Remind the event in the proof of Lemma \ref{lem:no-extra-minimum}
	\begin{equation}
		\small
		\begin{split}
			H = \Bigg\{&\max_{1\leq j \leq N}\left\|\langle\nabla R_{\bS}(\bw_{j}),\nabla R(\bw_{j})\rangle - \|\nabla R(\bw_{j})\|\right\| \leq \frac{\alpha^{2}}{4}, \\ 
			&\max_{1\leq j \leq N}\left\|\nabla^2 R_{\bS}(\bw_{j}) - \nabla^2 R(\bw_{j})\right\| \leq \frac{\lambda}{4},\\
			&\bw^{*}_{\bS,k}\ \text{is a local minimum of } R_{\bS}(\cdot), \ k=1,\dots,K
			\Bigg\},
		\end{split}
	\end{equation}
	We have $\bbP(H^{c})\leq \xi_{n,1} + \xi_{n,2}$, and on the event $H$
	\begin{enumerate}
		\item[\emph{i}:] $\cM_{\bS} = \{\bw^{*}_{\bS,1}, \dots, \bw^{*}_{\bS,K}\}$;
		\item[\emph{ii}:] For any $\bw \in \cW$, if $\|\nabla R_{\bS}(\bw)\| < \alpha^2/(2L_0)$ and $\nabla^2 R_{\bS}(\bw)\succ -\lambda/2$, then  $\|\bw - \cP_{\cM_{\bS}}(\bw)\| \leq \lambda\|\nabla R_{\bS}(\bw)\| / 4$.
	\end{enumerate}	
	By Assumption \ref{ass:smoothness},
	\begin{equation}\label{eq:decompose}
		\small
		\begin{split}
			\left|\mE\left[R(\bw_{t}) - R_{\bS}(\bw_{t})\right]\right| & \leq \left|\mE\left[(R(\bw_{t}) - R_{\bS}(\bw_{t}))\textbf{1}_{H}\right]\right| + \left|\mE\left[(R(\bw_{t}) - R_{\bS}(\bw_{t}))\textbf{1}_{H^{c}}\right]\right| \\
			& \leq \left|\mE\left[(R(\bw_{t}) - R(\cP_{\cM_{\bS}}(\bw_{t})))\textbf{1}_{H}\right]\right| \\
			& + \left|\mE\left[(R_{\bS}(\bw_{t}) - R_{\bS}(\cP_{\cM_{\bS}}(\bw_{t})))\textbf{1}_{H}\right]\right| \\
			& + \left|\mE\left[(R(\cP_{\cM_{\bS}}(\bw_{t})) - R_{\bS}(\cP_{\cM_{\bS}}(\bw_{t})))\textbf{1}_{H}\right]\right| + 2M\bbP(H^{c})\\
			& \leq  2L_{0}\mE\left[\|\bw_{t} - \cP_{\cM_{\bS}}(\bw_{t})\|\textbf{1}_{H}\right] \\
			& + \left|\mE\left[(R(\cP_{\cM_{\bS}}(\bw_{t})) - R_{\bS}(\cP_{\cM_{\bS}}(\bw_{t})))\textbf{1}_{H}\right]\right| + 2M\bbP(H^{c}).
		\end{split}
	\end{equation}
	Because $\zeta(t)<\alpha^2/(2L_0)$, $\rho(t) < \lambda/2$ and \eqref{eq:escape from saddle point}, we have on event $H$
	\begin{equation}
		\small
		\begin{split}
			\bbP_{\cA}\left(U\right)\geq 1- \delta,
		\end{split}
	\end{equation}
	where 
	\begin{equation}
		\small
		U = \left\{\nabla R_{\bS}(\bw_t)< \frac{\alpha^2}{2L_{0}}, \nabla^2 R_{\bS}(\bw_{t}) \succ -\frac{\lambda}{2}\right\}.
	\end{equation}
	Thus we have
	\begin{equation}\label{eq:optimize}
		\begin{split}
			\small
			\mE[\|\bw_{t} - \cP_{\cM_{\bS}}(\bw_{t})\|\textbf{1}_{H}] 
			&\leq \mE[\|\bw_{t} - \cP_{\cM_{\bS}}(\bw_{t})\|\textbf{1}_{H\cap U^{c}}]
			+\mE[\|\bw_{t} - \cP_{\cM_{\bS}}(\bw_{t})\|\textbf{1}_{H\cap U}]\\
			&\leq \frac{4}{\lambda}\zeta(t) + D\delta,
		\end{split}
	\end{equation}
	where the second inequality is due to the property $(ii)$ in Lemma \ref{lem:no-extra-minimum} holds on event $H$. According to \eqref{eq:absolute generalization on local minima}, we have 
	\begin{equation}\label{eq:generalize}
		\small
		\begin{split}
			& \left|\mE\left[(R(\cP_{\cM_{\bS}}(\bw_{t})) - R_{\bS}(\cP_{\cM_{\bS}}(\bw_{t})))\textbf{1}_{H}\right]\right|\\ 
			& \leq\mE\left|[(R(\cP_{\cM_{\bS}}(\bw_{t})) - R_{\bS}(\cP_{\cM_{\bS}}(\bw_{t})))\textbf{1}_{H}]\right|\\
			& \leq \mE\left[\max_{1\leq k \leq K}|R(\bw^{*}_{\bS,k}) - R_{\bS}(\bw^{*}_{\bS,k})|\right]\\
			& \leq \sum_{k=1}^K\mE\left[|R(\bw^{*}_{\bS,k}) - R_{\bS}(\bw^{*}_{\bS,k})|\right]\\
			& \leq K\Bigg[\frac{2M}{\sqrt{n}} + \frac{8L_{0}}{n\lambda} \left[ L_{0} + \left\{\frac{64L_{0}^{2}L_{2}^{2}}{\lambda^{3}} + \frac{16L_{1}^{2}}{\lambda}\left(5\sqrt{\log d}  + \frac{4e\log d }{\sqrt{n}}\right)^{2}\right\}\min\left\{3D, \frac{3\lambda}{2L_2}\right\}\right]\Bigg].
		\end{split}
	\end{equation}
	Combination of equations \eqref{eq:decompose}, \eqref{eq:optimize} and \eqref{eq:generalize} completes the proof of \eqref{eq:non-convex gen}.
	\par To establish \eqref{eq:non-convex gen nslm}, we bound $\left|\mE\left[(R_{\bS}(\cP_{\cM_{\bS}}(\bw_{t})) - R(\cP_{\cM_{\bS}}(\bw_{t})))\textbf{1}_{H}\right]\right|$ in a different manner. 
	Remind $\cM = \{\bw_{1}^{*},\cdots,\bw_{K}^{*}\}$ is the set of population local minima. Let 
	\begin{equation}\label{eq:no spurious local minimum}
		G = \left\{R_{\bS}(\cdot)\ \text{has no spurious local minimum}\right\}.
	\end{equation}
	Then the assumption implies that $\bbP(G^{c}) \leq \delta^{\prime}$.
	Note that
	\begin{equation}\label{eq:bound nslm1}
		\small
		\begin{aligned}
			\left|\mE[(R(\cP_{\cM_{\bS}}(\bw_{t})) - R_{\bS}(\cP_{\cM_{\bS}}(\bw_{t})))\textbf{1}_{H}]\right| & \leq 
			\left|\mE[(R(\cP_{\cM_{\bS}}(\bw_{t})) - R_{\bS}(\cP_{\cM_{\bS}}(\bw_{t})))\textbf{1}_{H\bigcap G}]\right| \\
			& + \left|\mE[(R(\cP_{\cM_{\bS}}(\bw_{t})) - R_{\bS}(\cP_{\cM_{\bS}}(\bw_{t})))\textbf{1}_{H\bigcap G^{c}}]\right|\\
			& \leq \left|\mE[(R(\cP_{\cM_{\bS}}(\bw_{t})) - R_{\bS}(\cP_{\cM_{\bS}}(\bw_{t})))\textbf{1}_{H\bigcap G}]\right| + 2M\delta^{\prime}\\
			& =\left|\mE[(R(\cP_{\cM_{\bS}}(\bw_{t})) - R_{\bS}(\bw_{\bS,1}^{*}))\textbf{1}_{H\bigcap G}]\right| + 2M\delta^{\prime} \\
			& \leq \left|\mE[(R(\cP_{\cM_{\bS}}(\bw_{t})) - R_{\bS}(\bw_{\bS,1}^{*}))\textbf{1}_{H}]\right| + 4M\delta^{\prime},
		\end{aligned}
	\end{equation}
	where the last inequality is due to $\bbP(G^{c}) \leq \delta^{\prime}$. Moreover, under Assumption \ref{ass:smoothness}
	\begin{equation}\label{eq:bound nslm2}
		\small
		\begin{aligned}
			\left|\mE[(R(\cP_{\cM_{\bS}}(\bw_{t})) - R_{\bS}(\bw_{\bS,1}^{*}))\textbf{1}_{H}]\right| &\leq  \left|\mE[(R(\cP_{\cM_{\bS}}(\bw_{t})) - R(\cP_{\cM}(\cP_{\cM_{\bS}}(\bw_{t}))))\textbf{1}_{H}]\right| \\
			& + \left|\mE[(R(\cP_{\cM}(\cP_{\cM_{\bS}}(\bw_{t}))) - R(\bw_{1}^{*}))\textbf{1}_{H}]\right|\\
			& + \left|\mE[(R(\bw_{1}^{*}) - R(\bw_{\bS,1}^{*}))\textbf{1}_{H}]\right| \\
			& + \left|\mE[(R(\bw_{\bS,1}^{*}) - R_{\bS}(\bw_{\bS,1}^{*}))\textbf{1}_{H}]\right| \\
			& \leq \left|\mE\left[\max_{k}\left\{R(\bw_{\bS,k}^{*}) - R(\bw_{k}^{*})\right\}\textbf{1}_{H}\right]\right|\\
			& + \max_{k}\left\{|R(\bw_{k}^{*}) - R(\bw_{1}^{*})|\right\} + \left|\mE[(R(\bw_{1}^{*}) - R(\bw_{\bS,1}^{*}))]\right|\\
			& + \left|\mE[(R(\bw_{\bS,1}^{*}) - R_{\bS}(\bw_{\bS,1}^{*}))]\right| + 4M\bbP(H^{c}).
		\end{aligned}
	\end{equation}
	Due to Proposition \ref{prop:dist of local minimum}, $R(\bw_{\bS,k}^{*}) - R(\bw_{k}^{*}) \geq 0$, then 
	\begin{equation}
		\small
		\begin{aligned}
			\left|\mE\left[\max_{k}\left\{R(\bw_{\bS,k}^{*}) - R(\bw_{k}^{*})\right\}\textbf{1}_{H}\right]\right| &\leq \left|\mE\left[\sum_{k=1}^{K}(R(\bw_{\bS,k}^{*}) - R(\bw_{k}^{*}))\right]\right|\\
			& \leq \sum_{k=1}^{K} \left|\mE[(R(\bw_{\bS,k}^{*}) - R(\bw_{k}^{*}))]\right|.
		\end{aligned}
	\end{equation}
	According to Lemma \ref{lem:genearlization error on minima},
	\begin{equation}\label{eq:bound nslm3}
		\small
		\begin{aligned}
			|\mE[R(\bw_{\bS,k}^{*}) - R_{\bS}(\bw_{\bS,k}^{*})]| & \leq \frac{8L_{0}}{n\lambda} \left[ L_{0} + \left\{\frac{64L_{0}^{2}L_{2}^{2}}{\lambda^{3}} + \frac{16L_{1}^{2}}{\lambda}\left(5\sqrt{\log d}  + \frac{4e\log d }{\sqrt{n}}\right)^{2}\right\}\min\left\{3D, \frac{3\lambda}{2L_2}\right\}\right]\\
			&= \frac{8L_{0}^{2}}{n\lambda} + \frac{L_{0}}{K}\min\left\{3D,\frac{3\lambda}{2L_{2}}\right\}\xi_{n,1}.
		\end{aligned}
	\end{equation}
	Then 
	\begin{equation}\label{eq:bound nslm4}
		\small
		\begin{aligned}
			\mE[R(\bw_{\bS,k}^{*}) - R(\bw_{k}^{*})] & = \mE[R(\bw_{\bS,k}^{*}) - R_{\bS}(\bw_{\bS,k}^{*})] + \mE[R_{\bS}(\bw_{\bS,k}^{*}) - R_{\bS}(\bw_{k}^{*})] \\ 
			& \leq \frac{8L_{0}^{2}}{n\lambda} + \frac{L_{0}}{K}\min\left\{3D,\frac{3\lambda}{2L_{2}}\right\}\xi_{n,1},
		\end{aligned}
	\end{equation}
	where the inequality is due to the definition of $\bw_{\bS,k}^{*}$. \eqref{eq:bound nslm1}, \eqref{eq:bound nslm2}, \eqref{eq:bound nslm3} and \eqref{eq:bound nslm4} together implies
	\begin{equation}\label{eq:bound nslm sum1}
		\small
		\begin{aligned}
			\left|\mE[(R(\cP_{\cM_{\bS}}(\bw_{t})) - R_{\bS}(\cP_{\cM_{\bS}}(\bw_{t})))\textbf{1}_{H}]\right| &\leq
			\frac{8(K+2)L_{0}^{2}}{n\lambda} + \frac{(K+2)L_{0}}{K}\min\left\{3D,\frac{3\lambda}{2L_{2}}\right\}\xi_{n,1} \\
			& + \max_{k}\left\{|R(\bw_{k}^{*}) - R(\bw_{1}^{*})|\right\} + 4M(\delta^{\prime} + \xi_{n,1} +\xi_{n,2}).
		\end{aligned}
	\end{equation}
	Now we deal with the term $\max_{k}\left\{|R(\bw_{k}^{*}) - R(\bw_{1}^{*})|\right\}$. Note that
	\begin{equation}\label{eq:bound nslm 5}
		\small
		\begin{aligned}
			|R(\bw_{k}^{*}) - R(\bw_{1}^{*})| &\leq |\mE[R(\bw_{\bS,k}^{*}) - R_{\bS}(\bw_{k}^{*})]| + |\mE[R(\bw_{\bS,1}^{*}) - R_{\bS}(\bw_{1}^{*})]| \\
			& + |\mE[R_{\bS}(\bw_{\bS,k}^{*}) - R_{\bS}(\bw_{\bS,1}^{*})]| \\
			& \leq \frac{16L_{0}^{2}}{n\lambda} + \frac{2L_{0}}{K}\min\left\{3D,\frac{3\lambda}{2L_{2}}\right\}\xi_{n,1} +|\mE[R_{\bS}(\bw_{\bS,k}^{*}) - R_{\bS}(\bw_{\bS,1}^{*})]|.
		\end{aligned}
	\end{equation}
	Because on the event $H\bigcap G$, $R_{\bS}(\bw_{\bS,k}^{*}) - R_{\bS}(\bw_{\bS,k}^{*}) = 0$,
	\begin{equation}\label{eq:bound nslm 6}
		\small
		|\mE[R_{\bS}(\bw_{\bS,k}^{*}) - R_{\bS}(\bw_{\bS,1}^{*})]| \leq 2M(\bbP(H^{c}) + \bbP(G^{c})) \leq 2M(\xi_{n,1} + \xi_{n,2} + \delta^{\prime}).
	\end{equation}
	Combining \eqref{eq:bound nslm sum1}, \eqref{eq:bound nslm 5} and \eqref{eq:bound nslm 6}, we \eqref{eq:non-convex gen nslm}.
	\begin{equation}\label{eq:bound nslm sum2}
		\small
		\begin{aligned}
			\left|\mE[(R(\cP_{\cM_{\bS}}(\bw_{t})) - R_{\bS}(\cP_{\cM_{\bS}}(\bw_{t})))\textbf{1}_{H}]\right| & \leq \frac{8(K+4)L_{0}^{2}}{n\lambda} + \frac{(K+4)L_{0}}{K}\min\left\{3D,\frac{3\lambda}{2L_{2}}\right\}\xi_{n,1} \\
			& + 6M(\delta^{\prime} + \xi_{n,1} +\xi_{n,2}).
		\end{aligned}
	\end{equation}
	\eqref{eq:decompose}, \eqref{eq:optimize} and \eqref{eq:bound nslm sum2} implies \eqref{eq:non-convex gen nslm}.
\end{proof}
	We notice the technique of deriving the order $\tilde{\cO}(1/n)$ when empirical risk has no spurious local minima with high probability is very tricky. Because the obstacle is when we derive upper bound of $\left|\mE\left[(R_{\bS}(\cP_{\cM_{\bS}}(\bw_{t})) - R(\cP_{\cM_{\bS}}(\bw_{t})))\textbf{1}_{H}\right]\right|$, the involved $\cP_{\cM_{\bS}}(\bw_{t})$ is related to the proper algorithm, then it is not guaranteed to converge to a specific empirical local minima which makes us can not directly apply Lemma \ref{lem:genearlization error on minima}.  However, if the proper algorithm is guaranteed to find a specific local minima e.g., GD finds the minimal norm solution for over-parameterized neural network, which is called ``the implicit regularization of GD'' \citep{bartlett2021deep}, the order of $\tilde{\cO}(1/n)$ can be maintained even the assumption on empirical local minima is violated.     
\subsection{Proof of Theorem \ref{thm:excess risk for non-convex}}\label{app:Proof in sec4.3}
The proof is based on the Lemma \ref{lem:no-extra-minimum} in the above section. 
\paragraph{Restate of Theorem \ref{thm:excess risk for non-convex}}
\emph{Under Assumption \ref{ass:smoothness}, \ref{ass:local strong convexity} and \ref{ass:strict saddle}, if $\bw_{t}$ satisfies \eqref{eq:escape from saddle point}, by choosing $t$ in \eqref{eq:escape from saddle point} such that $\zeta(t) < \alpha^2/(2L_0)$ and $\rho(t) < \lambda/2$, we have 
	\begin{equation}\label{eq:non-convex er}
		\small
		\begin{aligned}
			\mE_{\cA,\bS}\left[R(\bw_{t}) - R(\bw^{*})\right] & \leq \frac{4L_{0}}{\lambda}\zeta(t) + L_{0}D\delta + \frac{2KM}{\sqrt{n}} \\
			& + \frac{8KL_{0}^{2}}{n\lambda} + \left(L_{0}\min\left\{3D,\frac{3\lambda}{2L_{2}}\right\} + 2M\right)\xi_{n, 1} + 2M\xi_{n, 2}\\
			& + \mE_{\cA,\bS}[R_{\bS}(\cP_{\cM_{\bS}}(\bw_{t})) - R_{\bS}(\bw_{\bS}^{*})],
		\end{aligned}
	\end{equation}
	If with probability at least $1-\delta^{\prime}$ ($\delta^{\prime}$ can be arbitrary small), $R_{\bS}(\cdot)$ has no spurious local minimum, then
	\begin{equation}\label{eq:non-convex er nslm}
		\small
		\begin{aligned}
			\mE_{\cA,\bS}\left[R(\bw_{t}) - R(\bw^{*})\right] & \leq \frac{4L_{0}}{\lambda}\zeta(t) + L_{0}D\delta + 8M\delta^{\prime} + \frac{8(K+4)L_{0}^{2}}{n\lambda} \\
			& + \left(\frac{(K+4)L_{0}}{K}\min\left\{3D,\frac{3\lambda}{2L_{2}}\right\} + 8M\right)\xi_{n,1} + 8M\xi_{n,2},
		\end{aligned}
\end{equation}
	where $\xi_{n,1}$ and $\xi_{n, 2}$ are defined in Theorem \ref{thm:generalization error for non-convex}, and $\bw_{\bS}^{*}$ is the global minimum of $R_{\bS}(\cdot)$.}
\begin{proof}
	By Assumption \ref{ass:smoothness} and the relationship $R_{\bS}(\bw_{\bS}^{*}) \leq R_{\bS}(\bw^{*})$, we have the following decomposition
	\begin{equation}\label{eq:decomposition}
		\small
		\begin{aligned}
			\mE \left[R(\bw_{t}) - R(\bw^{*})\right] & = \mE\left[R(\bw_{t}) - R_{\bS}(\bw^{*})\right]\\
			&\leq \mE[R(\bw_{t}) - R_{\bS}(\bw_{\bS}^{*})]\\
			& \leq |\mE\left[(R(\bw_{t}) - R_{\bS}(\bw_{\bS}^{*}))\textbf{1}_{H}\right]|
			+ |\mE\left[(R(\bw_{t}) - R_{\bS}(\bw_{\bS}^{*}))\textbf{1}_{H^{c}}\right]|\\
			& \leq \left|\mE[(R(\bw_{t}) - R(\cP_{\cM_{\bS}}(\bw_{t})))\textbf{1}_{H}]\right| + \left|\mE[(R(\cP_{\cM_{\bS}}(\bw_{t})) - R_{\bS}(\cP_{\cM_{\bS}}(\bw_{t})))\textbf{1}_{H}]\right| \\
			& + \mE[(R_{\bS}(\cP_{\cM_{\bS}}(\bw_{t})) - R_{\bS}(\bw_{\bS}^{*}))\textbf{1}_{H}] + 2M\bbP(H^{c})\\
			& \leq L_{0}\mE\left[\|\bw_{t} - \cP_{\cM_{\bS}}(\bw_{t})\|\textbf{1}_{H}\right] + \mE[\left|R(\cP_{\cM_{\bS}}(\bw_{t})) - R_{\bS}(\cP_{\cM_{\bS}}(\bw_{t}))\right|\textbf{1}_{H}]\\
			& + \mE[R_{\bS}(\cP_{\cM_{\bS}}(\bw_{t})) - R_{\bS}(\bw_{\bS}^{*})] + 2M\bbP(H^{c}).
		\end{aligned}
	\end{equation}
	The upper bound of the first and second terms in the last inequality can be easily derived from the proof of Theorem \ref{thm:generalization error for non-convex} which implies
	\begin{equation}
		\small
		\begin{aligned}
			L_{0}\mE\left[\|\bw_{t} - \cP_{\cM_{\bS}}(\bw_{t})\|\textbf{1}_{H}\right] & + \mE[\left|R(\cP_{\cM_{\bS}}(\bw_{t})) - R_{\bS}(\cP_{\cM_{\bS}}(\bw_{t}))\right|\textbf{1}_{H}]\\
			& \leq \frac{4L_{0}}{\lambda}\zeta(t) + L_{0}D\delta + \frac{2KM}{\sqrt{n}} + \frac{8KL_{0}^{2}}{n\lambda} + L_{0}\min\left\{3D,\frac{3\lambda}{2L_{2}}\right\}\xi_{n,1}.
		\end{aligned}
	\end{equation}
	Plugging this into \eqref{eq:decomposition}, we get \eqref{eq:non-convex er}. 
	\par
	Next, we move on to \eqref{eq:non-convex er nslm}. According to \eqref{eq:decomposition}, 
	\begin{equation}\label{eq:er nslm decomposition}
		\small
		\begin{aligned}
			\mE \left[R(\bw_{t}) - R(\bw^{*})\right] & = \mE\left[R(\bw_{t}) - R_{\bS}(\bw^{*})\right]\\
			&\leq \mE[R(\bw_{t}) - R_{\bS}(\bw_{\bS}^{*})]\\
			& \leq |\mE\left[(R(\bw_{t}) - R_{\bS}(\bw_{\bS}^{*}))\textbf{1}_{H}\right]|
			+ |\mE\left[(R(\bw_{t}) - R_{\bS}(\bw_{\bS}^{*}))\textbf{1}_{H^{c}}\right]|\\
			& \leq \left|\mE[(R(\bw_{t}) - R(\cP_{\cM_{\bS}}(\bw_{t})))\textbf{1}_{H}]\right| \\
			& + \left|\mE[(R(\cP_{\cM_{\bS}}(\bw_{t})) - R_{\bS}(\cP_{\cM_{\bS}}(\bw_{t})))\textbf{1}_{H}]\right| \\
			& + \mE[(R_{\bS}(\cP_{\cM_{\bS}}(\bw_{t})) - R_{\bS}(\bw_{\bS}^{*}))\textbf{1}_{H}] + 2M\bbP(H^{c}).
		\end{aligned}
	\end{equation}
	According to \eqref{eq:optimize}, 
	\begin{equation}\label{eq:er nslm1}
		\small
		\left|\mE[(R(\bw_{t}) - R(\cP_{\cM_{\bS}}(\bw_{t})))\textbf{1}_{H}]\right| \leq L_{0}\mE\left[\|\bw_{t} - \cP_{\cM_{\bS}}(\bw_{t})\|\textbf{1}_{H}\right] \leq \frac{4L_{0}}{\lambda}\zeta(t) + L_{0}D\delta.
	\end{equation}
	Moreover,
	\begin{equation}\label{eq:decomposition opt}
		\small
		\begin{aligned}
			\mE[(R_{\bS}(\cP_{\cM_{\bS}}(\bw_{t})) - R_{\bS}(\bw_{\bS}^{*}))\textbf{1}_{H}]
			& \leq |\mE[(R_{\bS}(\cP_{\cM_{\bS}}(\bw_{t})) - R_{\bS}(\bw_{\bS}^{*}))\textbf{1}_{H\bigcap G}]| \\
			& + |\mE[(R_{\bS}(\cP_{\cM_{\bS}}(\bw_{t})) - R_{\bS}(\bw_{\bS}^{*}))\textbf{1}_{H\bigcap G^{c}}]|.
		\end{aligned}
	\end{equation}
	Because on the event $H\bigcap G$, $R_{\bS}(\cP_{\cM_{\bS}}(\bw_{t})) - R_{\bS}(\bw_{\bS}^{*}) = 0$, \eqref{eq:decomposition opt} implies 
	\begin{equation}\label{eq:er nslm2}
		\small
		\mE[(R_{\bS}(\cP_{\cM_{\bS}}(\bw_{t})) - R_{\bS}(\bw_{\bS}^{*}))\textbf{1}_{H}] \leq 2M\bbP(G^{c}) = 2M\delta^{\prime}.
	\end{equation}
	\eqref{eq:er nslm decomposition}, \eqref{eq:bound nslm sum2}, \eqref{eq:er nslm1} and \eqref{eq:er nslm2} implies
	\eqref{eq:non-convex er nslm}.
\end{proof}

\section{An Algorithm Approximates the SOSP}\label{app:alg SOSP}
\begin{algorithm}[t!]
	\caption{Projected Gradient Descent (PGD)}
	\label{alg:psgd}
	\begin{algorithmic}
		\STATE   {\textbf{Input:} Parameter space $B_{1}(\textbf{0}, 1)$, initial point $\bw_{0}$, learning rate $\eta=\frac{1}{L_{1}}$, tolerance $\epsilon\leq \min\left\{\frac{8\beta^{3}L_{2}^{3}}{27L_{1}^{3}}, \frac{27}{64^{3}L_{2}^{3}}, \frac{\beta}{2}\right\}$, }
		\FOR	 {$t = 0,1,\cdots$}
		\IF 	 {$\|\nabla R_{\bS}(\bw_{t})\| \geq \epsilon$}
		\IF    	 {$\bw_{t} \in B_{2}(\textbf{0}, 1)$ with $\|\bw_{t}\|=1$}
		\STATE   {$\bw_{t + 1} = \left(1 - \frac{\beta}{L_{1}}\right)\bw_{t}$}
		\ELSE
		\STATE   {$\bw_{t + 1} = \cP_{B_{2}(\textbf{0}, 1)}\left(\bw_{t} - \eta\nabla R_{\bS}(\bw_{t})\right)$}
		\ENDIF
		\ELSE	 
		\IF		{$\nabla^{2}R_{\bS}(\bw_{t})\preceq -\epsilon^{\frac{1}{3}}$}
		\STATE   {Computed $\bu_{t}\in B_{2}(\textbf{0}, 1)$ such that $(\bu_{t} - \bw_{t})^{T}\nabla^{2} R_{\bS}(\bw_{t})(\bu_{t} - \bw_{t}) \leq -\frac{\beta^{2}\epsilon^{\frac{1}{3}}}{8L_{1}}$}
		\STATE 	 {$\bw_{t + 1} = \sigma\bu_{t} + (1 - \sigma)\bw_{t}$ with $\sigma = \frac{3L_{1}\epsilon^{\frac{1}{3}}}{2\beta L_{2}}$.}
		\ELSE
		\STATE 	 {Return $\bw_{t + 1}$}
		\ENDIF
		\ENDIF
		\ENDFOR	
	\end{algorithmic}
\end{algorithm} 
For non-convex problems, as we have mentioned in the main body of this paper, we consider proper algorithm that approximates SOSP. Here, we present a detailed discussion to them, and propose such a proper algorithm to make it more concrete.
\par
There are extensive papers about non-convex optimization working on proposing algorithms that approximate SOSP, see   \citep{ge2015escaping,fang2019sharp,daneshmand2018escaping,jin2017escape,jin2019stochastic,xu2018first,mokhtari2018escaping} for examples. However, to the best of our knowledge, theoretical guarantee of vanilla SGD approximating SOSP remains to be explored, especially for the constrained parameter space. The most related result is Theorem 11 in \citep{ge2015escaping} that projected perturbed noisy gradient descent approximates a $(\epsilon, \sqrt{L_{2}\epsilon})$-SOSP (The definition of $(\epsilon, \gamma)$-SOSP is in the main body of this paper.) in a computational cost of $\cO(\epsilon^{-2})$. Though this result is only applied to equality constraints. 
\par
Considering the mismatch of settings between this paper and the existing literatures, we propose a gradient-based method Algorithm \ref{alg:psgd} inspired by \citep{mokhtari2018escaping} to approximate SOSP for non-convex problems. Without loss of generality, we assume that the convex compact parameter space $\cW$ is $B_{2}(\textbf{0}, 1)$. The proposed algorithm is conducted under the following assumption which implies that there is no minimum on the boundary of the parameter space $\cW$.
\begin{assumption}\label{ass:local minima on the boundry}
	For any $\bw\in B_{2}(\emph{\textbf{0}}, 1)$ with $\|\bw\| = 1$, there exists $L_{1} > \beta > 0$ such that $\langle\nabla R_{\bS}(\bw), \bw\rangle\geq \beta$. 
\end{assumption} 
\par
We have following discussion to the proposed Algorithm \ref{alg:psgd} before providing its convergence rate. The involved quadratic programming can be efficiently solved under Assumption \ref{ass:strict saddle} \citep{nocedal2006numerical}. In addition, we can find $\bu_{t}$ in Algorithm \ref{alg:psgd} is because the minimal value of the quadratic loss is $-\beta^{2}\epsilon^{1/3}/8L_{1}$. The next theorem states the convergence rate of the proposed Algorithm \ref{alg:psgd}.
\begin{theorem}\label{thm:escape from saddle point}
	Under Assumption \ref{ass:smoothness} and \ref{ass:local minima on the boundry}, let $\bw_{t}$ updated in Algorithm \ref{alg:psgd}, 
	by choosing 
	\begin{equation}
		\small
		\epsilon\leq \min\left\{\frac{8\beta^{3}L_{2}^{3}}{27L_{1}^{3}}, \frac{27}{64^{3}L_{2}^{3}}, \frac{\beta}{2}\right\},
	\end{equation}
	and $\sigma=3L_{1}\epsilon^{\frac{1}{3}} / 2\beta L_{2}$, the algorithm breaks at most 
	\begin{equation}
		\small
		2M\max\left\{\frac{2L_{1}}{\epsilon^{2}}, \frac{256L_{2}^{2}}{9\epsilon}\right\} = \cO(\epsilon^{-2})
	\end{equation}
	number of iterations.
\end{theorem}
\begin{proof}
	$\|\nabla R_{\bS}(\bw_{t})\| \geq \epsilon$ holds for two cases.
	\paragraph{Case 1:} If $\bw_{t}\in B_{2}(\textbf{0}, 1)$ with $\|\bw\| = 1$, then we have
	\begin{equation}\label{eq:descent equation 1}
		\small
		\begin{aligned}
			R_{\bS}(\bw_{t + 1}) - R_{\bS}(\bw_{t}) & \leq \left\langle\nabla R_{\bS}(\bw_{t}), \bw_{t + 1} - \bw_{t} \right\rangle + \frac{L_{1}}{2}\|\bw_{t + 1} - \bw_{t}\|^{2} \\
			& \leq -\frac{\beta^{2}}{L_{1}} + \frac{\beta^{2}}{2L_{1}} \\
			& = -\frac{\beta^{2}}{2L_{1}} \\
			& < -\frac{\epsilon^{2}}{2L_{1}},
		\end{aligned}
	\end{equation}
	due to the Assumption \ref{ass:local minima on the boundry} and Lispchitz gradient. 
	\paragraph{Case 2:} If $\bw_{t}\in B_{2}(\textbf{0}, 1)$ but $\|\bw_{t}\|< 1$ then
	\begin{equation}\label{eq:descent equation 2}
		\small
		\begin{aligned}
			R_{\bS}(\bw_{t + 1}) - R_{\bS}(\bw_{t}) & \leq \left\langle\nabla R_{\bS}(\bw_{t}), \bw_{t + 1} - \bw_{t} \right\rangle + \frac{L_{1}}{2}\|\bw_{t + 1} - \bw_{t}\|^{2} \\
			& \overset{a}{\leq} \left(-L_{1} + \frac{L_{1}}{2}\right)\|\bw_{t + 1} - \bw_{t}\|^{2} \\
			& =  -\frac{L_{1}}{2}\|\bw_{t + 1} - \bw_{t}\|^{2}.
		\end{aligned}
	\end{equation}
	Here $a$ is due to the property of projection. Then, if $\|\bw_{t + 1}\| < 1$, one can immediately verify that 
	\begin{equation}
		\small
		R_{\bS}(\bw_{t + 1}) - R_{\bS}(\bw_{t}) \leq -(1/2L_{1})\|\nabla R_{\bS}(\bw_{t})\|^2 \leq -\frac{\epsilon^{2}}{2L_{1}}. 
	\end{equation}
	On the other hand, if $\|\bw_{t}\| < 1$ while $\|\bw_{t + 1}\| = 1$, descent equation \eqref{eq:descent equation 2} implies $R_{\bS}(\bw_{t + 1}) - R_{\bS}(\bw_{t}) \leq 0$. More importantly, $\bw_{t + 1}$ goes back to the sphere. Then we go back to Case 1. Thus we have 
	\begin{equation}
		\small
		R_{\bS}(\bw_{t + 2}) - R_{\bS}(\bw_{t}) \leq R_{\bS}(\bw_{t + 2}) - R_{\bS}(\bw_{t + 1}) + R_{\bS}(\bw_{t + 1}) - R_{\bS}(\bw_{t}) \leq -\frac{\epsilon^{2}}{2L_{1}}
	\end{equation}
	in this situation. 
	\par
	Combining the results in these two cases, we have 
	\begin{equation}
		\small
		-2M \leq R_{\bS}(\bw_{2t}) - R_{\bS}(\bw_{0}) = \sum\limits_{j=1}^{t}R_{\bS}(\bw_{2(j)}) - R_{\bS}(\bw_{2(j - 1)}) \leq -\frac{t\epsilon^{2}}{2L_{1}}.
	\end{equation}
	Thus, $t\leq 4L_{1}M/\epsilon^{2}$. Then we can verify that $\bw_{t}$ approximates a first-order stationary point in the number of $\cO\left(\epsilon^{-2}\right)$ iterations.
	\par
	On the other hand, when $\|\nabla R_{\bS}(\bw_{t})\| \leq \epsilon \leq \beta / 2$, we notice that 
	\begin{equation}
		\small
		\|\nabla R_{\bS}(\bw)\| = \|\nabla R_{\bS}(\bw)\|\|\bw\| \geq \langle\nabla R_{\bS}(\bw), \bw\rangle \geq \beta,
	\end{equation}
	for any $\bw \in B_{2}(\textbf{0}, 1)$ with $\|\bw\| = 1$. Then by Lipschitz gradient, we have  
	\begin{equation}
		\small
		\begin{aligned}
			\|\bw - \bw_{t}\| & \geq \frac{1}{L_{1}}\|\nabla R_{\bS}(\bw) - \nabla R_{\bS}(\bw_{t})\| \\
			& \geq \frac{1}{L_{1}}\left(\|\nabla R_{\bS}(\bw)\| - \|\nabla R_{\bS}(\bw_{t})\|\right) \\
			& \geq \frac{1}{L_{1}}\left(\beta - \epsilon\right) \\
			& \geq \frac{\beta}{2L_{1}},
		\end{aligned}
	\end{equation}
	for any $\bw$ satisfies $\|\bw\| = 1$. Thus we can choose the $\bu_{t}$ in Algorithm \ref{alg:psgd}, and $\bu_{t}\in B_{2}(\textbf{0}, 1)$. Then with the Lipschitz Hessian, by taking $\sigma = \frac{3L_{1}\epsilon^{\frac{1}{3}}}{2\beta L_{2}}$ and $\epsilon\leq\min\left\{\frac{8\beta^{3}L_{2}^{3}}{27L_{1}^{3}}, \frac{27}{64^{3}L_{2}^{3}}\right\}$,
	\begin{equation}
		\small
		\begin{aligned}
			R_{\bS}(\bw_{t + 1}) - R_{\bS}(\bw_{t}) & \leq \sigma\left\langle R_{\bS}(\bw_{t}), \bu_{t} - \bw_{t}\right\rangle + \frac{\sigma^{2}}{2}(\bu_{t} - \bw_{t})^{T}\nabla^{2}R_{\bS}(\bw_{t})(\bu_{t} - \bw_{t}) + \frac{\sigma^{3}L_{2}}{6}\|\bu_{t} - \bw_{t}\|^{3} \\
			& \leq \sigma\|R_{\bS}(\bw_{t})\|\|\bu_{t} - \bw_{t}\| - \sigma^{2}\frac{\beta^{2}\epsilon^{\frac{1}{3}}}{16L_{1}^{2}} + \frac{\sigma^{3}L_{2}}{6}\left(\frac{\beta}{2L_{1}}\right)^{3} \\
			& \overset{a}{\leq} \sigma\frac{\beta\epsilon}{2L_{1}} - \sigma^{2}\frac{\beta^{2}\epsilon^{\frac{1}{3}}}{16L_{1}^{2}} + \sigma^{3}\frac{L_{2}\beta^{3}}{48L_{1}^{3}}\\
			& \leq \frac{3\epsilon^{\frac{4}{3}}}{4L_{2}} - \frac{9\epsilon}{128L_{2}^{2}}\\
			& \leq -\frac{9\epsilon}{256L_{2}^{2}}, 
		\end{aligned}
	\end{equation}
	where $a$ is from the value of $\bu_{t}$, and the last two inequality is due to the choice of $\sigma$ and $\epsilon$. Thus, combining this with  \eqref{eq:descent equation 1} and \eqref{eq:descent equation 2}, we see the Algorithm break after at most 
	\begin{equation}
		\small
		2M\max\left\{\frac{4L_{1}}{\epsilon^{2}}, \frac{256L_{2}^{2}}{9\epsilon}\right\} = \cO(\epsilon^{-2})
	\end{equation}
	iterations.
\end{proof}
From the result, we see that PGD approximates some $(\epsilon, \epsilon^{\frac{1}{3}})$ second-order stationary point at a computational cost of $\cO(\epsilon^{-2})$. 

\subsection{Excess Risk Under Non-convex problems}
We have the following corollary about the expected excess risk of the proposed PGD Algorithm \ref{alg:psgd}. This corollary is proved when we respectively plug $\zeta(t) = \max\left\{2\sqrt{ML_{1} / t}, 512L_{2}^{2} /9t \right\}$, $\rho(t) = \zeta(t)^{\frac{1}{3}}$ and $\delta = 0$ into the Theorem \ref{thm:generalization error for non-convex}. 
\begin{corollary}
	Under Assumption \ref{ass:smoothness}, \ref{ass:local strong convexity}, \ref{ass:strict saddle}, and \ref{ass:local minima on the boundry}. For $t$ satisfies with 
	\begin{equation}
		\small
		\max\left\{2\sqrt{\frac{ML_{1}}{t}}, \frac{512L_{2}^{2}}{9t}\right\} \leq \min\left\{\frac{8\beta^{3}L_{2}^{3}}{27L_{1}^{3}}, \frac{27}{64^{3}L_{2}^{3}}, \frac{\beta}{2}, \frac{\alpha^{2}}{2L_{0}}, \frac{\lambda^{3}}{8}\right\}
	\end{equation}
	we have 
	\begin{equation}
		\small
		\begin{aligned}
			\min_{1\leq s\leq t}|\mE_{\cA,\bS}\left[R(\bw_{s}) - R(\bw^{*})\right]| & \leq \frac{2L_{0}}{\lambda\sqrt{n}} +  \frac{4L_{0}}{\lambda}\max\left\{2\sqrt{\frac{ML_{1}}{t}}, \frac{512L_{2}^{2}}{9t}\right\} + \frac{2KM}{\sqrt{n}} \\
			& + \frac{8KL_{0}^{2}}{n\lambda} + \left(L_{0}\min\left\{6,\frac{3\lambda}{2L_{2}}\right\} + 2M\right)\xi_{n, 1} + 2M\xi_{n, 2}\\
			& + \mE_{\cA,\bS}[R_{\bS}(\cP_{\cM_{\bS}}(\bw_{t})) - R_{\bS}(\bw_{\bS}^{*})]. 
		\end{aligned}
	\end{equation}
	where $\bw_{t}$ is updated by PGD, $\xi_{n,1}$ and $\xi_{n,2}$ are respectively defined in Theorem \ref{thm:generalization error for non-convex} with $D=2$.
\end{corollary}
\section{Experiments}\label{app:experiments}
\begin{figure}[t!]\centering
	\begin{minipage}{0.32\linewidth}
		\includegraphics[width=1\linewidth]{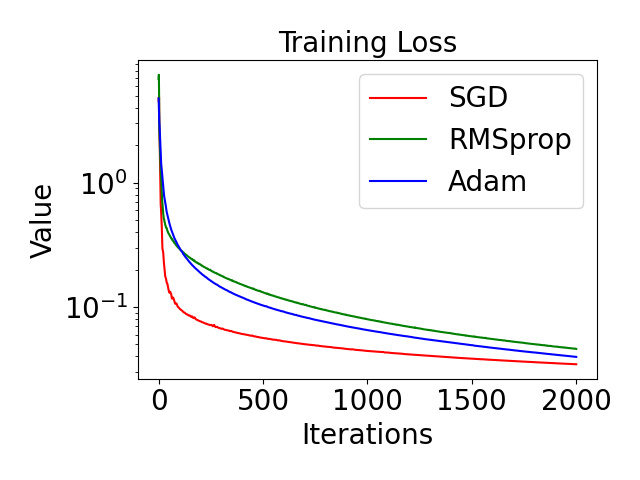}
	\end{minipage}	
	\begin{minipage}{0.32\linewidth}
		\includegraphics[width=1\linewidth]{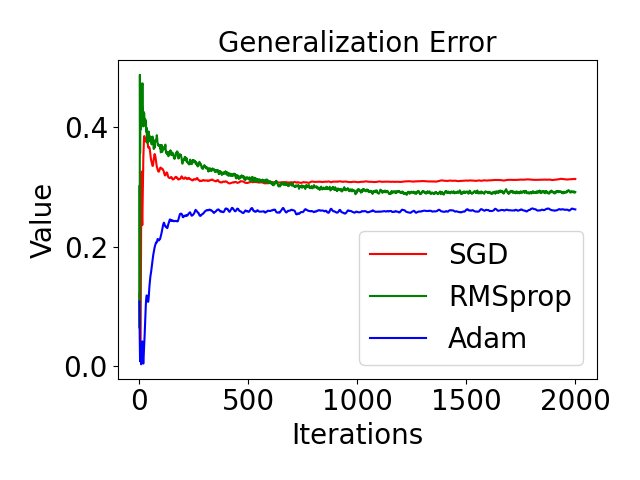}
	\end{minipage}
	\begin{minipage}{0.32\linewidth}
		\includegraphics[width=1\linewidth]{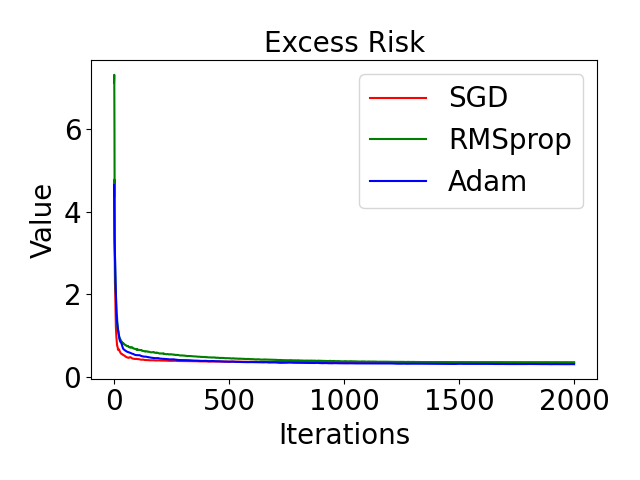}
	\end{minipage}
	\caption{Results of digits dataset under cross entropy loss. From the left to right are respectively training loss, generalization error, and excess risk.}
	\label{fig:digits}
\end{figure}

\begin{figure}[t!]\centering
	\begin{minipage}{0.32\linewidth}
		\includegraphics[width=1\linewidth]{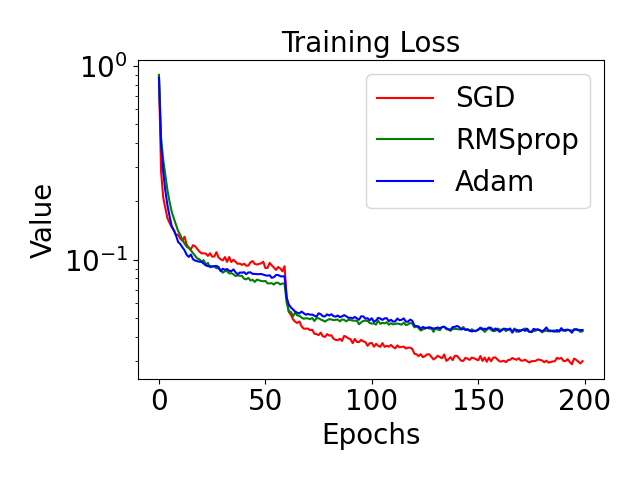}
	\end{minipage}	
	\begin{minipage}{0.32\linewidth}
		\includegraphics[width=1\linewidth]{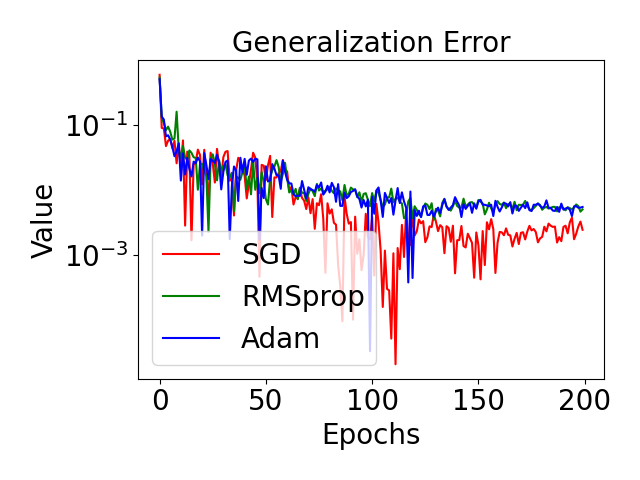}
	\end{minipage}
	\begin{minipage}{0.32\linewidth}
		\includegraphics[width=1\linewidth]{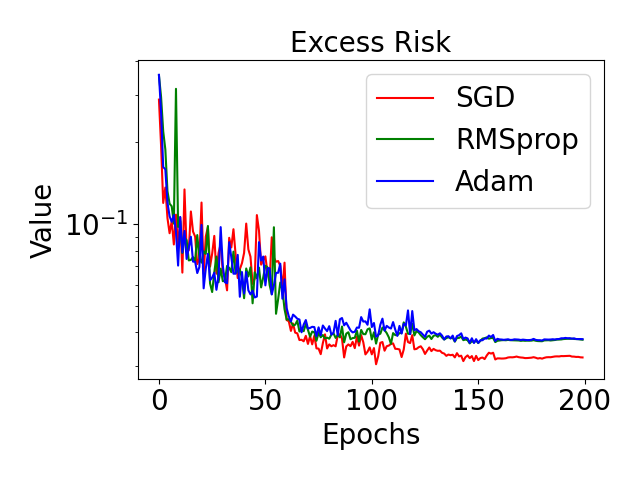}
	\end{minipage}
	\caption{Results of MNIST dataset on LeNet5. From the left to right are respectively training loss, generalization error and excess risk.}
	\label{fig:lenet}
\end{figure}

\begin{figure}[t!]\centering
	\begin{minipage}[t]{0.32\linewidth}
		\centering
		\includegraphics[width=1\linewidth]{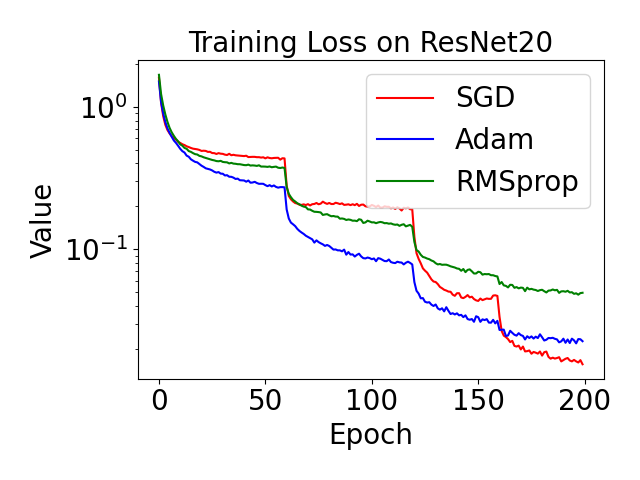}
		\vspace{0.02cm}
		\includegraphics[width=1\linewidth]{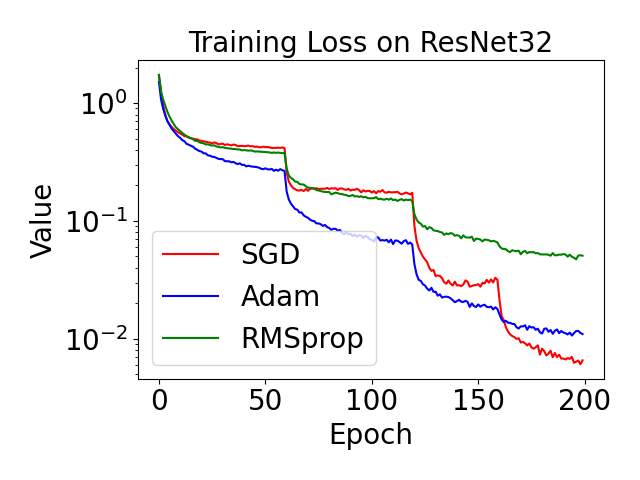}
		\vspace{0.02cm}
		\includegraphics[width=1\linewidth]{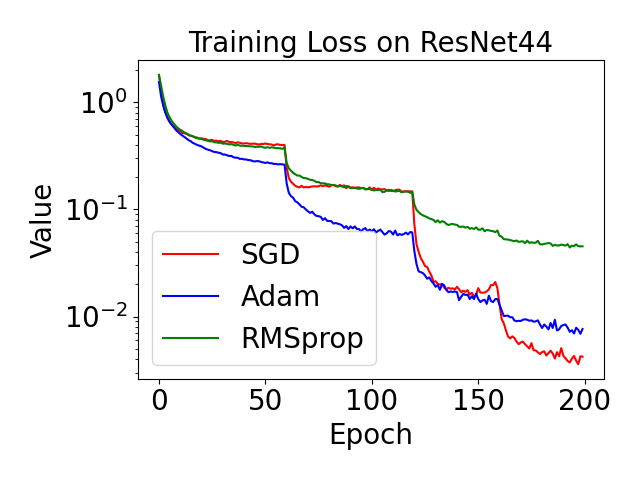}
		\vspace{0.02cm}
		\includegraphics[width=1\linewidth]{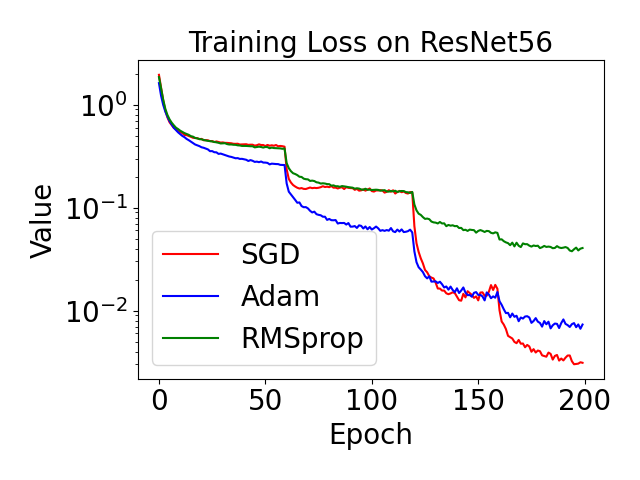}
		\vspace{0.02cm}
	\end{minipage}
	\begin{minipage}[t]{0.32\linewidth}
		\centering
		\includegraphics[width=1\linewidth]{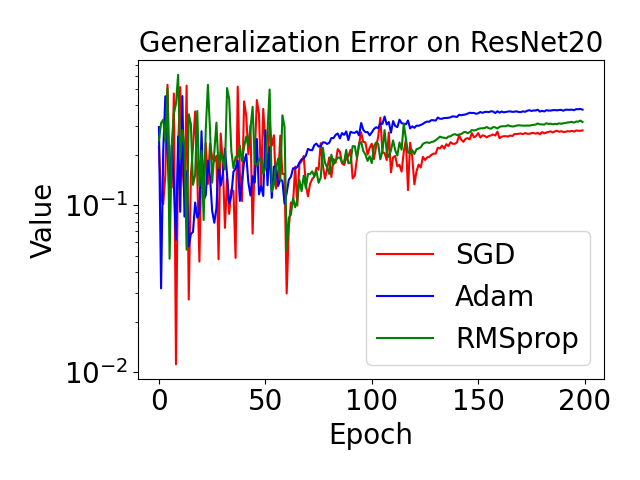}
		\vspace{0.02cm}
		\includegraphics[width=1\linewidth]{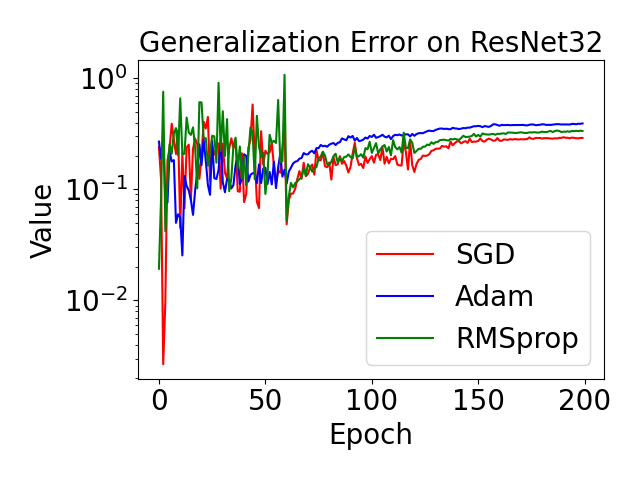}
		\vspace{0.02cm}
		\includegraphics[width=1\linewidth]{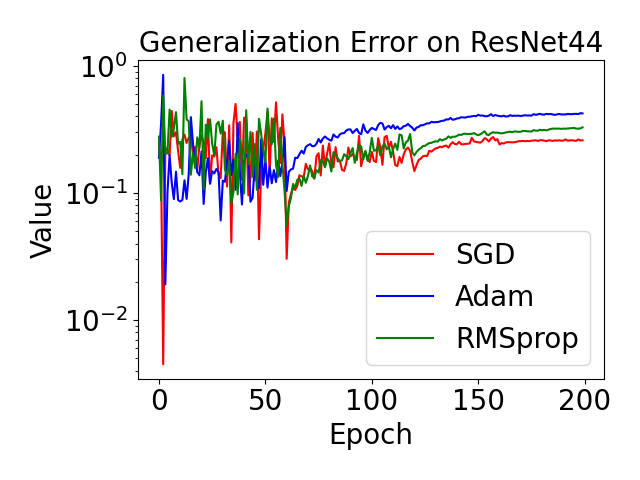}
		\vspace{0.02cm}
		\includegraphics[width=1\linewidth]{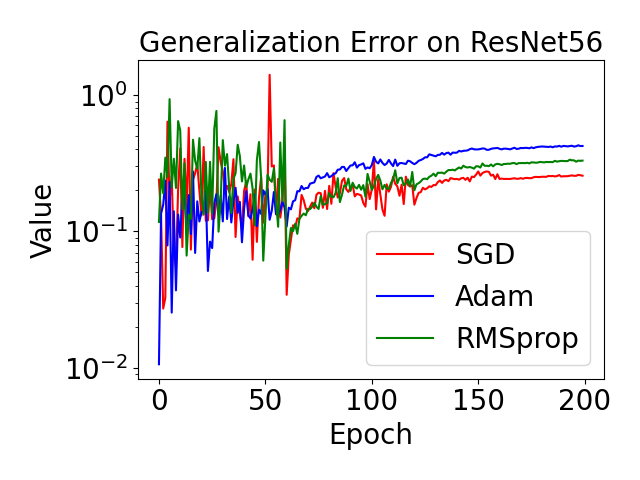}
		\vspace{0.02cm}
	\end{minipage}
	\begin{minipage}[t]{0.32\linewidth}
		\centering
		\includegraphics[width=1\linewidth]{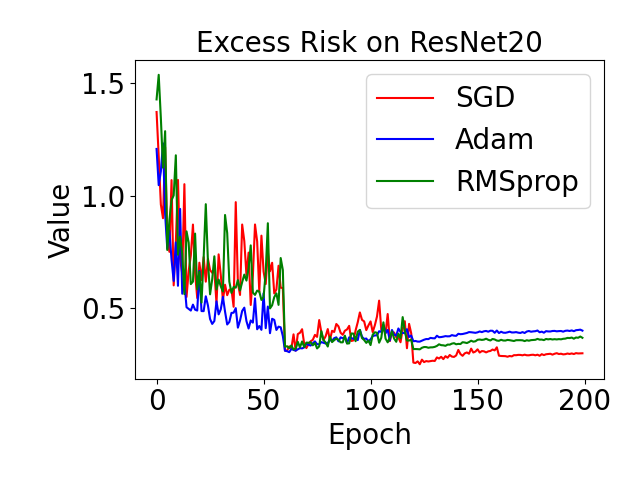}
		\vspace{0.02cm}
		\includegraphics[width=1\linewidth]{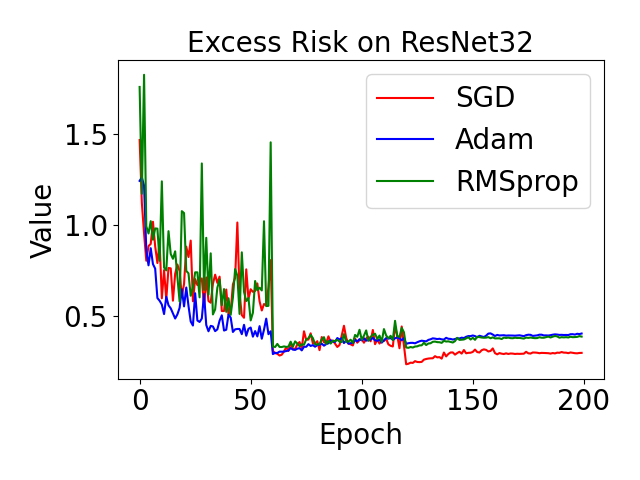}
		\vspace{0.02cm}
		\includegraphics[width=1\linewidth]{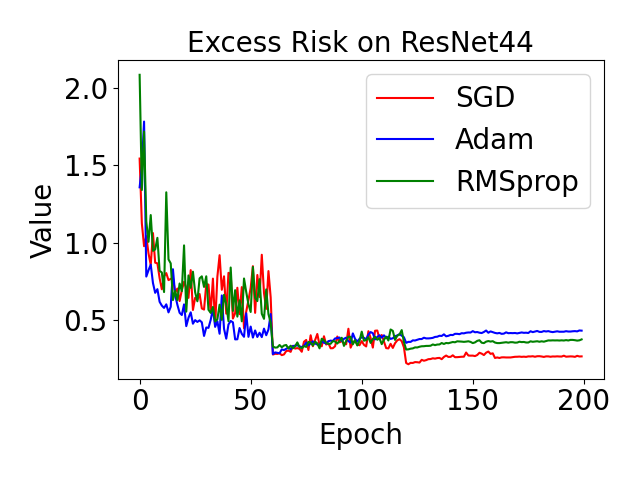}
		\vspace{0.02cm}
		\includegraphics[width=1\linewidth]{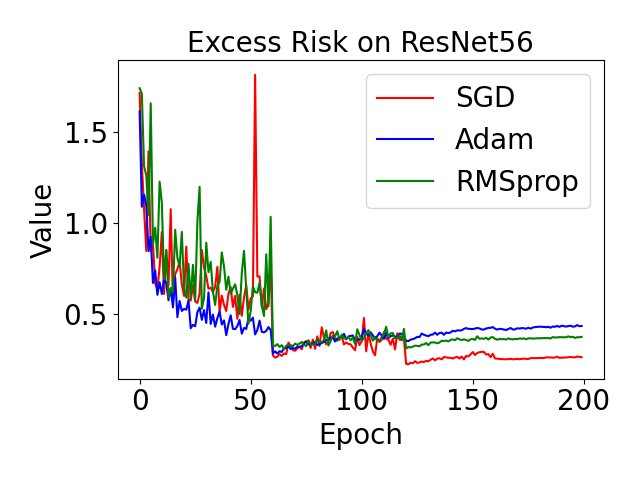}
		\vspace{0.02cm}
	\end{minipage}
	\caption{Results of CIFAR10 dataset on various structures of ResNet i.e., $20, 32, 44, 56$. From the left to right are respectively training loss, generalization error and excess risk.}
	\label{fig: cifar10}
\end{figure}

\begin{figure}[t!]\centering
	\begin{minipage}[t]{0.32\linewidth}
		\centering
		\includegraphics[width=1\linewidth]{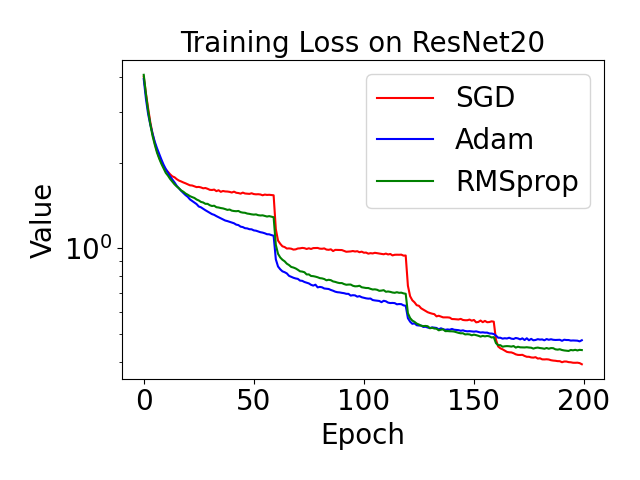}
		\vspace{0.02cm}
		\includegraphics[width=1\linewidth]{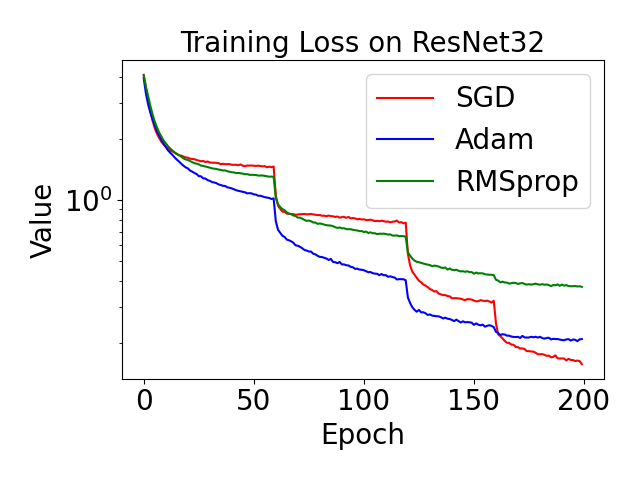}
		\vspace{0.02cm}
		\includegraphics[width=1\linewidth]{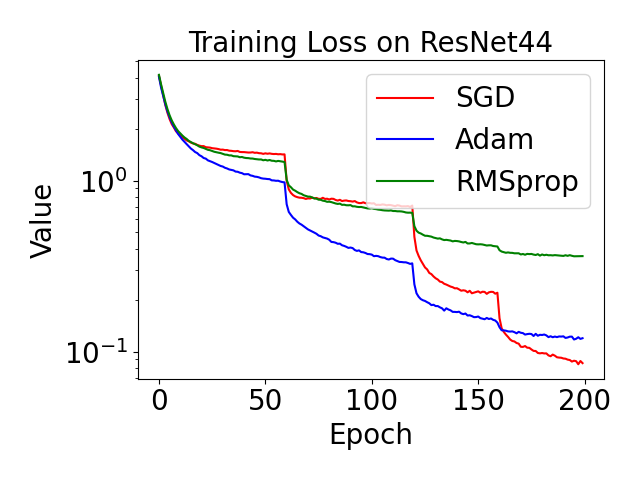}
		\vspace{0.02cm}
		\includegraphics[width=1\linewidth]{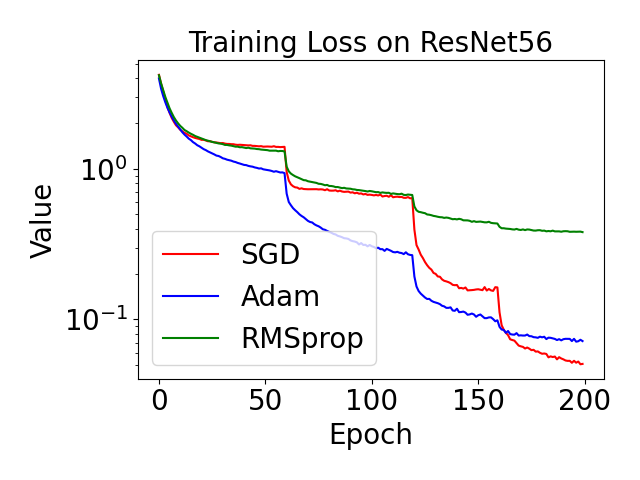}
		\vspace{0.02cm}
	\end{minipage}
	\begin{minipage}[t]{0.32\linewidth}
		\centering
		\includegraphics[width=1\linewidth]{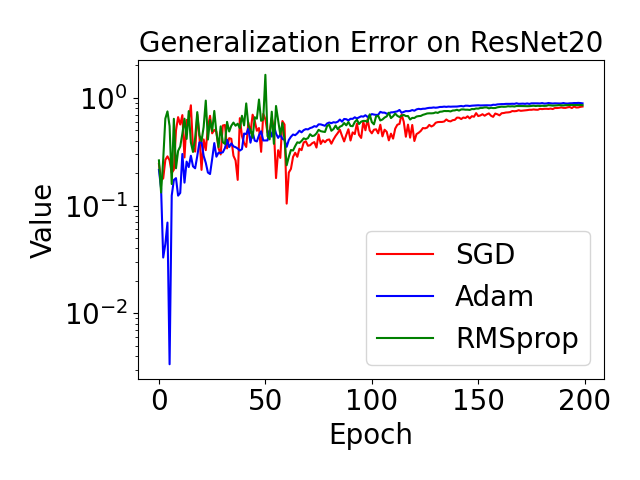}
		\vspace{0.02cm}
		\includegraphics[width=1\linewidth]{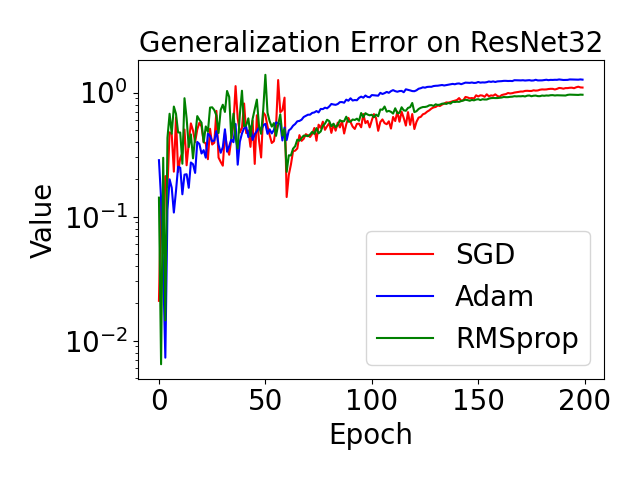}
		\vspace{0.02cm}
		\includegraphics[width=1\linewidth]{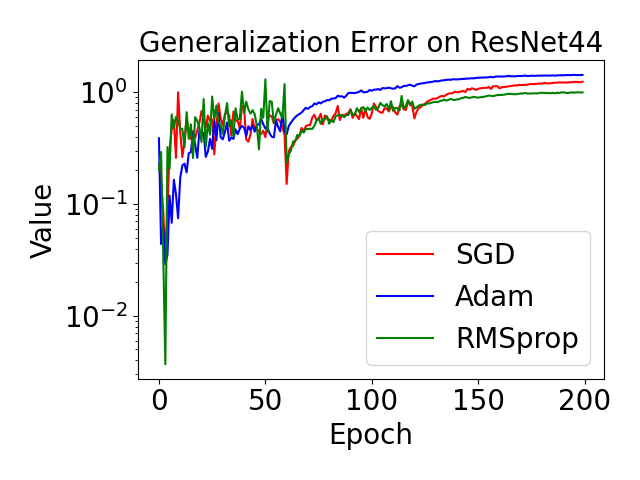}
		\vspace{0.02cm}
		\includegraphics[width=1\linewidth]{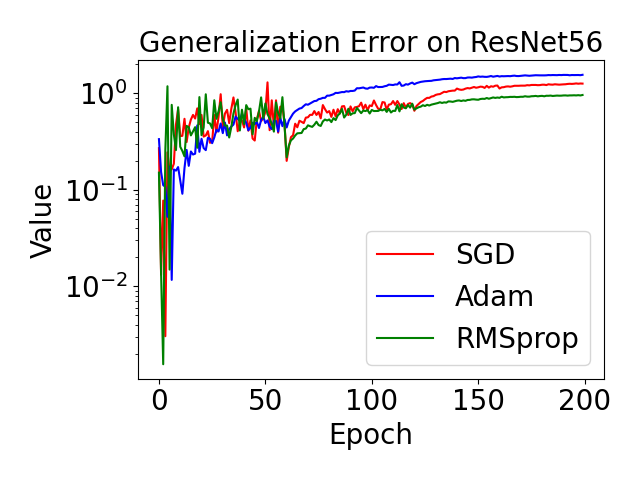}
		\vspace{0.02cm}
	\end{minipage}
	\begin{minipage}[t]{0.32\linewidth}
		\centering
		\includegraphics[width=1\linewidth]{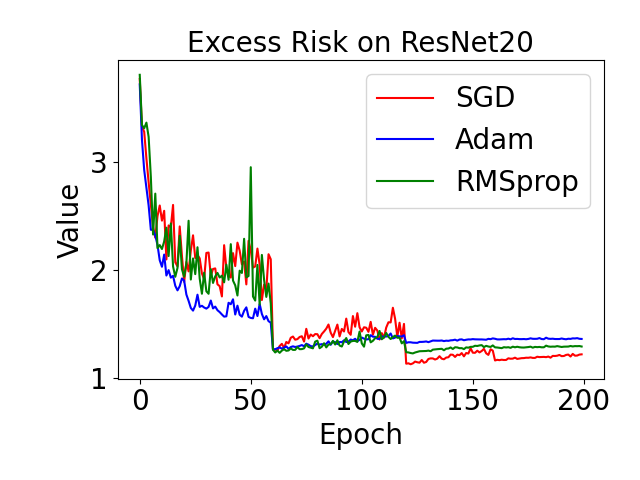}
		\vspace{0.02cm}
		\includegraphics[width=1\linewidth]{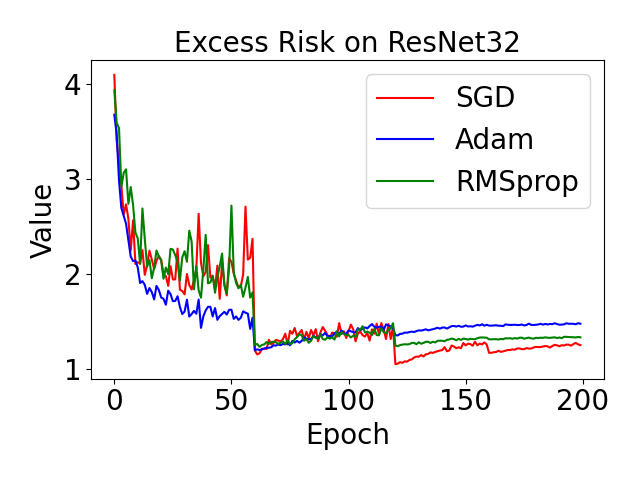}
		\vspace{0.02cm}
		\includegraphics[width=1\linewidth]{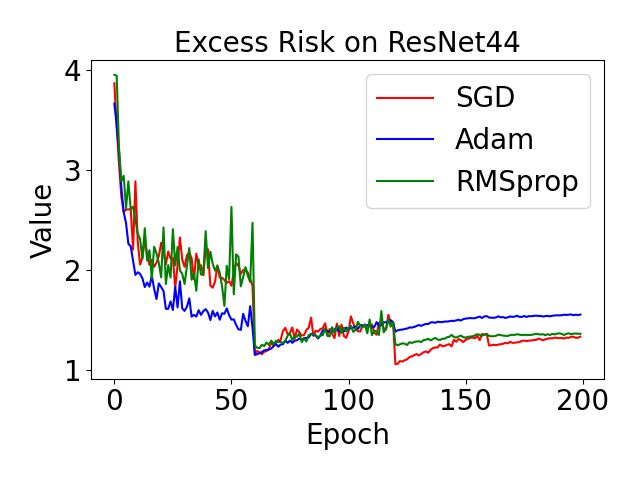}
		\vspace{0.02cm}
		\includegraphics[width=1\linewidth]{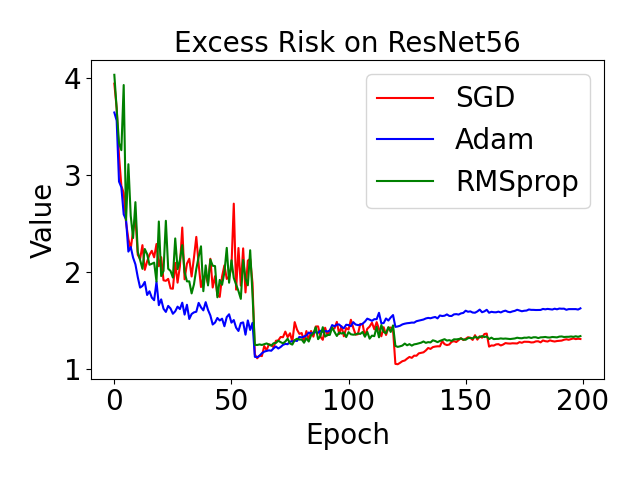}
		\vspace{0.02cm}
	\end{minipage}
	\caption{Results of CIFAR100 dataset on various structures of ResNet i.e., $20, 32, 44, 56$. From the left to right are respectively training loss, generalization error and excess risk.}
	\label{fig: cifar100}
\end{figure}
In this section, we empirically verify our theoretical results in this paper. The experiments are respectively conducted for convex and non-convex problems. We choose SGD \citep{robbins1951stochastic}; RMSprop \citep{tieleman2012lecture}, and Adam \citep{kingma2014adam} as three proper algorithms which are widely used in the field of machine learning. Since we can not access the exact population risk $R(\bw_{t})$ as well as $\inf_{\bw} R(\bw)$ during training. Hence, we use the loss on test set to represent the excess risk. Our experiments are conducted on a server with single NVIDIA V100 GPU. All the reported results are the average over five independent runs. 
\subsection{convex problems}
We conduct the experiments on multi-class logistic regression to verify our results for convex problems. We use the dataset \emph{digits} which is a set with $1800$ samples from $10$ classes. The dataset is available on package \emph{sklearn} \citep{scikit-learn}. 
\par
We split $70\%$ data as the training set and the others are used as the test set. We follow the training strategy that all the experiments are conducted for 2000 steps, the learning rates are respectively $0.1$, $0.001$, and $0.001$ for SGD, RMSprop, and Adam. They are decayed with the inverse square root of update steps. The results are summarized in the Figure \ref{fig:digits}.
\par
From the results, we see that training loss for the three proper algorithms converge close to zero, while the generalization error and excess risk converge to a constant. The observation is consistent with our theoretical conclusion in Section \ref{sec:testing error of convex function}.  
\subsection{Non-convex problems on Neural Network}
For the non-convex problem, we conduct experiments on image classification with various neural network models. Specifically, we use convolutional neural networks LeNet5 \citep{lecun1998gradient} and ResNet \citep{he2016deep}. The two structures are widely used in the image classification tasks, and they are leveraged to verify our conclusions for non-convex problems with model parameters in the same order of $n$ and much larger than $n$.
\par
For both structures, we follow the classical training strategy. All the experiments are conducted for $200$ epochs with cross entropy loss. The learning rates are set to be $0.1, 0.002, 0.001$ respectively for SGD, RMSprop, and Adam. More ever, the learning rates are decayed by a factor $0.2$ at epoch $60, 120, 160$. We use a uniform batch size $128$ and weight decay $0.0005$. 
\subsubsection{Model Parameters in the Same Order of Training Samples}
\paragraph{Data.} The dataset is MNIST \citep{lecun1998gradient} which contain binary images of handwritten digits with $50000$ training samples and $10000$ test samples. 
\paragraph{Model.} The model is LeNet5 which is a five layer convolutional neural network with nearly $60,000$ number of parameters.  
\paragraph{Main Results.} The results are summarized in Figure \ref{fig:lenet}. Our code is based on \url{https://github.com/activatedgeek/LeNet-5}. From the results, we see that the training loss monotonically decreases with the update steps, while both the generalization error and excess risk tend to converge to some constant. This is consistent with our theoretical results in Section \ref{sec: testing error for non-convex function} when $d$ is in the same order of $n$. 
\subsection{Model Parameters Larger than the Order of Training Samples}
\paragraph{Data.} The datasets are CIFAR10 and CIFAR100 \citep{krizhevsky2009learning}, which are two benchmark datasets of colorful images both with $50000$ training samples, $10000$ testing samples but from $10$ and $100$ object classes respectively. 
\paragraph{Model.} The model we used is ResNet in various depths i.e., $20, 32, 44, 56$. The four structures respectively have nearly $0.27$, $0.46$, $0.66$, and $0.85$ millions of parameters.   
\paragraph{Main Results.} The experimental results for CIFAR10 and CIFAR100 are respectively in Figure \ref{fig: cifar10} and \ref{fig: cifar100}. Our code is based on \url{https://github.com/kuangliu/pytorch-cifar}. The results show the optimization error, generalization error, and excess risk exhibit similar trends as the results on MNIST dataset. Thus, although our bounds in Section \ref{sec:analysis of non-convex function} are non-vacuous when $d$ is in the same order of $n$. The empirical verification on the over-parameterized neural network indicates that our results potentially can be applied to the regime of $d\gg n$.  
\section{Examples}\label{app:examples}
In this Section, we present three examples satisfies our assumptions imposed in this paper. Let us start with a linear regression problem for convex optimization. 
\begin{example}[Linear Regression]
	Let $\bz = (\bx, y)$, $y = \bx^{\top}\bw^{*} + \epsilon$ for independent noise $\epsilon$, and $f(\bw, \bz) = (y - \bw^{\top}\bx)^{2}$.
\end{example}
For any $\bz$, the quadratic loss $f(\bw, \bz)$ is convex, and satisfies our smoothness condition Assumption \ref{ass:smoothness}. Obviously, when the Hessian of population risk $E[\bx\bx^{\top}]$ is positively definite, the population risk is local (global) strongly convex, thus Assumptions \ref{ass:smoothness}, \ref{ass:local strong convexity}, and \ref{ass:convexity} are satisfied. However, for any instantaneous loss $f(\bw, z)$ has Hessian of $\bx\bx^{\top}$ which means $f(\bw, \bz)$ is not necessarily strongly convex with respect to $\bw$ for any $\bz$. Thus, we can only treat it as a convex loss function when applying the technique in \citep{hardt2016train}, and get the excess risk bound of order $O(\sqrt{1/n})$. However, the empirical minimizer has a excess risk of order $O(1/n)$ which matches our result. By the way, the technique in \citep{zhang2017empirical} also can be applied here, while they require the number of data is sufficiently large, while we do not have such requirement.
\par
The above example has a globally strongly convex population risk, let us consider the following example with locally but not globally strongly convex population risk.  
\begin{example}[Robust Regression]
	Let $\bz = (\bx, y)$, $y = \bx^{\top}\bw^{*} + \epsilon$ for independent noise $\epsilon$, and $f(\bw, \bz) = \phi(y - \bw^{\top}\bx)$, with 
	\begin{equation}
		\small
		\phi(u) = 
		\begin{cases} u^{2} - \frac{1}{3}u^{3} \qquad & 0\leq u \leq 1, \\
		u^{2} + \frac{1}{3}u^{3} \qquad & 0\leq u \leq 1, \\
		|u| & |u| \geq 1.
		\end{cases}
	\end{equation}
\end{example}
By computing the gradient and Hessian, one can verify that for any $\bz$, our robust regression loss $f(\bw, \bz)$ is convex, and satisfies our smoothness condition Assumption \ref{ass:smoothness}. Again, when the matrix $E[\bx\bx^{\top}]$ is positively definite, the population risk of this example is locally but not globally strongly convex. Then the example satisfies our Assumption \ref{ass:smoothness}-\ref{ass:convexity}. One can also show that the empirical risk minimizer has the generalization bound of order $\cO(1 / n)$ when $\mE[\epsilon^{2}]$ is small enough. The error also matches our generalization bound in Theorem \ref{thm:stability of convex function}.
\par
Finally, we consider an example of non-convex loss that satisfies our imposed Assumptions \ref{ass:smoothness} and \ref{ass:strict saddle}. 
\begin{example}
	Let $\bz_{i}$ be mixture Gaussian data such that $\bz_{i}\sim \frac{1}{2}\cN(\bw_{1}^{*}, \bI) +  \frac{1}{2}\cN(\bw_{2}^{*}, \bI) = p_{\bw^{*}}(\cdot)$. The maximizing likelihood loss is $f(\bw, \bz) = -\log{p_{\bw}(\bz)}$. 
\end{example}
By checking the gradient and Hessian, the loss function $f(\bw, \bz)$ satisfies smoothness Assumption \ref{ass:smoothness}. The population risk $R(\bw) = -E_{\bz\sim p_{\bw^{*}}}[\log{p_{\bw}(\bz)}]$, which has two global minima $(\bw_{1}^{*}, \bw_{2}^{*}), (\bw_{2}^{*}, \bw_{1}^{*})$, and a saddle point $((\bw_{1}^{*} + \bw_{2}^{*}) / 2, (\bw_{1}^{*} + \bw_{2}^{*}) / 2)$. Thus, this problem violates the PL-inequality which says that every local minima are global minima. However, by Lemma 16 in \citep{mei2018landscape}, we can compute the Hessian to check that the two population global minima are all strict local minima, while the saddle point is strict saddle point. Thus, the example satisfies our Assumptions \ref{ass:smoothness} and \ref{ass:strict saddle}. 

	\clearpage

\end{document}